\documentclass{article}

\usepackage{arxiv}

\usepackage[utf8]{inputenc} 
\usepackage[T1]{fontenc}    
\usepackage{hyperref}       
\usepackage{url}            
\usepackage{booktabs}       
\usepackage{amsfonts}       
\usepackage{nicefrac}       
\usepackage{microtype}      
\usepackage{graphicx}
\usepackage[natbib=true,style=alphabetic-verb,maxalphanames=9,maxbibnames=9]{biblatex}
\usepackage{doi}
\usepackage{amsmath}
\usepackage{setspace}
\usepackage{amsthm}
\usepackage{enumitem}
\usepackage{caption}
\usepackage{subcaption}

\newtheorem{theorem}{Theorem}
\newtheorem{corollary}{Corollary}
\renewcommand\labelenumi{(\roman{enumi})}
\renewcommand\theenumi\labelenumi
\DeclareMathOperator{\sign}{sgn}
\DeclareMathOperator{\supp}{supp}
\DeclareMathOperator{\spn}{span}

\bibliography{references}

\title{Wavelet Neural Networks versus Wavelet-based Neural Networks}

\renewcommand{\shorttitle}{\textit{arXiv} Template}

\DeclareLabelalphaTemplate{
  \labelelement{
    \field[final]{shorthand}
    \field{label}
    \field[strwidth=1,strside=left,ifnames=1]{labelname}
    \field[strwidth=1,names=3,strside=left,ifnames={4-}]{labelname}
    \field[strwidth=1,strside=left]{labelname} 
  }
  \labelelement{
    \literal{--}
  }
  \labelelement{
    \field{year}    
  }
}
\DeclareFieldFormat[article]{title}{#1} 
\DeclareFieldFormat[inproceedings]{title}{#1}

\begin{document}

\author{
  Lubomir T. Dechevsky, Kristoffer M. Tangrand \\
  lde009@uit.no, ktangrand@gmail.com \\
  Faculty of Engineering Science and Technology \\
  UiT -- The Arctic University of Norway \\
  P. O. Box 385, N-8505 Narvik, Norway
}
\fontsize{12pt}{12pt}\selectfont

\maketitle

\begin{abstract}
This is the first paper in a sequence of studies including also \citep{llhm2022} and \citep{llhm2022_1} in which we introduce a new type of \emph{neural networks} (NNs) -- \emph{wavelet-based neural networks} (WBNNs) -- and study their properties and potential for applications.
We begin this study with a comparison to the currently existing type of \emph{wavelet neural networks} (WNNs) and show that WBNNs vastly outperform WNNs.
One reason for the vast superiority of WBNNs is their advanced hierarchical tree structure based on \emph{biorthonormal multiresolution analysis} (MRA).
Another reason for this is the implementation of our new idea to incorporate the wavelet tree depth into the neural width of the NN.
The separation of the roles of wavelet depth and neural depth provides a conceptually and algorithmically simple but very highly efficient methodology for sharp increase in functionality of swarm and deep WBNNs and rapid acceleration of the machine learning process.
In Theorem \ref{th:1} (Section \ref{s2}) we obtain a new result for the established WNNs: we propose a type of activation which is shown to lead to optimal performance of WNNs and show that even optimal performance of WNNs is vastly outperformed by WBNNs.
In Section \ref{s3}, in Theorems \ref{th:2} and \ref{th:3} we obtain new results about the learnability via WNNs and WBNNs and in Corollary \ref{cor:1} we show that WBNNs can be used to learn efficiently not only any regular distribution in $L_{1, loc}$ but also singular distributions like the Dirac delta and its derivatives.
In the same section we provide the general characteristics (i--iii) of the rich diversity of activation operators that can be used in machine learning via WBNNs of univariate and multivariate manifolds in two, three and higher-dimensional spaces.
Here we establish the principal differences between non-threshold and threshold activation in learning fractal and piecewise smooth manifolds, respectively.
In Section \ref{s4} we briefly address the importance of interconnection and interaction between \emph{swarm AI} and \emph{deep evolutionary AI} and the relevance to computational implementations using CPU and GPU parallelism.
In Section \ref{s5} we introduce a new activation method based on the concept of \emph{decreasing rearrangement} in functional analysis and function space theory.
Theorem \ref{th:4} is a uniform approximation theorem (UAT) providing qualitative proof of the consistency of the learning process when using the \emph{decreasing rearrangement activation}.
Theorem \ref{th:5} provides an important quantitative upgrade of the UAT in Theorem \ref{th:4} by showing that decreasing rearrangement activation of WBNNs results in machine learning process which is optimal in two key aspects: \emph{fastest learning} and \emph{maximal compression}.
In Section \ref{s6} we consider four representative model examples which are then subjected to comprehensive graphical comparison the results of which have been systematized into a collection of conclusions and comments.
Section \ref{s7} is comprised of the proofs. 
In the concluding Section \ref{s8} we discuss the connection of the present results with the studies in \citep{llhm2022}, \citep{llhm2022_1}, as well as some additional computational and research topics.

\end{abstract}


\section{Introduction}\label{intro}
The purpose of this paper is to propose a new approach to machine learning of geometric manifolds in $\mathbb{R}^n$, where $n=1,2,3,4,...$ using single-layer or deep neural networks (NNs) based on Riesz unconditional bases of biorthonormal wavelets.

The first attempt to marry the theory of NNs with wavelet theory dates back to the early 1990s \citep{Zhang1992-wt}. This initial study gave rise to particularly constructed NNs which were named by the authors of \citep{Zhang1992-wt} as \emph{wavelet NNs} (WNNs). In the course of the next twenty years, the theory and applications of WNNs were studied by numerous authors. The results of these studies have been summarized in \citep{Alexandridis2013-gb} which constitutes a comprehensive account of the current status of the study of WNNs. Before going into the mathematical details of the construction and functioning of WNNs, let us note that this type of networks was introduced relatively early, when wavelet theory was still quite new to the developers of applications in the field of Artificial Intelligence (AI). Due to this, WNNs make use only of a very small subset of the useful properties of wavelet bases. Thus, while the theory of WNNs relies on the basic property of wavelet basis functions that they are dilations and translations of one and the same function, this theory ignores the more advanced properties of wavelet bases related to \emph{Multi-Resolution Analysis} (MRA). As a consequence of this, methods using WNNs are no more than a variant of meshless kernel estimation methods. The typical representative of these meshless methods are the ones using radial basis functions \citep{80341}. The only essential difference with the variant of WNNs is that radial basis functions are replaced by tensor-product functions with sufficient number of vanishing moments. Not surprisingly, the mathematical apparatus used with WNNs is identical with the one for radial basis functions: iterative gradient or subgradient optimization methods. Unfortunately, these iterative methods can guarantee providing the global extrema only when the respective criterial functionals (objective functions) are convex. (Of course, in the particular case when the convex criterial functional is quadratic, possibly with linear constraints, in addition to iterative methods there exists also a broad variety of methods of computational linear algebra.) In practical applications, however, the realistic criterial functionals are most often non-convex, with multiple local extrema and saddlepoint singularities. In this general situation, the optimization methods used with WNNs and radial basis functions produce only local extrema which are close to the global extrema only if a very good initial starting point of the iterative algorithm is proposed. The usual defence of this type of results is to claim that all the local extremal values are close in value to the global extremal value. Here is a typical exposition of this type of argument \citep{LeCun2015-iu}: "..., In particular, it was commonly thought that simple gradient descent would get trapped in poor local minima -- weight configurations for which no small change would reduce the average error. In practice, poor local minima are rarely a problem with large networks. Regardless of the initial conditions, the system nearly always reaches solutions of very similar quality. Recent theoretical and empirical results strongly suggest that local minima are not a serious issue in general. Instead, the landscape is packed with a combinatorially large number of saddlepoints where the gradient is zero, and the surface curves up most dimensions and curves down in the remainder. The analysis seems to show that saddlepoints with only a few downward curving directions are present in very large numbers, but almost all of them have very similar values of the objective function. Hence, it does not much matter which of these saddlepoints the algorithm gets stuck at."

Some critical analysis of the above text in \citep{LeCun2015-iu} is due, as follows. While there is some rationale in the above claims for large and very large sample and network sizes, these claims cannot be accepted even as basic "rules of thumb"; expressions like "rarely", "nearly always", "very similar", "strongly suggest", "not a serious issue in general", "seems to show", "almost all", "does not much matter" are not good replacements for logical quantors. The argument about saddlepoints with only very few negative components in the signature (that is - to use fuzzy terminology in the spirit of \citep{LeCun2015-iu} - 'saddlepoints which are almost local minima'), is also unconvincing as a qualitative statement without any criteria or means for quantitative measurement. Even if a saddlepoint in a problem with a very large size has only one downward-curving dimension, the respective value of the criterial functional (objective function) can be much larger than the global minimum, if the downward curve is sufficiently steep. The one rigorous conclusion that can be drawn from the above excerpt of \citep{LeCun2015-iu} is that, once the criterial functional (objective function) ceases to be (globally) convex, iterative gradient/subgradient optimization is no longer a reliable approach to achieving quality learning results. The only way to achieve best (or, at least, sufficiently high) quality results is to start from a very good initial point of the iterative process, but the traditional way of achieving this is by human intervention, i.e., the use of natural, rather than artificial intelligence. In fact, the authors of \citep{LeCun2015-iu} acknowledge this in another excerpt of their text, as follows: "... The conventional option is to hand design good feature extractors, which requires a considerable amount of engineering skill and domain expertise. But this can all be avoided if good features can be learned automatically using a general-purpose learning procedure. This is the key advantage of deep learning ...". Our comment to this excerpt is that in the general case of non-convex criterial functionals, gradient search only plays the role of an auxiliary tool for improvement of already good results. Achieving these good initial results using AI is thus claimed in \citep{LeCun2015-iu} to be the main aim of deep learning, and, in general, this aim cannot be attained by only using local optimization methods. The major weakness of WNNs proposed in \citep{Zhang1992-wt} is that the crucial problem of finding a good initial starting point for iterative local optimization has only been addressed by the vague recommendation that some 'explicit link between the network coefficients and the wavelet transform' should be provided. This weakness persists also in the later developments and upgrades of WNNs discussed in \citep{Alexandridis2013-gb}. 
Our present study shows that for sufficiently large samples this weakness can be overcome, at least partially, via tools from functional analysis. 
Namely, there is a rigorous mathematical general way to automatically improve the quality of the starting point of the iterative optimization process, valid for all cases when the criterial functional can be interpreted as distance between two mathematical objects. Since, to the best of the authors' knowledge, this systematic approach seems to be new in the context of NNs (and much more certainly so in the specific context of WNNs), we shall outline its main idea already in this early stage of our exposition, as follows.
If the metric criterial functional of an optimization problem has an \emph{equivalent metric which can be computed efficiently without the need of iterative optimization}, then this equivalent metric $d_1$:

\begin{equation}\label{eq:1}
    0 < c_0 d_1 (x,y) \leq d(x,y) \leq c_1 d_1 (x,y)
\end{equation}

can be used to generate a consistently good starting point for the optimization of the original metric $d$, provided that the equivalence constants $c_j, \; j=0,1$, with $0 < c_0 \leq 1 \leq c_1 < \infty$, do not depend on $x,y$ and, in numerical problems, they are independent of the sample size of the numerical data (in the sequel of this exposition we shall use the notation $d \asymp d_1$). 
In most practical applications in numerical analysis the metrics $d$ and $d_1$ are induced by respective equivalent norms or quasinorms \cite{bergh1976} or their seminorm variants. This refers not only to deterministic quantities, but also in the indeterministic case, e.g. when considering equivalent risks in statistical estimation. A typical model example, with various applications in deterministic approximation and statistical risk estimation, is the Peetre $\mathbb{K}$-\emph{functional} between Lebesgue space $L_p$ and homogeneous Sobolev space $\dot{W}_{p}^{k}, \; 1 \leq p \leq \infty, \; k \in \mathbb{N}$, \cite{bergh1976}, 

\begin{equation}\label{eq:2}
    K(h^k,\: f; \: L_p, \dot{W}_p^k) = \inf_{\varphi \in \dot{W}_p^k} ({||{f-\varphi}||}_{L_p} + h^k {||\varphi||}_{\dot{W}_p^k})
\end{equation} 

where $f \in L_p + \dot{W}_p^k$ (the algebraic sum of the two spaces) and $h$ is the step (in applications, related to the sample size). An equivalent seminorm of $K(h^k, f; L_p, \dot{W}_p^k)$ is ${||f-f_{k,h}||}_{L_p} + h^k {||f_{k,h}||}_{\dot{W}_p^k}$ where $f_{k,h}$ is the Steklov mean value of $f$ with parameters $k$ and $h$ \cite{dech97}, the equivalence constants being independent on $f$ and $h$. The numerical computation of $f_{k,h}$ is based on quadrature formulae and does not involve optimization. Thus, $\varphi_0 = f_{k,h}$ can be used as a starting point of iterative optimization. The quality of $\varphi_0$ as initial solution of the optimization problem depends on the size of the equivalence constants $c_0$ and $c_1$: if $c_0 = 1 = c_1$ (isometric equivalence) then $\varphi_0$ is the exact solution of the optimization problem. In the considered example, $c_1$ increases rapidly with increase of $k$ and so for a fixed step $h > 0$ (fixed sample) the quality of $\varphi_0$ as initial solution deteriorates with the increase of $k$. Fix $k \in \mathbb{N}$ and let $h \rightarrow 0+$ (take sufficiently large sample): since $c_j$, $j=0,1$, do not depend on $h$ (the sample size), for sufficiently large sample sizes $\varphi_0$ will be a consistently good starting point of the iterative optimization. Our first new result in Section 2 is to use the above idea for automatic generation of an initial starting point of iterative optimization in the case of WNNs. Although satisfactory from theoretical point of view, the practical usefulness of this generation would be rather limited, because the generated initial solution of the optimization problem would be consistently close to the global optimum only for sufficiently large sample sizes. Nothing is guaranteed for large samples of any a priori fixed size, let alone samples of medium or small size (in our numerical examples in the sequel of this exposition we shall consider sample sizes with $N \geq 2^{12}$ as very large, $2^{10} < N < 2^{12}$ as large, $2^9 \leq N \leq 2^{10}$ as medium-to-large, $2^8 \leq N < 2^9$ as medium, $2^7 \leq N < 2^8$ as medium-to-small, and $1 \leq N < 2^7$ as small). Our conclusion is that WNNs can be used efficiently for finding a consistently good local extremum only for very large sample sizes.

This weakness cannot be overcome within the conceptual construction of WNNs: a more advanced construction of relevant NNs is needed which we shall introduce in the present paper and call \emph{Wavelet-Based Neural Network} (WBNN). The principal difference between WNNs and WBNNs is explained, as follows. Let $\varphi$ be a scaling function (father wavelet) and $\psi$ be the corresponding wavelet (mother wavelet) obtained by MRA \citep{daubechies10lect1992},\citep{Dahmen1997WaveletAM}, so that for any $j_0 \in \mathbb{Z}$ the functions 

\begin{equation}\label{eq:3}
    \begin{split}
        \varphi_{j_0k_0} (x_1) = 2^{\frac{j_0}{2}} \varphi(2^{j_0}x_1 - k_0), \quad
        \psi_{jk}(x_1)=2^{\frac{j}{2}}\psi(2^{j}x_1 - k) \\
    \end{split}
\end{equation}

$x_1 \! \in \! \mathbb{R}, \; j = j_0, j_0+1,..., \, k_0 \in \mathbb{Z},\; k \! \in \! \mathbb{Z}$, form an orthonormal basis of $L_2(\mathbb{R})$ and $\int\limits_{\mathrm{supp}\; \psi} x_{1}^{\lambda} \psi (x_1) \, \mathrm{d} x = 0$ holds for all $\lambda = 0,1,...$ with $\lambda < r$ for some, henceforward fixed,  $r > 0$, where $\mathrm{supp}\; \psi$ is the support of $\psi$ in $\mathbb{R}$.

Let $\varphi$ be compactly supported in $\mathbb{R}$ (which implies the same for $\psi$). Let $B^s_{pq}(\mathbb{R}), -\infty<s<+\infty, 0<p\le\infty, 0<q\le$ be the inhomogeneous Besov space with smoothness index $s$, metric power index $p$ and metric logarithmic index $q$ (a definition will be given below). Assume that $\varphi \in B^r_{\infty\infty}(\mathbb{R})$, $\psi \in B^r_{\infty\infty}(\mathbb{R})$. Then, since $\varphi$ and $\psi$ are compactly supported, $\varphi \in B^r_{p\infty}(\mathbb{R})$, $
\psi \in B^r_{p\infty}(\mathbb{R})$ holds for every $p: 0 < p \le \infty$. For $x=(x_1,...,x_n) \in \mathbb{R}^n$, $k=(k_1,...,k_n) \in \mathbb{Z}^n$, consider \citep{dechevsky1999alone}

\begin{equation}\label{eq:4}
\begin{split}
    \varphi^{[0]}_{0k}(x) = \varphi_{0k_1}(x_1)\varphi_{0k_2}(x_2)\varphi_{0k_3}(x_3) ... \varphi_{0k_n}(x_n) \\
    \psi^{[1]}_{jk}(x) = \psi_{jk_1}(x_1)\varphi_{jk_2}(x_2)\varphi_{jk_3}(x_3) ... \varphi_{jk_n}(x_n) \\
    \psi^{[2]}_{jk}(x) = \varphi_{jk_1}(x_1)\psi_{jk_2}(x_2)\varphi_{jk_3}(x_3) ... \varphi_{jk_n}(x_n) \\
    \hdots \\
    \psi_{jk}^{[2^n-1]}(x) = \psi_{jk_1}(x_1)\psi_{jk_2}(x_2)\psi_{jk_3}(x_3) ... \psi_{jk_n}(x_n) \;.
\end{split}
\end{equation}

Denote $\varphi^{[0]} = \varphi_{00}^{[0]}$, $\psi^{[l]}=\psi_{00}^{[l]}$, $l=1,2,...,2^{n}-1$. Then, $\varphi^{[0]} \in B_{p\infty}^r(\mathbb{R}^n)$, $\psi^{[0]} \in B_{p\infty}^r(\mathbb{R}^n$ for any $p: 0 < p \leq \infty$, where $\psi^{[l]}$ is orthogonal to all polynomials of $n$ variables of total degree less than $r$. Besides, $\{\varphi_{0k}^{[0]}, \psi_{jk}^{[l]}\}_{k \in \mathbb{Z}^n, j=0,1,...,2^{n}-1}$ is an orthonormal basis of $L_2(\mathbb{R}^n)$.

Moreover, for $f \in B^s_{pq}(\mathbb{R}^n), 0 < p \leq \infty, 0 < q \leq \infty, n(\frac{1}{p}-1)_{+} < s < r$, 

\begin{equation}\label{eq:5}
    f(x)=\sum_{k \in \mathbb{Z}^n} \alpha_{0k} \varphi_{0k}^{[0]}(x) + \sum_{j=0}^\infty \sum_{k \in \mathbb{Z}^n} \sum_{l=1}^{2^n-1} \beta_{jk}^{[l]} \psi_{jk}^{[l]}(x)
\end{equation}

for Lebesgue a. e. $x \in \mathbb{R}^n$ holds, where $\alpha_{0k} = <\varphi_{0k}^{[0]},f> = \int_{\mathbb{R}^n}\varphi_{0k}^{[0]}(x)f(x)dx$, $\beta_{jk}^{[l]}=<\psi_{jk}^{[l]},f>$ and $a_{+}=\max\{a,0\}, a \in \mathbb{R}$. 
Convergence in  (\ref{eq:5}) is in the quasinorm topology of the inhomogeneous Besov space $B_{pq}^s(\mathbb{R}^n)$ and, in view of the lower constraint about $s$, also in every Lebesgue point of $f$, i.e., Lebesgue almost everywhere (Lebesgue -- a.e.) on $\mathbb{R}^n$. 
Here, $B_{pq}^s(\mathbb{R}^n)$ admits the following quasinorm in terms of wavelet coefficients:

\begin{equation}\label{eq:6}
    \|f\|_{B_{pq}^s(\mathbb{R}^n)} = \Bigg\{ \Big(\sum_{k \in \mathbb{Z}^n} |\alpha_{0k}|^p \Big)^{\frac{q}{p}} + \sum_{j=0}^{\infty} \bigg[2^{j[s+n(\frac{1}{2}-\frac{1}{p})]} \Big(\sum_{k \in \mathbb{Z}^n} \sum_{l=1}^{2^n-1} |\beta_{jk}^{[l]}|^p\Big)^{\frac{1}{p}} \bigg]^q \Bigg\}^{\frac{1}{q}} \;.
\end{equation}

The construction introduced in (\ref{eq:3}-\ref{eq:6}) above generates an MRA with an orthonormal wavelet basis 

\begin{equation}\label{eq:7}
    \{\varphi_{0\mu}, \mu \in \mathbb{Z}\} \cup \{ \psi_{j \nu}^{(l)}, j=0,1,..., \nu \in \mathbb{Z}^n, l=1,...,2^{n-1} \}
\end{equation}

A typical example of such compactly supported wavelets are the Daubechies wavelets \cite{daubechies10lect1992} which will be the ones used in the remaining part of this paper. It is possible to generalize this construction to generate a broader class of MRAs based on bi-orthonormal wavelets \cite{Dahmen1997WaveletAM}. These are of considerable interest in the case of polynomial spline-wavelets which are the type preferred in image processing for $n=2$ and surface processing for $n=3$. In this case it is imperative to use bi-orthonormal and not orthonormal spline-wavelets, because only in the proper bi-orthonormal case can the spline-wavelet be compactly supported. (Moreover, an additional advantage in image processing is that there exist proper bi-orthonormal spline-wavelets which are compactly supported and whose graphs are symmetric.) There are no MRAs with compactly supported orthonormal polynomial spline-wavelet bases. We intend to consider the use of bi-orthonormal compactly supported spline-wavelets in a subsequent publication dedicated to deep image learning.

An important property of (bi)-orthonormal MRAs which follows from (\ref{eq:3}) is that $j=0$ in (\ref{eq:4}) can be replaced by any $j_0 \in \mathbb{Z}$, such that (\ref{eq:5}) continues to hold true with $j=0$ replaced by $j=j_0$. In this case, (\ref{eq:6}) defines an equivalent norm in $B_{pq}^s (\mathbb{R}^n)$ for $p \geq 1$, $q \geq 1$ (quasinorm for $0 < p < 1$ and/or $0 < q < 1$) with equivalence constants dependant on $j_0$. (The concept of equivalent metrics continues to hold true for quasinorms, because quasinormed abelian groups are metrizable \cite[Section 3.10]{bergh1976} -- see also Section 3.

Consider now the above construction with $j_0 \in \mathbb{Z}$. Let $J \in \mathbb{Z}$ be such that $j_0 < J < \infty$ and consider the truncation $\sum_{j=j_0}^J$ of the series $\sum_{j=j_0}^\infty$ in (\ref{eq:5}) and (\ref{eq:6}). This defines a subspace $V_J \subset B_{pq}^s (\mathbb{R}^n)$ such that 

\begin{equation}\label{eq:8}
    V_J = \mathrm{span} \big( \{\varphi_{j_0\mu} : \mu \in \mathbb{Z}\} \cup \{\psi_{j \nu}^{[l]} : l=1,...,2^{n}-1, \nu \in \mathbb{Z}^n, j=j_0,...,J\}\big).
\end{equation}

Due to the properties of MRA, the following sequence of nested inclusions holds: 

\begin{equation}\label{eq:9}
    V_{j_0} \subset V_{j_{0+1}} \; \subset ... \subset V_J \subset V_{J+1} \subset ...
\end{equation}

with

\begin{equation}\label{eq:10}
    \overline{\bigcup^\infty_{j=j_0}} V_j = L_2(\mathbb{R}^n)
\end{equation}

holds where $\bar{x}$ is the topological closure in $Y$ of $X \subset Y$, where $Y$ is a complete topological space. (In the case of MRA, the complete topological space $Y$ is $L_2$ with respect to the topology induced by the inner product in $L_2$ or, equivalently, the norm in $L_2$.) Consider also the spaces $W_j=\mathrm{span} \{\psi_{jv}^{[l]} : l=1,...,2^{n-1}, \nu \in \mathbb{Z}^n \}$ where $j=j_0,...,J$.

By the properties of MRA, $f \in V_J$ admits two equivalent representations, as follows:

\begin{equation}\label{eq:11}
    \sum_{k_J \in \mathbb{Z}^n} \alpha_{J k_J} \varphi_{J k_J}(x) = f(x) = \sum_{k_{j_0} \in \mathbb{Z}^n} \alpha_{j_0 k_{j_0}} \varphi_{j_0 k_{j_0}} (x) + \sum_{j=j_0}^J \sum_{k_j \in \mathbb{Z}^n} \sum_{l_j=1}^{2^n-1} \beta_{j k_j}^{[l_j]} \psi_{j k_j}^{[l_j]} (x), x \in \mathbb{R}^n,
\end{equation}

where the equalities in (\ref{eq:11}) are in the sense of Lebesgue -- a.e. Invoking the introduced spaces $W_j$, (\ref{eq:11}) can be equivalently rewritten as 

\begin{equation}\label{eq:12}
    V_J = V_{j_0} \bigoplus W_{j_0} \bigoplus W_{j_{0+1}} \bigoplus ... \bigoplus W_J = V_{j_0} \bigoplus \bigoplus_{j=j_0}^J W_j .
\end{equation}

In applications involving processing of numerical data with sample size $N$, $J$ is chosen dependent on $N: J=J(N)$. Since $\int_{\mathbb{R}^n} \psi (x) dx = 0$, the usual selection of $J(N)$ is such, that the size of the support of $\psi_{J k_J}^{[l_J]}$ is comparable to the average step $h_N$ between adjacent data points: 

\begin{equation}\label{eq:13}
    \mathrm{diam} (\mathrm{supp} \: \psi_{J k_J}^{[l_J]}) \asymp h_N, \mathrm{where} \: h_N \asymp \frac{1}{N}, 
\end{equation}

with equivalence constants independent of $N$. In view of the definition of $\psi_{J k_J}^{[l_J]}$, (\ref{eq:13}) implies 

\begin{equation}\label{eq:14}
    J(N) \asymp \log_2 N
\end{equation}

with equivalence constants independent of $N$. 
With this selection, the father wavelet (scaling function) $\varphi_{Jk_J}$ acts as a consistent approximation of the Dirac $\delta$-function at the point $x$, as long as $x \in \mathrm{supp} \: \varphi_{Jk_J}$, and the rate of this approximation improves with the number of consecutive vanishing moments of $\psi$ additional to the condition $\int\limits_{\mathbb{R}} \psi (x) dx = 0$ needed for consistent approximation (these would be $\int\limits_{\mathbb{R}} x \psi (x) dx = 0$, $\int\limits_{\mathbb{R}} x^2 \psi (x) dx = 0$ etc.). With this selection, the coefficient $\alpha_{Jk_J}$ is taken to be equal to the value of $f$ at the point where $\varphi_{Jk_J}$ is concentrated as a $\delta$-function. 

So far, we have been considering (\ref{eq:3} - \ref{eq:12}) for $B_{pg}^s (\mathbb{R}^n)$ of functions defined on the whole space $\mathbb{R}^n$. This means that the subspaces $V_j$ and $W_j$, $j \in \mathbb{Z}$, are (countably) infinite-dimensional. To make the construction computationally feasible, in numerical applications we limit the consideration to only those $f \in B_{pq}^s (\mathbb{R}^n)$ which are \emph{compactly supported} with $\mathrm{diam} (\mathrm{supp} \: f)$ comparable to the diameter of the convex hull of the numerical data set (this numerical data set is finite, therefore its convex hull is a bounded subset of $\mathbb{R}^n$, so its closure is compact $(n < \infty))$. In this case, the subspaces $V_j, W_j, j \in \mathbb{Z}$ are all finite-dimensional, with dimensions depending on $j$, the distribution of the data set and its sample size $N$. 

Now, let us extend the consideration to include also $f \in B_{pq}^s (\Omega)$, where $\Omega \subset \mathbb{R}^n$ is a nonvoid compact hyper-rectangle, i.e., 

\begin{equation}\label{eq:15}
    \Omega = \overset{n}{\underset{i=1}{\times}} [a_i, b_i]
\end{equation}

where $"\times"$ indicates Cartesian product, and $-\infty < a_i < b_i < +\infty, \: i=1,...,n$. The construction (\ref{eq:3}-\ref{eq:12}) continues to hold also in this case which is very similar to the case of $f \in B_{pq}^s (\mathbb{R}^n)$ such that it is compactly supported, with $supp \: f$ contained in the closure of the convex hull of the numerical data set. In the present case, the finite dimensions of $V_j, W_j$ depend on $j$, the sizes of $b_i-a_i, \: i=1,...,n$, and the sample size $N$. Moreover, $j_0 \in \mathbb{Z}$ is bounded from below by $diam(\Omega)$: 

\begin{equation}\label{eq:16}
    j_0 \asymp \log_2 \mathrm{diam} (\Omega),
\end{equation}

with constants of equivalence possibly dependent on $n$, but not on $N$. There is one notable technical modification: the orthonormal wavelet bases in $V_j$ and $W_j$, $j=j_0,j_0+1,...,J$, contain \emph{boundary-corrected} wavelet basis functions \cite{cohen93}, \cite{cohen93_1}. For the theory of deep learning of $n$-dimensional manifolds developed here, there is no principal difference between the case of compactly supported $f \in B_{pq}^s (\mathbb{R}^n)$ and the case of $f \in  B_{pq}^s (\Omega)$ with $\Omega$ a compact hyper-rectangle in $\mathbb{R}^n$. Therefore, in our application we shall focus on the former one of these two cases, to avoid the construction of boundary-corrected wavelets. 

Now, we are in a position to provide a definition of WNNs which is equivalent to the original definition in \cite{Zhang1992-wt} and \cite{Alexandridis2013-gb}, but is in a new form which allows an insightful comparison with the new type of WBNNs.

Consider the left-hand side (LHS) of the identity in (\ref{eq:11}). Define first a 1-layer WNN with its $Jk_J$-th node being a neuron processing the $\alpha_{Jk_J}$-th coefficient in the expansion in the LHS of (\ref{eq:11}). The edges of the WNN's graph are only the ones connecting the input to the $Jk_J$-th neuron and the ones connecting the $Jk_J$-th neuron to the output neuron where received results are summed up. In contrast to this construction, repeat the 1-layer NN but with 1--1 correspondence to the $\alpha$-coefficients in $V_{j_0}$ and the $\beta$-coefficients in $W_j, j=j_0,...,J$ (in the right-hand side (RHS) of (\ref{eq:11})). The definition of the edges of the graph of the 1-layer WBNN is the same as with the previous 1-layer WNN construction. The widths of the so-defined 1-layer WNN and WBNN are, of course, the same. A crucial advantage of the WBNN layer is its telescopic ordering which incorporates \emph{the wavelet depth} into the \emph{neural width} of the WBNN layer. Adding \emph{neural depth} to the 1-layer WNN and WBNN is done in one and the same way: the next layers are added as intermediate between the already defined 1st layer and the output neuron, and each intermediate layer has exactly the same structure and ordering as the 1st layer.

As we shall see in the next sections, the learnability conditions and universal approximation theorems for each of WNNs and WBNNs ensure that 1-layer NNs (with the widths specified via the LHS and RHS of (\ref{eq:11}) and the compactness of $\mathrm{supp} \: f$) are sufficient for learning every element of the range of the respective approximation theorem. From this point of view, deep WNNs and WBNNs are theoretically redundant, but as we shall see, they provide a highly efficient computing architecture for acceleration of the rate of convergence of the approximation process by using iterative algorithms. A maximally sparse structure of the edges between the $l$-th and the $(l+1)$-st layer should be used (a neuron on the $(l+1)$-st level is only connected with its corresponding neuron at level $l$ by way of the 1--1 correspondence between levels $l$ and $l+1$). In practice, in the context of WNNs, the intermediate levels of the NN are used for iterative local optimization starting with the initial approximation provided at level 1. Although the performance of the deep WNN is expected to be better than the one of the 1-layer WNN, this can be expected to be noticeable only for very large sample sizes and respective very large number of iterations (very deep WNN). For example, in the case of Fig. 4 of \citep{Zhang1992-wt} the number of iterations is 10000. In the case of WBNNs, the quality of initial approximation is expected to be very high, due to the efficient use of the wavelet depth within the layer of neural depth 1. As a consequence, in comparison with the very large sample size needed for acceptable performance of WNN, it can be expected for the initial approximation of the 1-layer WBNN to be sufficiently good for large to medium samples sizes, while as the ultimate approximation achieved by a deep WBNN (with the 1-layer WBNN as its initial layer) to be sufficiently good already for medium to small samples.

Notice the distinction we make between \emph{learnability conditions} and \emph{universal approximation theorems} for a given type of single-layer NN computing $f: \mathbb{R}^n \rightarrow \mathbb{R}$ for a given $n \in \mathbb{N}$. (To study the general case of parametric manifolds on $\mathbb{R}^n$, i.e., $f:\mathbb{R}^n \rightarrow \mathbb{R}^m$, $m \in \mathbb{N}$, $n \in \mathbb{N}$, it is sufficient to study it coordinate by coordinate, i.e., for $m=1$ only.)

A \emph{universal approximation theorem} (UAT) for a single-layer NN of width $N$ is a \emph{qualitative consistency} result (i.e., refers to existence of convergence in a given topology without specifying \emph{quantitative rates of convergence}) when $N \rightarrow +\infty$ under the assumption of an \emph{activation function} of specific type being used in the neural computation. For example, in the case of sigmoid activation, the respective UAT is due to Cybenko \citep{cybenko89}, \citep{cybenko92} \footnote[1]{In \citep{Zhang1992-wt} Cybenko's work has been imprecisely and incompletely cited. 
Here we provide the relevant corrected and complete citation \citep{cybenko89}, \citep{cybenko92}}, and refers to continuous functions, while in the case of \emph{Rectified Linear Unit} (ReLU) activation \citep{lu2017}, the respective UAT refers to the more general class of functions in $L_1$ (see \citep[Theorem 1]{lu2017}. 
The \emph{learnability set} (LS) for a given single-layer NN of width N is the largest set of $f: \mathbb{R}^n \rightarrow \mathbb{R}$ which can be approximated by neural computation via this NN when $N \rightarrow + \infty$ \emph{without the invocation of a specific activation function}, i.e., when \emph{activation is via the default identity function}. 

As mentioned above, both LSs and UATs are qualitative consistency results. In the next sections we shall show that it is possible to obtain quantitative upgrades of LSs and UATs where the consistency results are strengthened to results about concrete rates of approximation.

\section{A new result about WNNs}\label{s2}
In this section we shall show how the idea of using equivalent metric (as discussed in Section 1) can be used to generate an initial solution in a single-layer WNN which is a good starting point for local optimization in a deep WNN (having as 1st layer the said single-layer WNN). 
The universal approximation theorem invoked in \citep{Zhang1992-wt} is Cybenko's classical result, \citep{cybenko89}, \citep{cybenko92}, valid for continuous functions $f$. 
As noted in \citep{Alexandridis2013-gb}, after the publication of \citep{Zhang1992-wt} more general learnability conditions and universal approximation theorems were derived about WNNs. 
Namely, results about UATs for ReLU NNs were formulated and proved in \citep{lu2017} in relevance to the larger space of Lebesgue-integrable functions $f \in L_1 (\mathbb{R}^n)$.
The ideas of the proofs of the results in \citep{lu2017} make it possible to identify the LS for sufficiently wide NNs, including the case of WNNs with (\ref{eq:14}).

To present this result, here we shall use terminology which will allow us to compare this result to the respective result for WBNNs (derived in section 3 below). 
Namely, in the case of wavelets on $\mathbb{R}^n$, learnable by sufficiently wide single-layer WNN are all \emph{regular distributions} $f$ in $\mathcal{S}^{\prime} (\mathbb{R}^n)$ - the dual of the Laurent Schwartz space $\mathcal{S} (\mathbb{R}^n)$ \citep{reed1980}. 
In the case of boundary-corrected wavelets on $\Omega$-compact hyper-rectangle in $\mathbb{R}^n$ (see Section \ref{intro}), learnable by sufficiently wide single-layer WNN are all regular distributions $f \in D^{\prime} (\Omega)$ - the dual of the space $D (\Omega)$ \citep{reed1980} \footnote[2]{In \citep{reed1980} a slightly different notation is used: $D_\Omega$ instead of $D (\Omega)$ and $D^\prime_\Omega$ instead of $D^\prime (\Omega)$. }. 
In both cases, the regular distributions $f$ are exactly the elements of $L_{1, loc}(\Delta)$, $\Delta=\mathbb{R}^n$ or $\Delta=\Omega \subset \mathbb{R}^n$ (for $\Omega$ see Section \ref{intro}). Here, as usual, $L_{1,loc}$ consists of all $g$ defined Lebesgue -- a.e. on $\Delta$ and such that for every compact subset $C \subset \Delta$ the statement $f \chi_C \in L_1 (\Delta)$ holds true, where $\chi_C$ is the characteristic function of $C:$

$\chi_C (x)=1$ for 
$x \in C$ and 
$\chi_C (x) = 0$ for 
$x \in \Delta \setminus C$

In the case $\Delta=\Omega$, $\Omega$ considered in Section \ref{intro} is a compact in $\mathbb{R}^n$, therefore, $L_{1,loc} (\Omega) = L_1 (\Omega)$. 
Notice that in the wavelet context the width of WNN is exponential (solving (\ref{eq:14}) for N yields the equivalent $N \asymp 2^J$) and sufficiently large for $L_{1,loc}$ to be learnable via single-layer WNN, according to \citep{lu2017}. 
Therefore, 'deepening' the WNN does not result in increasing the set of learnable functions.
However, deep WNN may offer the following advantage: while the universal approximation theorem for sufficiently wide single-layer WNN provides only consistency of the approximation (i.e., convergence exists but may be arbitrarily slow), the use of a deep WNN with the said single-layer WNN forming its 1st layer may result in \emph{increasing the speed (rate) of approximation}. 

Now we shall formulate a model problem for which we shall be able to automatically generate a single-layer WNN that is asymptotically optimal with respect to the paradigm based on (\ref{eq:1}) in Section \ref{intro}. 
According to this paradigm, local optimization with starting point at the automatically generated solution at the 1st layer of a deep WNN is expected to provide the global minimum for asymptotically large sample sizes $(N \rightarrow + \infty)$. 

\textbf{Model problem.} Let $n=1$. 
From a random sample with size $N$, learn the density $f$ of an absolutely continuous cumulative distribution function $F: \mathbb{R} \rightarrow [0,1]$. In this case, $f=F^\prime$, $f \in L_1(\mathbb{R})$, $f \geq 0 \: \mathrm{on} \: \mathbb{R}$, $\int_{\mathbb{R}}|f(x)|dx=\int_\mathbb{R}f(x)dx=1$.

This problem was addressed in \citep{dech97} and \citep{dechevsky98}, as follows. The density $f$ was approximated by a (father) wavelet expansion, using the basis of a frame more general than a biorthonormal upgrade of LHS in (\ref{eq:11}). To consider the construction in \citep{dech97} and \citep{dechevsky98} strictly in our present context, it is necessary to consider only those particular cases for which the frame is orthonormal, thus corresponding to the LHS in (\ref{eq:11}). In \citep[Remark 2.3.2]{dechevsky98} were identified all cases when the frame used in \citep{dech97}, \citep{dechevsky98} is orthonormal: namely, these are exactly the cases in the LHS of (\ref{eq:11}) where the scaling function $\varphi$ in (\ref{eq:11}) is of Haar type, i.e., the normalized characteristic function 

\begin{equation}\label{eq:17}
    \varphi (x) = \frac{1}{a} \cdot \chi_{[-\frac{a}{2},\frac{a}{2}]} (x - x_0)
\end{equation}

where $a > 0, x_0 \in [-\frac{a}{2},\frac{a}{2}]$, and \emph{a} is chosen so, as to match with the selection of $j_0$ in the RHS of (\ref{eq:11}). 

With this choice of $\varphi$ in the LHS of (\ref{eq:11}), the random estimator of $f$ is obtained by plugging in the LHS of (\ref{eq:11}) the \emph{empirical density}

\begin{equation}\label{eq:18}
    \hat{f}_N (x) = \frac{1}{N} \sum_{i=1}^N \delta (x-x_i)
\end{equation}

where $\delta(\cdot)$ is the delta-function and $\{x_i, \: i=1,...,N\}$ is the sample data set. Here, as earlier, the selection of the level $J$ in (\ref{eq:11}) is such that (\ref{eq:14}) holds. In view of (\ref{eq:18}), the $Jk_J$-th neuron in the WNN associated with the basis function $\varphi_{Jk_J}$ in the LHS of (\ref{eq:11}) computes the \emph{empirical coefficient} 

\begin{equation}\label{eq:19}
    \hat{\alpha}_{Jk_J}=\frac{1}{N} \sum_{i=1}^N \varphi_{Jk_J} (X_i),\: k_J \in \mathbb{Z}
\end{equation}

In \citep{dechevsky98} the selection of metric in which the risk is measured was relevant to the expectation of "the generalized Cram\'{e}r-von Mises loss" \citep[sections 2.2 and 2.3]{dechevsky98}.
One very valuable feature of the estimates of this risk obtained in \citep{dechevsky98} was that they revealed the precise interconnection between the density's smoothness and the weight of its tails as $x \rightarrow \pm \infty$. In our present study we made the natural assumption of compactness of the density's support in correspondence with the compactness of the convex hull of the sample data set.
For densities with compact support with fixed diameter of the support, the risk estimates in \citep{dechevsky98} simplify and are only dependent on the smoothness of the density, and the following new result holds true.

\begin{theorem}\label{th:1}
    Under the above assumption about the compact support of the density $f$, assume $f \in B_{p\infty}^s (\mathbb{R}), \; 0 < p \leq \infty,\; 0 < s < 2$, and (\ref{eq:14}) holds. Let the single-layer WNN associated with the Haar-type basis (\ref{eq:17}) used in (\ref{eq:11}) be defined as above, with (\ref{eq:19}) holding true. Let the risk of estimating $f$ via the empirical density $\hat{f}_N$ in (\ref{eq:18}) be defined as in \citep[Section 2.2]{dechevsky98}. Then,
    \begin{enumerate}
        \item The risk $R(f, \hat{f}_N)$ of learning $f$ by the neural computations (\ref{eq:19}) in the LHS of (\ref{eq:11}) is: 
        \begin{equation}\label{eq:20}
            R(f, \hat{f}_N) \asymp N^{-\frac{s}{1+2s}}
        \end{equation}
        with constants of equivalence dependent on $s$ and the fixed size of the support of $f$, but not on the choice of $f \in B_{p\infty}^s(\mathbb{R})$.
        \item The rate in the RHS of (\ref{eq:20}) is \emph{asymptotically optimal} in the sense of \citep[Theorem 2, 3.2]{dechevsky98}.
    \end{enumerate}
\end{theorem}

As a consequence of Theorem \ref{th:1}, under its assumptions, the single-layer WNN computing (\ref{eq:19}) generates automatically an element of $V_J$ in (\ref{eq:12}) which is a good starting point for optimization search in $V_J$ when the sample size $N$ (and $J \asymp \log_2N$) is asymptotically large $(N \rightarrow \infty)$. 
For such $N$ and $J$, using a deep WNN upgrade of the single-layer WNN (with the latter being the 1st layer of the deep WNN) it is possible to obtain an essential improvement of the learning of $f$ within $V_J$ and, possibly, even obtain globally optimal solution of the iterative optimization search in $V_J$ performed by the deep WNN computing architecture. Thus, we have provided an instance when the equivalent-metric paradigm based on (\ref{eq:1}) in Section \ref{intro} offers an efficient AI alternative of the use of natural intellect in designing good starting point of deep WNN-compatible iterative optimization search, as discussed in \citep{LeCun2015-iu} -- see Section \ref{eq:1}. 
Although results of the type of Theorem \ref{th:1} provide a rigorous mathematical justification of the use of AI based on deep WNNs, in practice, notable improvement can be generally expected only, or almost only, for very large samples sizes $N$.

Let us note two additional new features in Theorem \ref{th:1} and its proof. 
\begin{itemize}
    \item Theorem \ref{th:1} shows that the optimal estimation rate can be achieved when the activation function is the default identity (i.e., for the empirical density).
    \item The quantitative result involving rates is achieved in Theorem \ref{th:1} under less restrictive assumptions for WNNs than the assumptions on generic NNs for the qualitative universal approximation theorem in \citep[Theorem 1]{lu2017} in the sense that the latter NN must be \emph{fully connected} (i.e., with \emph{densely distributed edges} between the nodes of the NN) while the former WNN has very sparse edge distribution.
\end{itemize}

In the remaining part of this exposition, we shall show that the alternative of using WBNNs associated with the RHS of (\ref{eq:11}) and (\ref{eq:12}) provides highly efficient AI algorithms with quality practical results achieved already for medium to small sample sizes $N$.

\section{WBNNs: learnability and universal approximation}\label{s3}

To study learnability conditions and universal approximation theorems in the case of WBNNs it will be necessary to study some properties of the scale of Besov spaces $B_{pq}^{s} (\mathbb{R}^n)$ as defined via (\ref{eq:5}, \ref{eq:6}). (The case of $B_{pq}^{s} (\Omega)$ using boundary-corrected wavelets in (\ref{eq:5}, \ref{eq:6}) can be studied, mutatis mutandis, but our focus will continue to be on the case $\Omega = \mathbb{R}^n$.) To study the necessary aspects of the properties of the Besov-space scale $\{ B_{pq}^s (\mathbb{R}^n), \: 0 < p, q \leq \infty, s \in \mathbb{R} \}, n \in \mathbb{N}$, we need some preparation, as follows. 

\begin{itemize}
    \item For the concept of \emph{quasinorm} (or $c$-norm ($c$-quasinorm) with $c \geq 1$ being the constant in the quasi-triangle inequality), we refer to \citep[Section 3.10]{bergh1976}. See also there for the definition of \emph{quasinormed abelian groups}.
    \item By \citep[Lemma 3.10.1]{bergh1976}, a quasinormed abelian group A with c-quasinorm ${||.||}_A, :\ c \geq 1$, is \emph{metrizable}, in the sense that the $\rho$-power $A^\rho$ of $A$ with 1-quasinorm ${||.||}_A^\rho, \: 0 < \rho = \frac{1}{1+\log_2 c} \leq 1$, is a metric space with $d(a,b)={||a-b||}_A^\rho$.
    \item A necessary and sufficient condition for a normed space $A$ to be \emph{complete} (i.e. for $A$ to be a \emph{Banach space} is given in \citep[Lemma 2.2.1]{bergh1976}.
    \item The previous item is being generalized in \citep[Lemma 3.10.2]{bergh1976} to a necessary and sufficient condition for a quasinormed abelian group $A$ to be \emph{complete}. (If A is not only abelian group, but also a vector space, then, endowed with the quasinorm ${||.||}_A$, it is called \emph{quasinormed space} and, if it is also complete, \emph{quasi-Banach space}.
    \item For example, consider the vector space $l_2$ of all sequences: $=(x_1,...,x_n,...), n \in \mathbb{N}, x_j \in \mathbb{R}$ or $x_j \in \mathbb{C}, j \in \mathbb{N}$, with quasinorm ${||x||}_{l_r} = ( \sum\limits_{j=1}^{\infty} {| a_j |}^r)^{\frac{1}{r}}$, $0 < r < \infty$, or $||x||_{l_\infty} = \max\limits_{j \in \mathbb{N}} |a_j|$, $r=\infty$.
    By the theory in \citep[Section 3.10]{bergh1976}, $l_r$ is a \emph{Banach space} for $r: 1 \leq r \leq \infty$ and \emph{quasi-Banach space} when $r: 0 < r < 1$. In the latter case, the constant $c$ in the quasi-triangle inequality for ${||.||}_{l_r}$ is $c = 2^{\frac{1-r}{r}} > 1$; for the power $\rho$ one gets $\rho = r \in (0,1)$ and the 1-quasinormed abelian group ${(l_r)}^r$ with 1-quasinorm ${||.||}_{l_r}^{r}$ is a metric space with metric $d(a,b) = {||a-b||}_{l_r}^r$.
    \item Using the properties of $l_r$ from the previous item, it is possible to establish that $B_{pq}^s$, as defined via (\ref{eq:5}, \ref{eq:6}), are Banach spaces for $1 \leq p \leq \infty$, $1 \leq p \leq \infty$, and quasi-Banach spaces when $0 < p < 1$ and/or $0 < q < \infty$ \citep{triebel83}.
    \item Another aspect of the theory of the Besov-space scale which proves to be relevant is \emph{the lifting property} of \emph{the Bessel potential} $J^\sigma$, $\sigma \in \mathbb{R}$, in the Besov-space scale. 
    Following the exposition in \citep[Section 1.2.1]{triebel83}, we define the Fourier transform $\mathcal{F}$ and its inverse $\mathcal{F}^{-1}$ first on $\mathcal{S}(\mathbb{R}^n)$, and then extend it to $\mathcal{S}^\prime (\mathbb{R}^n)$, after which, following \citep[Section 2.3.8]{triebel83}, we define the Bessel potential $J^\sigma : \mathcal{S}^\prime (\mathbb{R}) \rightarrow \mathcal{S}^\prime (\mathbb{R}^n)$, as follows:
    
    \begin{equation}\label{eq:21}
        J^\sigma f = \mathcal{F}^{-1} [(1+{|.|}^2)^{-\frac{\sigma}{2}} \mathcal{F} f], f \in \mathcal{S}^\prime (\mathbb{R}^n).
    \end{equation}

    where $\sigma \in \mathbb{R}$. 
    Now the $\sigma$-\emph{lifting property} of the Bessel potential in the Besov-space scale can be formulated, as follows \citep[Section 2.3.8]{triebel83}.
    $J^\sigma$ acts bijectively on $\mathcal{S}^\prime (\mathbb{R}^n)$ and the restriction of $J^\sigma$ on $\mathcal{S} (\mathbb{R}^n)$ acts bijectively on $\mathcal{S} (\mathbb{R}^n)$.
    Moreover, if $s, p, q$ are as in (\ref{eq:5}, \ref{eq:6}) and $f \in B_{pq}^{s_1} (\mathbb{R}^n)$, where $s_1 \in \big( -\infty, n {(\frac{1}{p} - 1)}_+ \big] \cup \big[ r, \infty \big)$ with $\sigma : s = s_1 + \sigma$, then $J^\sigma f \in B_{pq}^s$ and formulae (\ref{eq:5}, \ref{eq:6}) can be applied on $g = J^\sigma f$ and 
    
    \begin{equation}\label{eq:22}
        {||J^\sigma f||}_{B_{pq}^s (\mathbb{R}^n)} \asymp {||f||}_{B_{pq}^{s - \sigma}(\mathbb{R}^n)}
    \end{equation}
\end{itemize}

Moreover, using (\ref{eq:22}) when $s_1 \in (- \infty, n {(\frac{1}{p} - 1)}_+ ] $, i.e., for $\sigma > 0$ allows to extend the wavelet-based representation (\ref{eq:5}) and the quasinorm definition (\ref{eq:6}) for arbitrary $s \in \mathbb{R}$, i.e., including also singular distributions like the $\delta$-function and its derivatives which are not in $L_{1,loc}$ \footnote[1]{This fact is relatively easy to derive even in the $n$-dimensional case, using the theory of Fourier multipliers on $L_{p} (\mathbb{R}^n)$, $1 \leq p \leq \infty$, see \citep[Section 6.1]{bergh1976}, \citep[Chapter 1]{brenner1975} and \citep[Introduction]{hairer2017}.}. 
Indeed, choose and fix any $s_1 \in \mathbb{R}$ and select and fix $\sigma$ such that $n {( \frac{1}{p} - 1)}_{+} < s - \sigma < r$. Then (\ref{eq:5},\ref{eq:6}) will make sense for $f$ replaced by $J^\sigma f$ and (\ref{eq:22}) can be used to define an equivalent quasinorm in $B_{pq}^{s_1} (\mathbb{R}^n)$.

Now we are ready to formulate and prove the following results about WBNNs.

\begin{theorem}\label{th:2}
    Using arbitrary samples with size $N$ with $J(N)$ satisfying (\ref{eq:14}), $f \in \mathcal{S}^\prime (\mathbb{R}^n)$, is learnable for $N \rightarrow \infty$ by WNNs if, and only if (iff) $f$ is also learnable by WBNNs, i.e. the learnability sets by WNNs and by WBNNs coincide.
\end{theorem}

\begin{theorem}\label{th:3}
    Let $N$ and $J(N)$ be as in Theorem \ref{th:2}, and let $f \in B_{pq}^s (\mathbb{R}^n)$, $0 < p \leq \infty$, $0 < q \leq \infty$, $s \in \mathbb{R}$.
    Then, for any $r: n {(\frac{1}{p}-1)}_{+} < r < \infty$ and orthonormal wavelet basis satisfying (\ref{eq:3}, \ref{eq:4}) and for every $\sigma \in \mathbb{R}$ such that $s - \sigma \in (n {(\frac{1}{p} - 1)}_{+}, r)$ it holds true that $J^\sigma f$ is learnable by the WBNN generated by the said wavelet basis.
\end{theorem}

\begin{corollary}\label{cor:1}
    The space $B_{pq}^{s} (\mathbb{R}^n)$ is contained in the learnability set of WBNN for every $s \in \mathbb{R}$, $0 < p \leq \infty$, $0 < q \leq \infty$.
\end{corollary}

Corollary \ref{cor:1} implies that the learnability set of WBNNs contains not only all regular distributions in $\mathcal{S}^\prime (\mathbb{R}^n)$, but also singular distributions, since Besov spaces with $s < 0$ do contain singular distributions.

Theorem \ref{th:2} now suggests that Corollary \ref{cor:1} extends also to WNNs, but here lies one big difference between the \emph{efficient} use of WBNNs and WNNs.
Recovering $f$ from $g = J^\sigma f$ requires approximate numerical computation of $J^{-\sigma} g$ which is very numerically sensitive to errors in the computation of $g$ especially when $f$ can be a singular distribution.
Since for a given sample with size $N$ the quality of learning $g$ via WNNs is expected to be much worse than via WBNNs, the deterioration of the recovery of $f$ from $g$ when using WNNs would be much more exacerbated compared to the use of WBNNs so that the only case of $\sigma$ for which the use of WNNs could be marginally acceptable is the trivial case $\sigma=0$. 
(A detailed error analysis of the computations for $\sigma \neq 0$ would require the invocation of aspects of Paley-Wiener theory \citep[Sections 4--6]{frazier1991}, including sampling results of Shannon type \citep[Theorem 6.4]{frazier1991} which goes beyond the study of AI aspects considered here.)
Theorems \ref{th:3} and \ref{th:2} now imply that WNNs can be efficiently used (albeit only marginally for very large sample sizes $N$ only) for learning $f \in B_{pq}^{s} (\mathbb{R}^n)$ only for the original range $s, p, q$ for which (\ref{eq:5}) was formulated.
Note that for the these values of $s, p ,q \: \: \: B_{pq}^{s} (\mathbb{R}^n) \subset L_{1,loc}(\mathbb{R}^n)$.

The results obtained here for wavelet bases on $\mathbb{R}^n$ can in principle be reformulated for boundary-corrected wavelets on a compact hyperrectangle in $\mathbb{R}^n$, this modification is technically involved. For example, the lifting property of the Bessel potential has to be replaced by a respective property of fractional integro-differential operators of Riemann-Lionville, Gr\"{u}nwald-Letnikov, Caputo and other types under additional assumptions for each of these types \citep{samko1993}.

As far as UAT for WBNNs is concerned, it is much more rich and diversified than UAT for WNNs, due to the much more flexible telescopic structure of the RHS in (\ref{eq:12}).
While in the case of WNNs the focus has been only on sigmoid and ReLU activation, in the case of WBNNs there is a great variety of meaningful activation methods, each of which is with its own UAT and own range of practical applications.
In this section we study the common features of all these activation methods and provide a classification of these methods into two general subclasses, together with the general range of applications for each of these subclasses.

Any activation method can be defined as a (generally, \emph{nonlinear}) operator $\Lambda$ acting on the space sum in the RHS of (\ref{eq:12}) and being of \emph{shrinkage} type, i.e. having the following properties.

\begin{enumerate}
    \item The restriction of $\Lambda$ on $V_{j_0}$ coincides with the identity on $V_{j_0}$, i.e., 
    
    \begin{equation}\label{eq:23}
        \Lambda (\alpha_{j_0 k_0} \varphi_{j_0 k_0}) = \alpha_{j_0 k_0} \varphi_{j_0 k_0}
    \end{equation}
    
    for every $j_0, k_0,...$.
    \item Using the Euler representation of $z \in \mathbb{C}$
    
    \begin{equation}\label{eq:24}
        z = |z| (\cos ( \arg z ) + i \sin ( \arg z )), \arg z \in [0, 2 \pi),
    \end{equation}

    the  action of $\Lambda$ on the space $\bigoplus_{j=j_0}^J W_j$ in the RHS of (\ref{eq:12}) is defined such, that

    \begin{equation}\label{eq:25}
        \Lambda (\beta_{j k_j}^{[l_j]} \psi_{j k_j}^{[l_j]} ) = \Tilde{\beta}_{j k_j}^{[l_j]} \psi_{j k_j}^{[l_j]}
    \end{equation}

    where 

    \begin{equation}\label{eq:26}
        \arg \Tilde{\beta}_{j k_j}^{[l_j]} = \arg \beta_{j k_j}^{[l_j]} 
    \end{equation}

    \begin{equation}\label{eq:27}
        |\Tilde{\beta}_{j k_j}^{[l_j]}| \leq |\beta_{j k_j}^{[l_j]}|
    \end{equation}

    for every ($j$, $k_j$) participating in the formation of $\bigoplus_{j=j_0}^J W_j$.

    Notice that when $\beta_{j k_j}^{[l_j]} \in \mathbb{R}$ (\ref{eq:26}) reduces to 

    \begin{equation}\label{eq:28}
        \sign \Tilde{\beta}_{j k_j}^{[l_j]} = \sign \beta_{j k_j}^{[l_j]}
    \end{equation}

    where for $x \in \mathbb{R} \backslash \{ 0 \}$

    \begin{equation}\label{eq:29}
        \sign x = 
        \begin{cases}
          +1, & x > 0 ; \\
          -1, & x < 0 ; \\
          \textrm{undefined}, & x = 0 ;
        \end{cases}   
    \end{equation}

    and for the case $z=x=0$ it is convenient to define 
    
    \begin{equation}\label{eq:30}
        \arg 0 = \sign 0 = 0
    \end{equation}
    
\end{enumerate}

Clearly, $\Lambda$ has the special property that it preserves the directrice and respective orientation of every basis function in $V_{j_0}$ and $\bigoplus_{j=j_0}^{J} W_j$.
It is also clear that, in general, $\Lambda$ is nonlinear, since the shrinkage is individual for every basis function.
Now, we divide all possible activation methods $\Lambda$ with properties (i) and (ii) into two disjoint subclasses, as follows.

A. \emph{Non-threshold-type} activation methods have the following additional property

\begin{enumerate}[resume]
    \item for any selection of the coefficient vector $\{ \alpha_{j_0 k_0}, \beta_{j k_j}^{[l_j]}\}$ in the RHS of (\ref{eq:11}), such that 
    $\beta_{j k_j}^{[l_j]} \neq  0$ for some triple $(j, k_j, l_j)$, it is fulfilled that $\Tilde{\beta}_{j k_j}^{[l_j]} \neq 0$ holds true. (In other words, there is only reducing $|\beta_{j k_j}^{[l_j]}| > 0$ without ever "killing" the coefficient $\Tilde{\beta}_{j k_j}^{[l_j]}$, i.e., having $|\Tilde{\beta}_{j k_j}^{[l_j]}| = 0$.

\end{enumerate}

B. We shall say that the activation method $\Lambda$ is of \emph{threshold-type}, if (iii) is not fulfilled for $\Lambda$.

In the second part of this study we shall study an important model example of activation of WBNNs resulting in \emph{learning geometric manifolds with compression}. 
The analysis of concrete examples will show that there is an intrinsic separation of geometric manifolds into ones that are highly compressible and ones that are not.
From a geometric point of view, the latter class of manifolds will be identified as \emph{fractal-type} while the former class consists of manifolds of \emph{piecewise-smooth type}.
Our forthcoming study \citep{llhm2022} of diverse activation methods of both threshold and non-threshold type (type B and A) will demonstrate that activation of threshold type is performing well when learning piecewise-smooth manifolds, while activation of non-threshold type performs well when learning manifolds of fractal type.

\section{Particle vs multiagent simulation and swarm vs. deep evolutionary AI}\label{s4}
A very new research topic in AI research is establishing connection between \emph{swarm} AI and deep neural networks by the invocation of \emph{evolutionary algorithms} \citep{iba2018}, \citep{iba2022}.
Scientifically, this is a very new field, but conceptually it appeared in some of the most famous early sci-fi novels \citep{hoyle1957}, \citep{lem1964} (latest English translation \citep{lem2020}) which were written in the first few years after the concept of AI emerged as a term at the workshop "Dartmouth Summer Research Project on Artificial Intelligence" at Darthmouth College, Hannover, NH, USA in 1956, to designate a specialized branch of cybernetics.
Both of the novels of Fred Hoyle and Stanis{\l}aw Lem successfully predicted the development of important modern scientific trends: the latter -- nanotechnology and swarm AI; the former -- deep learning by AI systems and connections with evolutionary algorithms.

Our present interest to the connection between swarm and deep evolutionary AI is only limited to its computational aspects.
From this limited point of view, the above-said connection can be considered as a particular case of particle simulation and multiagent simulation (i.e., simulation of systems involving large numbers of relatively simple agents vs. simulation of systems involving small to moderate number of agents with relatively high level of individual intelligence features).
Notice that the most efficient simulation of each of these two types of system is performed on different types of parallel computing architectures.

\begin{enumerate}[label=(\alph*)]
    \item Swarm of sufficiently broad single-layered NNs (including single-layered WBNNs with (\ref{eq:14})) -- CPU parallelism -- (relatively expensive) large-size multi-CPU supercomputing architectures; e.g., hypercubic \citep{leighton1992}.
    \item Deep (multi-layered, sufficiently broad) NNs (including deep WBNNs with (\ref{eq:14})) -- GPU parallelism -- (relatively cheap) small-size multi-GPU computing architectures using GPGPU-programming -- currently in CUDA, and more recently, Python \citep{cupy_learningsys2017}.
\end{enumerate}

Using connections between swarm intelligence and deep NNs \citep{iba2018}, \citep{iba2022}, it is possible to emulate the performance of the expensive computing architectures in item (a) by the cheap computing architectures in item (b), but at the inevitable price of some loss of efficiency.
Ideally, hybrid multi-CPU multi-GPU computing architectures should be recommended.

\section{Best activation of WBNNs for fastest learning and maximal compression}\label{s5}
In \citep{dechevsky1999} was considered and systematized a rich diversity of threshold and non-threshold wavelet shrinkage methods for \emph{nonparametric statistical estimation} of \emph{densities} and \emph{denoising of nonparametric regression functions}.
In \citep{llhm2022_1} we shall show that each of these shrinkage methods gives rise to respective activation of WBNNs, generating highly efficient learning algorithms.
Moreover, in some cases these learning algorithms can be shown to be \emph{best possible} with respect to certain aspects which are important for applications.
As a model example of the high efficiency of learning with WBNNs, we shall study here the activation induced by only one of these wavelet shrinkage methods, namely, the one discussed in \citep[Appendix B10 b)]{dechevsky1999}.

For the sake of maximal clarity, we shall consider here only the simplest univariate case $n=1$.
This will be a very clear illustration of the optimal speed of learning and compression in model examples of curve learning in the next sections.
The general case $n \in \mathbb{N}$ and some graphical visualization for the cases $n=2$ and $n=3$ will be considered in \citep{llhm2022_1}.

Assume $f \in B_{p q}^{s} (\Omega)$ where $\Omega = \mathbb{R}$ or $\Omega=[a,b]$, $-\infty < a < b < \infty$, $0 < p \leq \infty, 0 < q \leq \infty$ and ${(\frac{1}{p}-1)}_{+} < s < r$. 
Assume also that both the metric power index $p_1$, the metric logarithmic index $q$ and the smoothness index $s$ are \emph{exact}, that is, $f \notin B_{p_1 q_1}^{s_1} (\Omega)$ for any $p_1: 0 \leq p_1 < p, 0 < q_1 \leq q$ and any $s_1: s_1 > s$.

As usual in our present study, when considering the domain $\mathbb{R}$, we shall be making the default assumption about compactness of $\supp f$ (in the case of boundary-corrected wavelets and $\Omega = [a,b]$ with $-\infty < a < b < +\infty$, we do not need this default assumption, i.e., it is possible that $f (a+) \neq 0$ and/or $f (b-) \neq 0$).
For the index triple $(p, q, s)$ consider now the \emph{Sobolev embedding plane} passing through the point $(p, q, s)$, i.e., 

\begin{equation}\label{eq:31}
    \{ (\rho, \eta, \sigma): \sigma-\frac{1}{\rho} = \tau (p, s) = s - \frac{1}{p}, 0 < \rho \leq \infty, 0 < \eta \leq \infty \}
\end{equation}

What is important about this selection is the well-known embedding

\begin{equation}\label{eq:32}
    B_{p q}^{s} (\Omega) \hookrightarrow B_{\rho \eta}^{\sigma} (
    \Omega),\: \sigma - \frac{1}{\rho} = s - \frac{1}{p}, \: 0 < p \leq \rho \leq \infty, \: 0 < q \leq \eta \leq \infty
\end{equation}

where $\Omega = \mathbb{R}$ or $\Omega = [a,b]$.

(For two quasinormed spaces $A$ and $B$, the notation $B \hookrightarrow A$ ("B is embedded/imbedded in A") means $B \subset A$ and $\exists c \in (0, \infty): {||b||}_{B} \leq c {||b||}_A$ for any $b \in B$.)

Due to the Sobolev embedding/imbedding (\ref{eq:32}), our assumption $f \in B_{p q}^{s} (\Omega)$ implies 

\begin{equation}\label{eq:33}
    f \in B_{\rho \eta}^{\sigma} (\Omega)
\end{equation}

for any $(\rho, \eta, \sigma): \sigma - \frac{1}{\rho} = s - \frac{1}{p}$, $0 < p \leq \rho \leq \infty$, $0 < q \leq \eta \leq \infty$.
In \citep[Appendix B10 b)]{dechevsky1999} it was explained that for the Besov spaces $B_{\rho \eta}^{\sigma}$ with $(\rho, \eta, \sigma)$ lying on one and the same Sobolev embedding plane, an important part of the function-space theory is related with the so-called \emph{decreasing rearrangement} of $f$ in all Besov spaces $B_{\rho \eta}^{\sigma}$ with $(\rho, \eta, \sigma)$ on the same Sobolev embedding plane.
A detailed consideration of the concept of decreasing rearrangement can be found in \citep[Section 1.3]{bergh1976} and the Peetre-Kr{\'e}e formula \citep[Theorem 3.6.1]{bergh1976} together with \citep[3.14.5.6 and Theorem 5.2.1 (2) for $q=p$ in local notation]{bergh1976} (see also \citep[(B8)]{dechevsky1999}.
For our purposes in our present context it is sufficient to consider the \emph{normalized decreasing rearrangement} of the $\beta$-coefficients in the series (\ref{eq:5}) and in $\bigoplus_{j=j_0}^{J} W_j$ in its truncation (\ref{eq:11}, \ref{eq:12}), as follows (compare with \citep[Appendix B10 b), items b1 and b2]{dechevsky1999}).
Recall that here we consider only the case $n=1$ in (\ref{eq:5}-\ref{eq:16}) - in particular, in (\ref{eq:5}-\ref{eq:8}, \ref{eq:11}) this implies $l=1$.
Thus, in the sequel of the present definition of \emph{decreasing rearrangement}, we shall be skipping the index $l$.

b1) Fix $j_0 \in \mathbb{Z}$ (with no loss of generality, it can be assumed that $j_0=0$).
Consider all $(j,k)$ in $(5,6)$ such that $\supp \psi_{jk} \cap \supp f \neq \emptyset$.
Denote the set of all such $(j,k)$ by $I (f,\psi) = I (f, \psi, j_0)$.
It is clear that for every fixed $j=j_0, j_0+1,...$ in the generalization of (\ref{eq:5}, \ref{eq:6}) involving $j_0$ the number $M_j$ of elements of $I(f, \psi)$ from the $j$-th level does not exceed $c(f, \psi)$.
$2^j$, for some $c(f, \psi) \in (0, \infty)$.
Therefore, $M_j$ is finite for any $j=j_0,j_0+1,...$, but $M = \sum\limits_{j=j_0}^{\infty} M_j$ is, generally, \emph{not finite}.
On the other hand, for the truncation (\ref{eq:11},\ref{eq:12}) the number $m(j_0,J) = \sum\limits_{j=j_0}^J M_j$ is finite, with $M=\lim_{J \rightarrow \infty} m (j_0, J)$. 
Denote by $i (f, \psi, j_0, J)$ the subset of $I (f, \psi, j_0)$ such that $(j, k)$ participates in the formation of the truncation (\ref{eq:11}) and $\bigoplus_{j=j_0}^J W_j$ in (\ref{eq:12}).

The number of elements of $i (f, \psi, j_0, J)$ is 

\begin{equation}\label{eq:34}
    m (j_0, J) \leq c(f, \psi) \cdot 2^{j_0} \sum\limits_{k=0}^{J-j0} 2^k = c(f, \psi) 2^{j_0} \cdot \frac{2^{J-j_0+1}-1}{2-1} = c (f,\psi) (2^{J+1}-2^{j_0}) \leq 2 c (f, \psi) 2^J,
\end{equation}

regardless of the choice of $j_0$.

b2) Recalling that $\tau = \tau(p, s) = s - \frac{1}{p}$ in (\ref{eq:31}), for every $(j,k) \in I (f, \psi)$ normalize $|\beta_{jk}|$ by multiplying with the factor $2^{j (\tau + \frac{1}{2})}$ and consider the decreasing rearrangement $\{ b_\nu : \nu = 1,...,M\}$ of the 
(possibly, infinite) set $\{2^{j (\tau + \frac{1}{2})} |\beta_{_{jk}}|: (j,k) \in I (f, \psi)\}$.

In the case of the truncation (\ref{eq:11}, \ref{eq:12}), we get the subset $\{2^{j(\tau + \frac{1}{2})} |\beta_{jk}|: (j,k) \in i (j_0, J)\}$ which is finite, with number of elements $O_{f,\psi} (2^J)$, according to (\ref{eq:34}).

For this model case, the activation operator $\Lambda = \Lambda_{\delta}$ is of threshold type, with threshold $\delta \in (0, \infty)$, defined in the following way.
Let the decreasing rearrangement of $I (f, \psi)$ be $\{b_\nu, \nu=1,...,M\}$, with $(j,k) \in I (f, \psi)$ being ordered in a respective sequence $\{(j_\nu,k_\nu), \nu=1,...,M\}$ where $(j_\nu, k_\nu)$ corresponds to $b_\nu$ for any $\nu=1,...,M$.

Then

\begin{equation}\label{eq:35}
    \Lambda (\beta_{j_\nu k_\nu} \psi_{j_\nu k_\nu}) = 
        \begin{cases}
          0, & \textrm{if} \; 2^{j_\nu (\tau + \frac{1}{2})} |\beta_{j_\nu k_\nu}| \in (0, \delta) \\
          \beta_{j_\nu k_\nu} \psi_{j_\nu k_\nu}, & \textrm{if} \; 2^{j_\nu (\tau + \frac{1}{2})} |\beta_{j_\nu k_\nu}| \geq \delta \\
        \end{cases}   
\end{equation}

The selection of the threshold $\delta$ depends on the concrete context of the learning process.
We shall call every threshold activation method designed via the sequence of steps b1) and b2) \emph{a decreasing rearrangement activation method}.
For this type of activation method with threshold $\delta$, the UAT corresponds to $\delta \rightarrow 0+$ and is given by the following theorem.

\begin{theorem}\label{th:4}
    Let $\delta \rightarrow 0+$ in (\ref{eq:35}), and let $f$ be as assumed above.
    Then, the summands in the series in the RHS of (\ref{eq:5}) can be commuted in such a way that (\ref{eq:5}) becomes 

    \begin{equation}\label{eq:36}
        \sum\limits_{k \in \mathbb{Z}} \alpha_{j_0 k} \varphi_{j_0 k} (x) + \sum\limits_{j=j_0}^{\infty} \sum\limits_{k \in \mathbb{Z}} \beta_{jk} \psi_{jk} (x) = f(x) = \sum\limits_{k \in \mathbb{Z}} \alpha_{j_0 k} \varphi_{j_0 k} + \sum_{\nu=1}^M \beta_{j_\nu k_\nu} \psi_{j_\nu k_\nu} (x)
    \end{equation}

    where the RHS converges to $f(x)$ Lebesgue -- a.e. in $\mathbb{R}$, as well as in the topology of $B_{p q}^{s} (\mathbb{R})$ and $B_{\rho \eta}^{\sigma} (\mathbb{R})$ for any $(\rho, \eta, \sigma): 0 < p \leq \rho \leq \infty, 0 < q \leq \eta \leq \infty$, $\sigma-\frac{1}{\rho}=s-\frac{1}{p}=\tau$.
\end{theorem}

Theorem \ref{th:4} continues to hold true for boundary--corrected wavelets and $\Omega=[a,b]$, with respective modifications in (\ref{eq:5}) and (\ref{eq:36}).

We shall now upgrade the qualitative result of Theorem \ref{th:4} by formulating a quantitative result which proves to be best possible in a certain sense specified below.
Among all Besov spaces with (quasi)norm (\ref{eq:6}), the ones which are Hilbert spaces are exactly 

\begin{equation}\label{eq:37}
    B_{p q}^s (\Omega) \textrm{ with } p=q=2 \textrm{ and } 0 < s < r,
\end{equation}

where, for $n=1$ considered here, $\Omega = \mathbb{R}$ or $\Omega = [a,b], -\infty < a < b < +\infty$.
Choose arbitrary triple $(p, q, s)$ such that $0 < p \leq 2, \: 0 < q \leq 2$, ${(\frac{1}{p} - 1)}_+ < s < r$, and consider the respective triple $\rho = \eta = 2$, $\sigma - \frac{1}{\rho} = s - \frac{1}{p}$.
For (\ref{eq:5}, \ref{eq:6}) to hold for this choice of $(\rho, \eta, \sigma)$ it is necessary that 

\begin{equation}\label{eq:38}
    0 \leq \sigma < r
\end{equation}

holds where $\sigma = 0$ corresponds to the case $B_{22}^{0} (\Omega) = L_2 (\Omega)$.
Therefore, (\ref{eq:5}, \ref{eq:6}) hold simultaneously for the couples $(p, q, s)$ and $(2, 2, \sigma)$, iff the following inequalities and equalities are simultaneously

\begin{equation}\label{eq:39}
    {(\frac{1}{p} - 1)}_+ < s < r,\;\;\; 0 \leq \sigma < r,\;\;\; \sigma = s - \frac{1}{p} + \frac{1}{2},\;\;\; p \leq 2,\;\;\; q \leq 2.
\end{equation}

Solving (\ref{eq:39}) for $p, q$ and $s$, we obtain 

\begin{equation}\label{eq:40}
    \frac{1}{r+\frac{1}{2}} < p \leq 2, \;\;\; 
    0 < q \leq 2, \;\;\;
    \frac{1}{p} - \frac{1}{2} \leq s < r,
\end{equation}

under which assumptions (\ref{eq:38}) holds.
Consider the orthocomplement 

\begin{equation}\label{eq:41}
    O_{j_0 \sigma} = O \big( V_{j_0}, B_{22}^{\sigma} (\Omega)\big) = V_{j_0}^{\perp (B_{22}^\sigma (\Omega))}
\end{equation}

of $V_{j_0}$ in $B_{22}^\sigma (\Omega)$, with respect to the inner product in $B_{22}^\sigma (\Omega)$, $0 \leq \sigma < r$.
This orthocomplement is well defined because $\{ \varphi_{j_0 k_0}, \psi_{j k_j} \}$ is an unconditional Riesz basis in all Besov spaces where (\ref{eq:5}, \ref{eq:6}) hold, and $B_{22}^\sigma (\Omega)$ is a Hilbert space, so that $V_{j_0}^\perp$ is well-defined with respect to the inner product in $B_{22}^\sigma (\Omega)$, $0 \leq \sigma < r$.
For a given $f \in B_{22}^\sigma (\Omega)$, define $f_{j_0} \in V_{j_0}$ as follows 

\begin{equation}\label{eq:42}
    f_{j_0} = \sum\limits_{k} \alpha_{j_0 k} (f) \varphi_{j_0 k}
\end{equation}

Let $k \in \mathbb{N}$ and consider an arbitrary subspace $S_k$ with $\dim S_k = k$, such that

\begin{equation}\label{eq:43}
    S_k \subset O_{j_0 \sigma} \subset B_{22}^\sigma (\Omega),
\end{equation}

and define the best-approximation functional 

\begin{equation}\label{eq:44}
    E_k (f; B_{22}^\sigma (\Omega)) = 
    \inf\limits_{s \in S_k} \inf\limits_{S_k \subset O_{j_0 \sigma}}
    {|| f - f_{j_0} - s ||}_{B_{22}^\sigma}.
\end{equation}

Now, we are in the position to formulate the following remarkable result.

\begin{theorem}\label{th:5}
    Assume that $f$ is as in Theorem \ref{th:4} with the additional assumption that (\ref{eq:40}) holds.
    Then,
    \begin{equation}\label{eq:45}
        {||f - \sum\limits_{\nu=1}^k \beta_{j_\nu k_\nu} \psi_{j_\nu k_\nu}||}_{B_{22}^{s - \frac{1}{p} + \frac{1}{2}} (\mathbb{R})} =
        E_k \big(f; B_{22}^{s - \frac{1}{p} + \frac{1}{2}} (\mathbb{R}) \big),
        k = 1,2,...
    \end{equation}
\end{theorem}

The result (\ref{eq:45}) holds, \emph{mutatis mutandis}, also for the case of boundary-corrected wavelets and $\Omega = [a, b]$, $-\infty < a < b < +\infty$.

Theorem \ref{th:5} shows that using a sufficiently broad (i.e., satisfying (\ref{eq:14})) single-layered WBNN for learning curves with Besov regularity while using the current activation method results in a learning strategy which is \emph{optimal} in the following two senses.

\begin{enumerate}[label=\arabic*.]
    \item \emph{Fastest learning} -- using a fixed number of active neurons, the learned function is closest possible to the original, with the distance measuring the closeness being in terms of ${||\cdot||}_{B_{22}^\sigma}$, $0 \leq \sigma < r$, that is, taking into account not only position in space $(\sigma = 0)$ but also fractional derivatives up to order $r$.
    \item \emph{Maximal compression} -- for a benchmark determined by a fixed distance between the target function and its learned version measured in terms of ${||\cdot||}_{B_{22}^\sigma}$, the benchmark result is achieved with the fewest possible activated neurons.
\end{enumerate}

In the remaining part of the present study, we shall illustrate graphically aspects 1. and 2. of the optimality of the learning process with WBNNs when the current activation method is being used.

\section{Representative model examples}\label{s6}

\begin{figure}[h]
    \centering
    \begin{subfigure}[b]{0.495\textwidth}
        \centering
        \includegraphics[width=\textwidth]{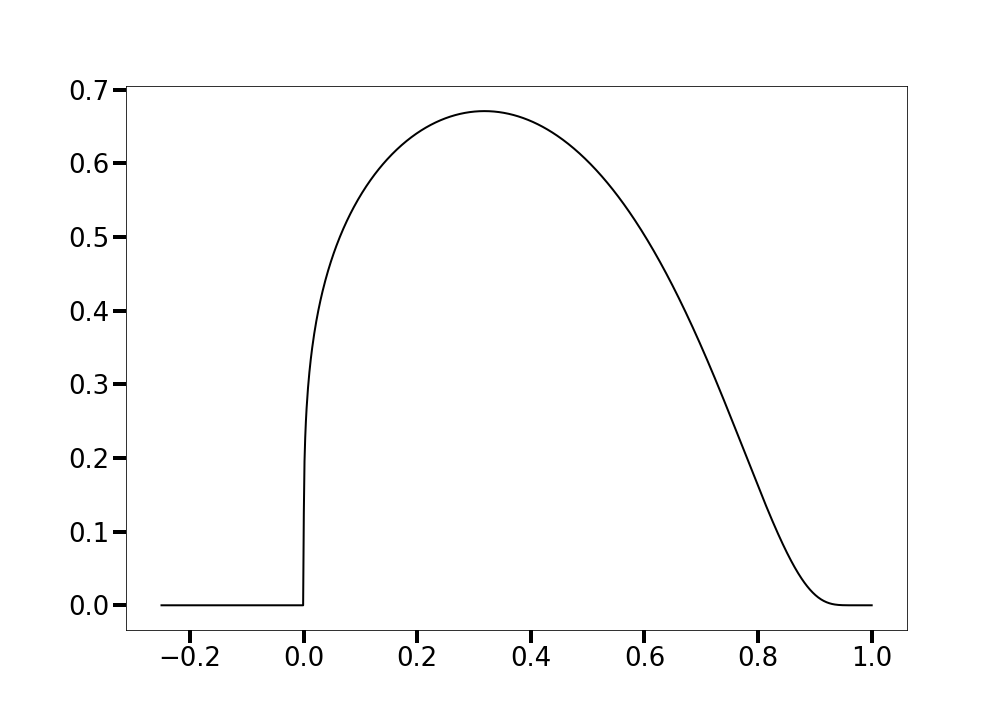}
        \caption{"$\lambda$-tear"}
        \label{fig:lambdatear}
    \end{subfigure}
    \hfill
    \begin{subfigure}[b]{0.495\textwidth}
        \centering
        \includegraphics[width=\textwidth]{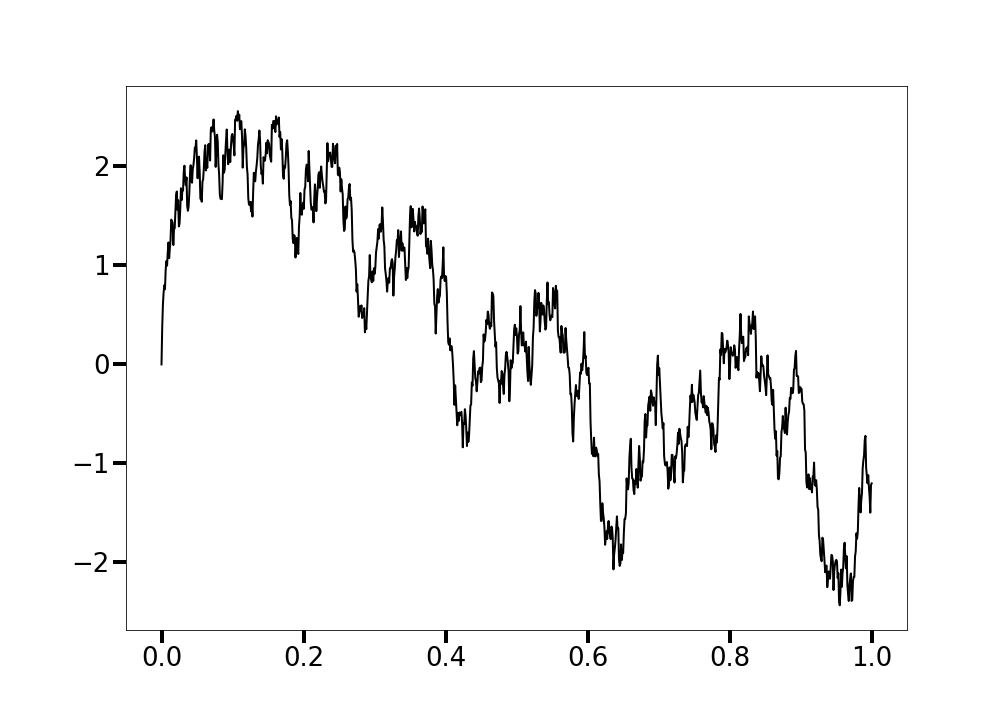}
        \caption{Weierstrass type fractal}
        \label{fig:weierstrass}
    \end{subfigure}
    \caption{Piecewise smooth type vs fractal type manifolds}
    \label{fig:main}
\end{figure}

We shall consider in detail graphical visualization related to two model examples which are representative in three important aspects.

\begin{enumerate}[label=\arabic*.]
    \item The first one (the "$\lambda$-tear") is a typical manifold of \emph{piecewise smooth type}, while the second one (Weierstrass-type curve) is a typical manifold of \emph{fractal type}.
    \item For both of them, their \emph{exact} metric power, metric logarithmic and smoothness index of their Besov regularity is known.
    \item The exact metric power, metric logarithmic and smoothness index of Besov regularity can be selected to be the \emph{same} for both examples, which allows for insightful graphical comparisons.
\end{enumerate}

\emph{Example 1} (see Fig. \ref{fig:lambdatear}). \emph{The "$\lambda$-tear"}

\begin{equation}\label{eq:46}
    f_1(x) = 
    \begin{cases}
      x_+^\lambda \exp (-\frac{x^2}{1-x^2}) & x \in (0,1) \\
      0 & x \in [-1,0] \cup [1,2] \\
    \end{cases}  
\end{equation}

where $\lambda \in (0,1)$.
This function is analytic for $ x \in [-1,0) \cup (0,1) \cup (1,2]$; it is $C^\infty$, but not analytic at $x = 1$; at $x = 0$ it is only $C^0$.
Its exact Besov regularity for $p: 1 \leq p \leq \infty$ is $f \in B_{p \infty}^{\lambda + \frac{1}{p}} (\Omega)$, where $\Omega = \mathbb{R}$ or $\Omega = [-1,2]$ \citep[Proposition 2.4.2]{brenner1975}, see also \citep[Section 7, Example 1]{dechevsky1999}.

\emph{Example 2} (see Fig. \ref{fig:weierstrass}). \emph{Weierstrass-type curve}

\begin{equation}\label{eq:47}
    f_2(x) = \sum\limits_{k=0}^\infty 1.5^{- \tau k} \sin (1.5^k \times 5x), \;\;\; x \in \mathbb{R}
\end{equation}

where $\tau \in (0,2)$.
For the purpose of comparing with Example 1, we shall consider only the restriction of $f_2$ onto $\supp f_1$, i.e., for $x \in [0,1]$.

The graph of $f_2$ is a typical \emph{self-similar monofractal} with constant local H{\"o}lder index $\tau$ and constant local fractal dimension $2 - \tau$ which is also its global fractal dimension on $[0,1]$.
When considering $ \Omega = \mathbb{R}$, for any compactly supported $\chi \in C^\infty (\mathbb{R})$ such that $[0,1] \subset \supp \chi$ and $\chi \equiv 1$ on $[0,1]$, $\chi \cdot f_2 \in B_{p \infty}^{\tau} (\mathbb{R})$ \citep[Proposition 2.4.1 for the imaginary part $G_t$ in local notations, with additional rescaling]{brenner1975}, see also \citep[Section 7, Example 2]{dechevsky1999}.
For the restriction $\overline{f}_2 = f_2 \bigg|_{\supp f_1}$ in the case of boundary-corrected wavelets with $\Omega=[0,1]$, we have directly $\overline{f}_2 \in B_{p \infty}^\tau ([0,1])$.
For every $p : 1 \leq p \leq \infty$ this Besov regularity of $\overline{f}_2$ is \emph{exact}.
Clearly, when $\tau = \lambda + \frac{1}{p}$, $f_1$ and $f_2$ have exactly the same exact Besov regularity.

Besides the detailed comparative study of Examples 1 and 2, we shall study some additional geometric aspects of the learning process on two other model examples: "Double chirp" and "Sinusoidal density".

\begin{figure}
    \centering
    \begin{subfigure}[b]{0.495\textwidth}
        \centering
        \includegraphics[width=\textwidth]{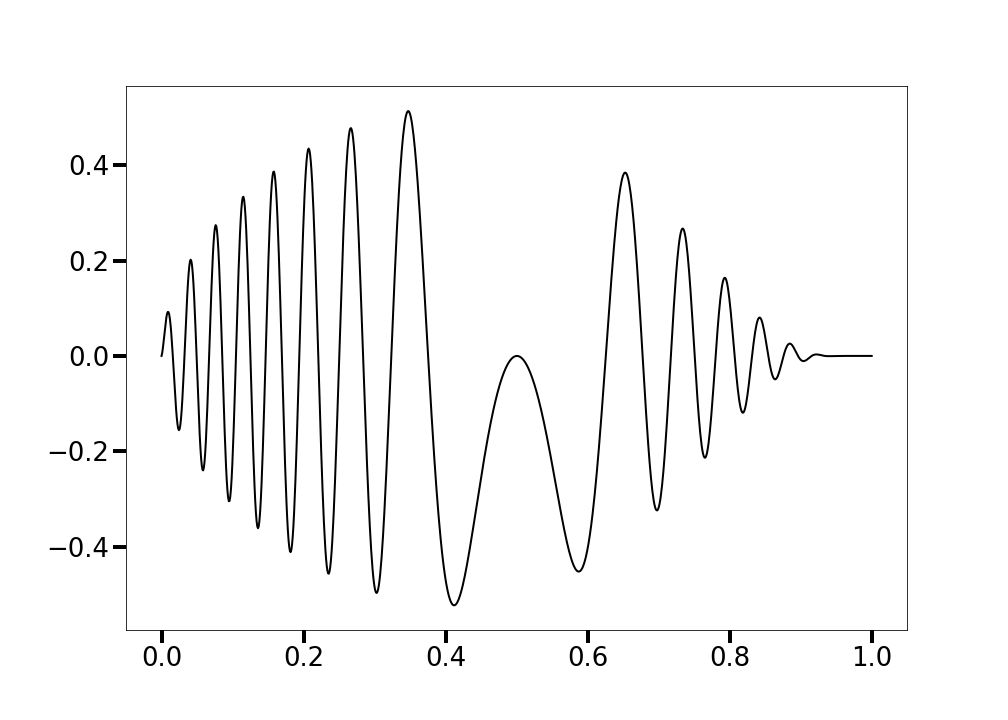}
        \caption{"Double chirp"}
        \label{fig:dchirp}
    \end{subfigure}
    \hfill
    \begin{subfigure}[b]{0.495\textwidth}
        \centering
        \includegraphics[width=\textwidth]{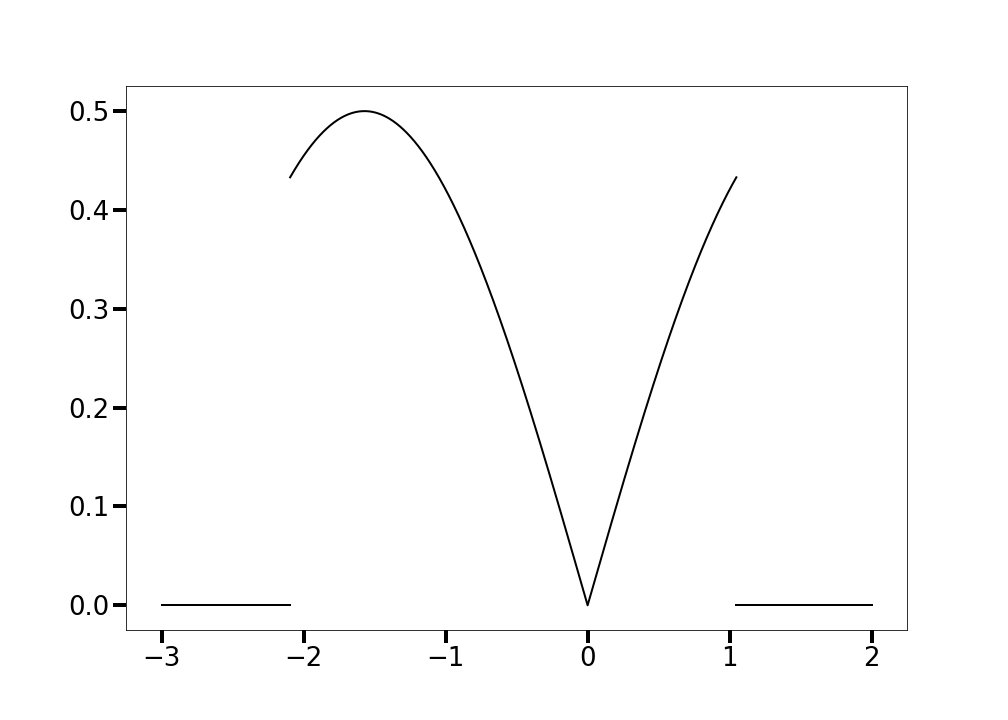}
        \caption{Sinusoidal density}
        \label{fig:sdens}
    \end{subfigure}
    \caption{Curves used to study Besov regularity}
    \label{fig:subsequent}
\end{figure}

\emph{Example 3} (see Fig. \ref{fig:dchirp}). \emph{"Double chirp"}

\begin{equation}\label{eq:48}
    f_3 (x) = \sqrt[4]{x} \exp (- \frac{x^2}{1 -  x^2} ) \sin [64 \pi x (1 - x)],
    x \in [0,1],
\end{equation}

Compare also \citep[Section 7, Example 3]{dechevsky1999}.
The graph of $f_3$ is very spatially inhomogeneous, containing at the endpoints 0 and 1 two chirps of a very different nature.
$f_3$ in (\ref{eq:48}) is a product of "$\lambda$-tear" for $\lambda = \frac{1}{2}$ and $C^\infty$-smooth function, so its Besov regularity is exactly the same as the Besov regularity of a "$\lambda$-tear" (Example 1) for $\lambda = \frac{1}{2}$.
Of special interest is to compare how the optimal learning algorithm deals in the spatially different parts of the graph for large, moderate and small samples, or for small, medium or high compression percentage.

\emph{Example 4} (see Fig. \ref{fig:sdens}). \emph{"Sinusoidal density"}

\begin{equation}\label{eq:49}
    f_4 (x) = 
    \begin{cases}
        \frac{1}{2} | \sin x | & x \in [-\frac{2 \pi}{3}, \frac{\pi}{3}] \\
        0 & x \in ( - \infty, - \frac{2 \pi}{3} ) \cup (\frac{\pi}{3} + \infty) \\
    \end{cases}
\end{equation}

$f_4$ is analytic on $( - \infty, -\frac{2 \pi}{3} ) \cup (- \frac{2 \pi}{3}, 0) \cup ( 0, \frac{\pi}{3} ) \cup ( \frac{\pi}{3}, + \infty )$ ;
$f_4$ is $C^0$ at $x = 0$; $f_4$ is discontinuous at $x = - \frac{2 \pi}{3}$ and $x = \frac{\pi}{3}$.

Because of the presence of the two jumps, $f_4 \in B_{p \infty}^{\frac{1}{p}} (\Omega)$, $1 \leq p \leq \infty$, $\Omega = \mathbb{R}$ or, for boundary corrected wavelets, $\Omega = [-\pi, \pi]$.
This result about the Besov regularity of $f_4$ follows from the result about Besov regularity of the Heaviside step function which is present in implicit form in the embeddings in \citep[Appendix B12b, item (iv)]{dechevsky1999}.
The function exhibits considerable spatial inhomogeneity in the neighbourhoods of the three points of singularity $(x = - \frac{2 \pi}{3}$, $x = 0$ and $x = \frac{\pi}{3}$).
Of particular interest is to compare the performance of the optimal learning algorithm in a neighbourhood of each of the three singularities. The comparison of the performance between the two jump-singularities at $x = - \frac{2 \pi}{3}$ and $x = \frac{\pi}{3}$ should include also comparative study of the local Gibbs phenomenon.

The sample sizes in Figures 1-13 are $N=2^{10}$ or less.


\begin{figure}
    \centering
    \begin{subfigure}[b]{0.7\textwidth}
        \centering
        \includegraphics[width=\textwidth]{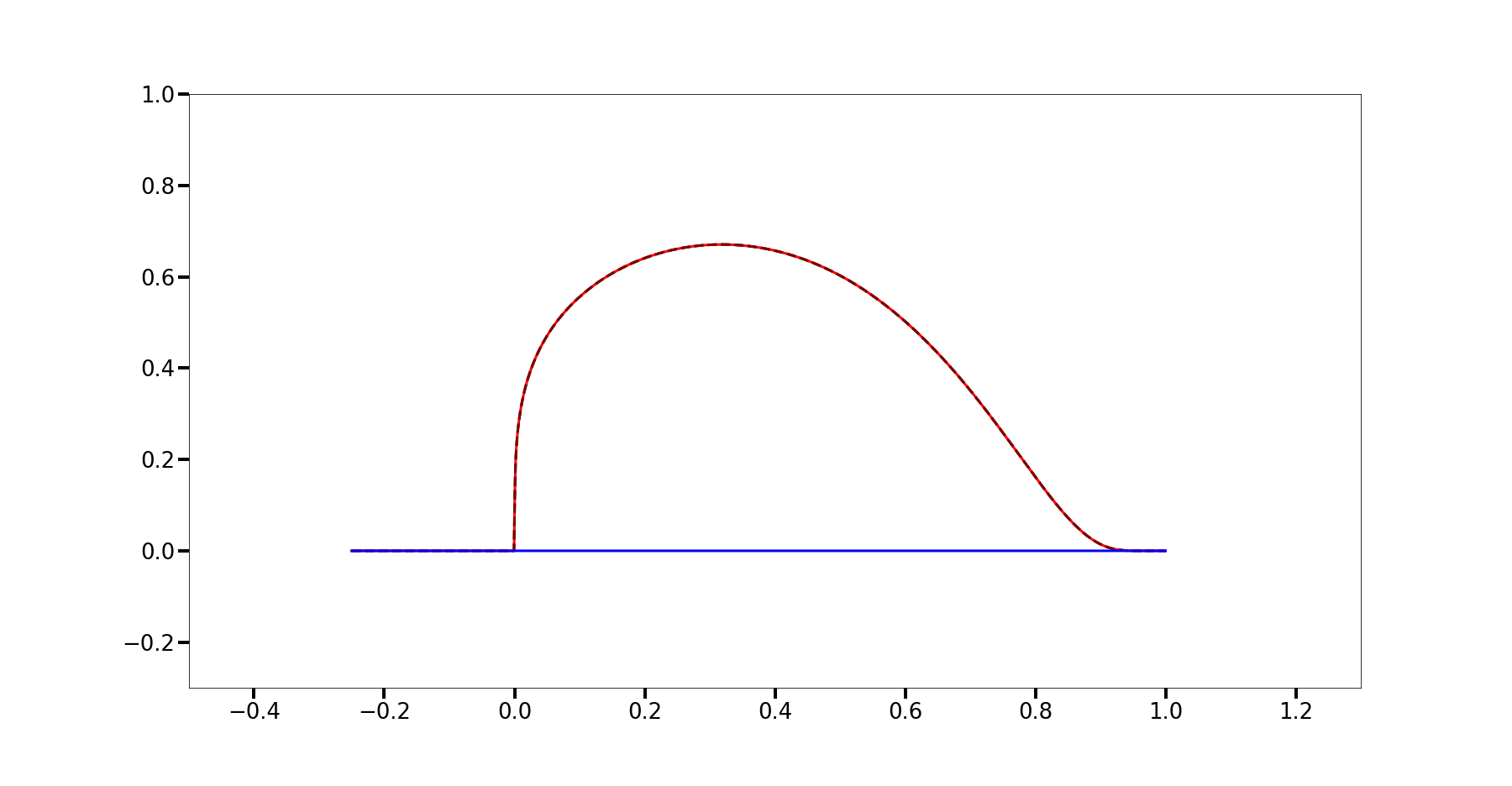}
        \caption{85 \% compression}
        \label{fig:perc_and_error_85_perc_lambdatear}
    \end{subfigure}

    \begin{subfigure}[b]{0.7\textwidth}
        \centering
        \includegraphics[width=\textwidth]{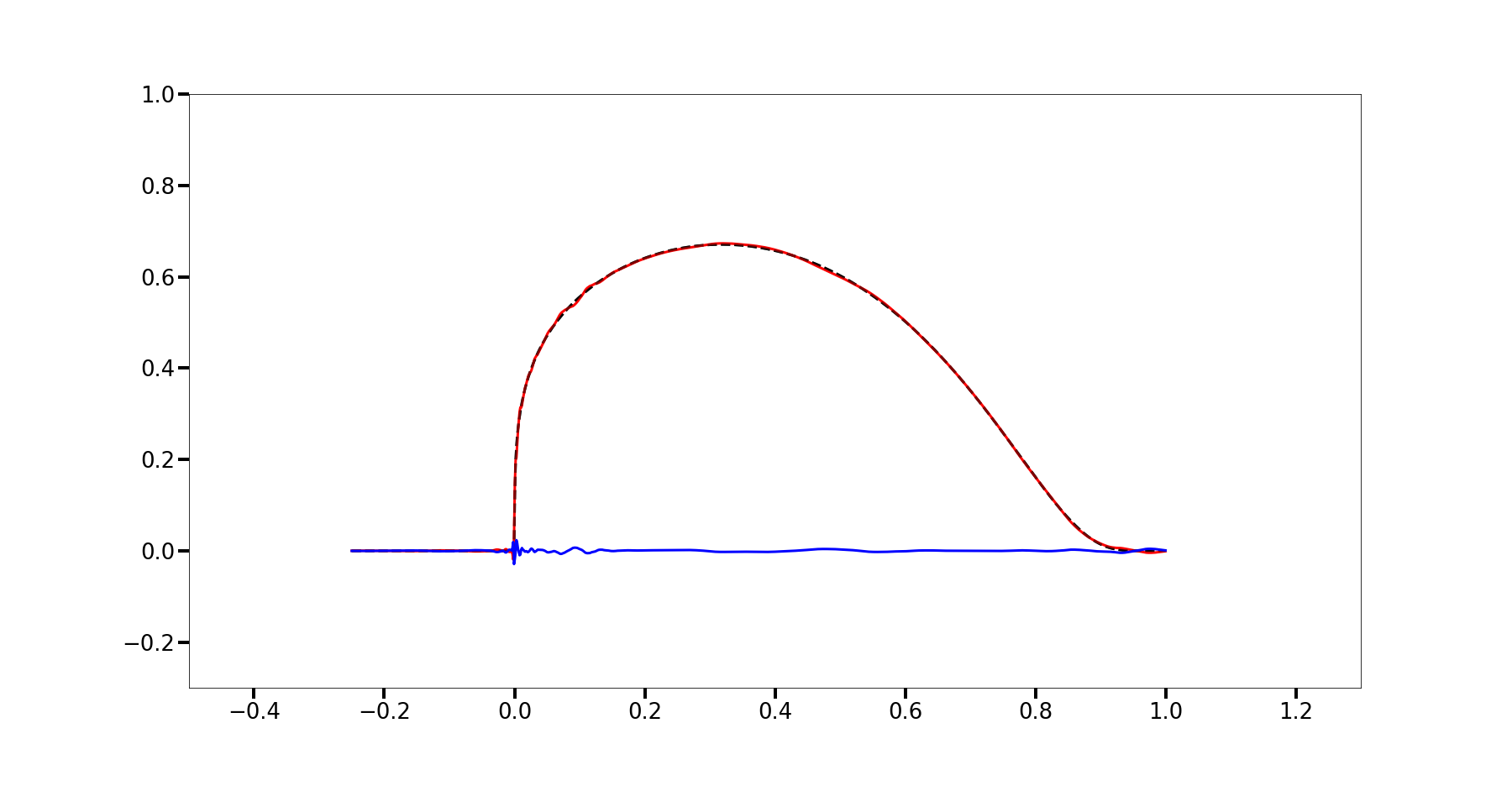}
        \caption{98 \% compression}
        \label{fig:perc_and_error_98_perc_lambdatear}
    \end{subfigure}

    \begin{subfigure}[b]{0.7\textwidth}
        \centering
        \includegraphics[width=\textwidth]{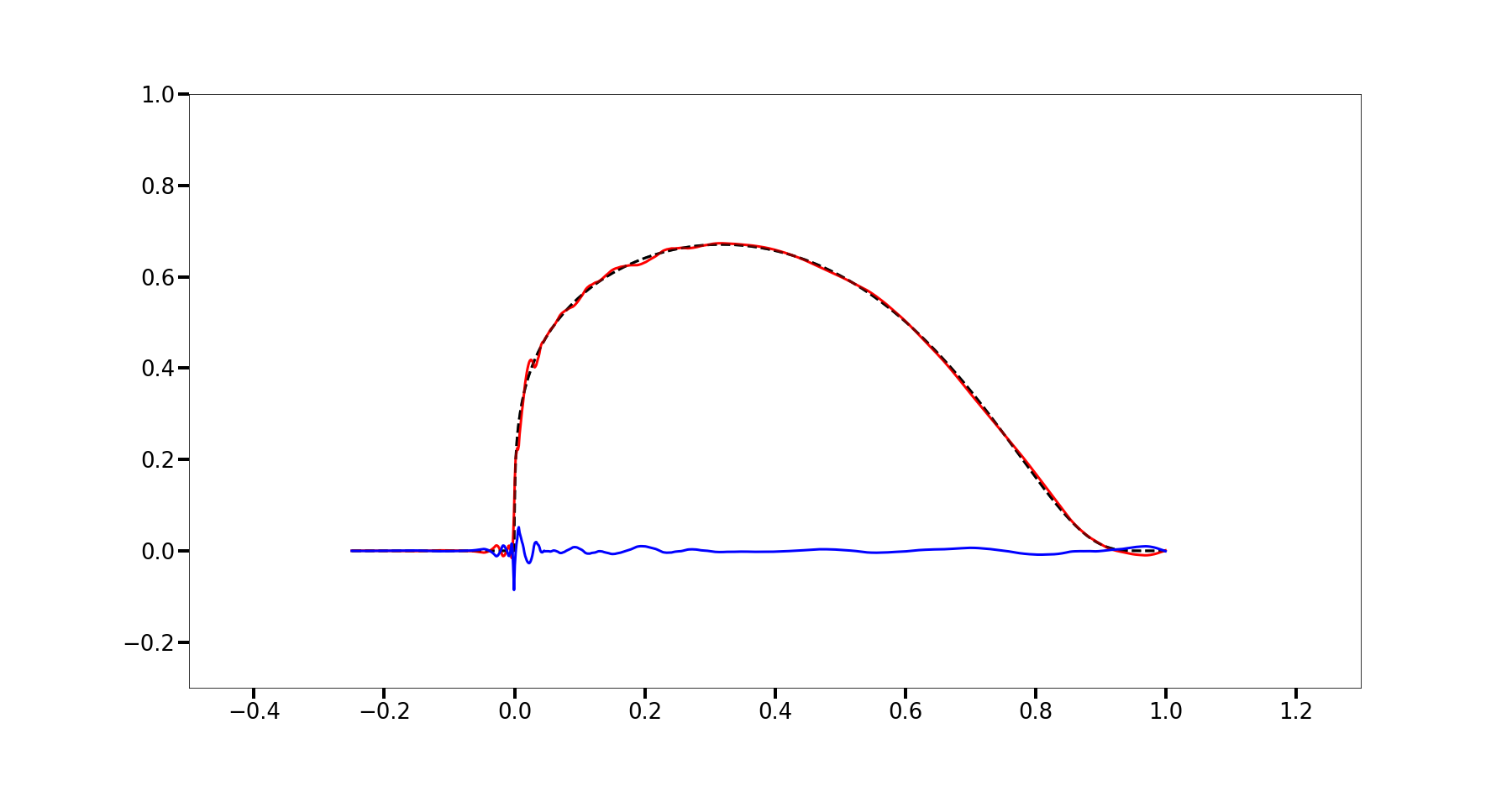}
        \caption{99 \% compression}
        \label{fig:perc_and_error_99_perc_lambdatear}
    \end{subfigure}
    \caption{Target function (dashed black), learned function (red) and the error between the two (blue) for the "$\lambda$-tear" under high compression.}
    \label{fig:lmtear:subsequent}
\end{figure}

\begin{figure}
    \centering
    \begin{subfigure}[b]{0.7\textwidth}
        \centering
        \includegraphics[width=\textwidth]{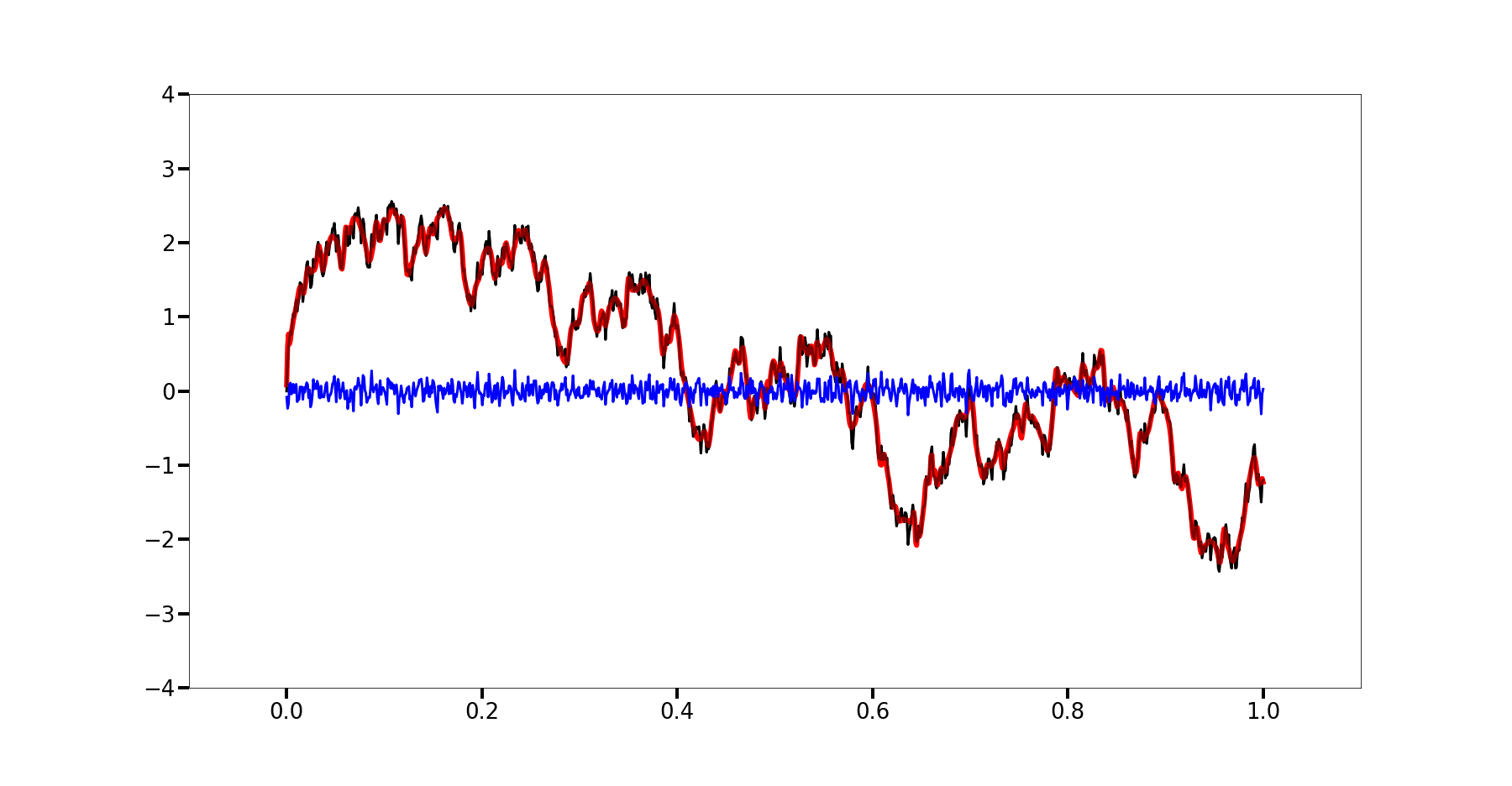}
        \caption{85 \% compression}
        \label{fig:perc_and_error_85_perc_fractal}
    \end{subfigure}

    \begin{subfigure}[b]{0.7\textwidth}
        \centering
        \includegraphics[width=\textwidth]{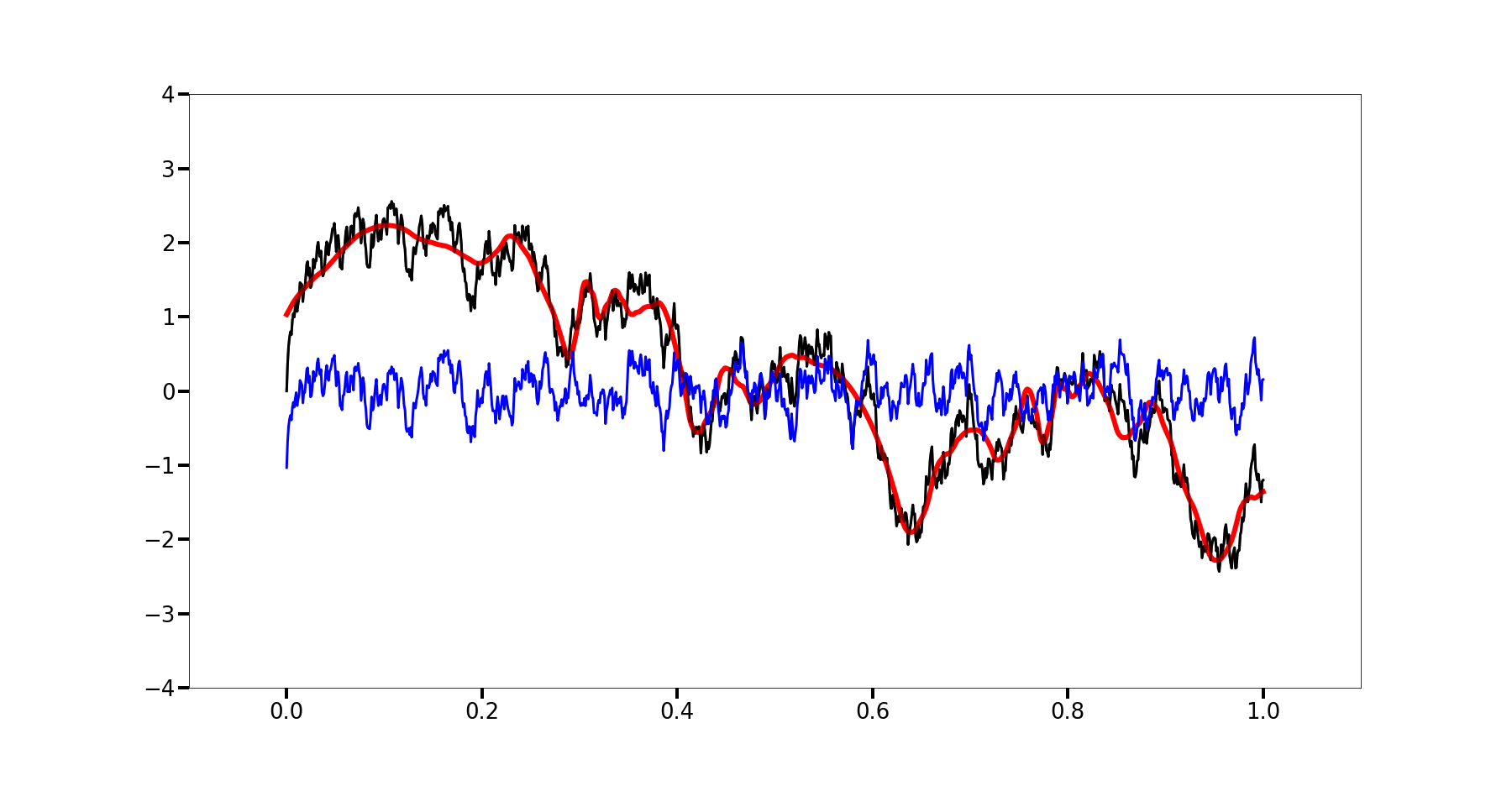}
        \caption{98 \% compression}
        \label{fig:perc_and_error_98_perc_fractal}
    \end{subfigure}

    \begin{subfigure}[b]{0.7\textwidth}
        \centering
        \includegraphics[width=\textwidth]{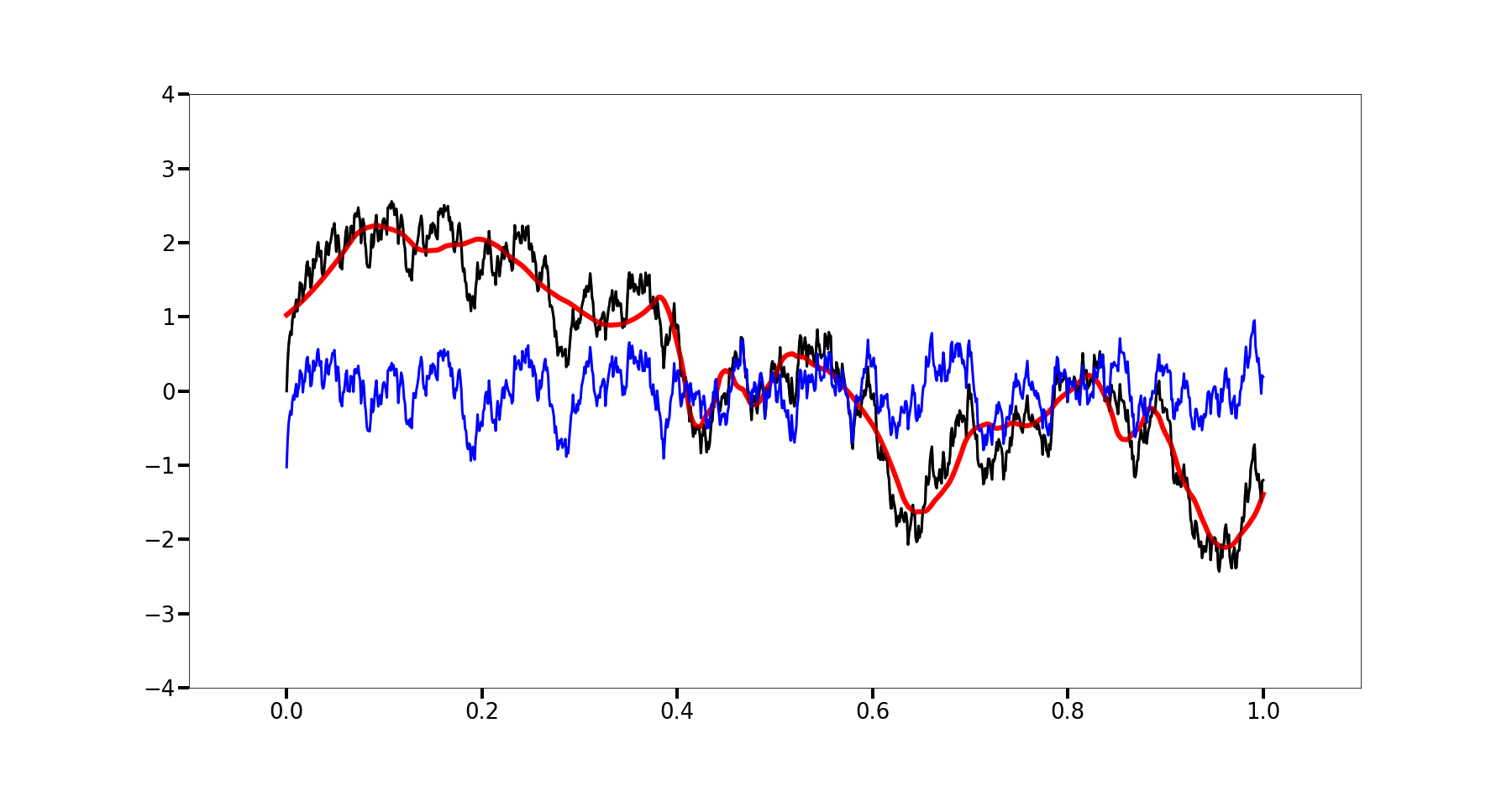}
        \caption{99 \% compression}
        \label{fig:perc_and_error_99_perc_fractal}
    \end{subfigure}
    \caption{Target function (dashed black), learned function (red) and the error between the two (blue) for the "Weierstrass function" under high compression.}
    \label{fig:fractal:subsequent}
\end{figure}

\begin{figure}
    \centering
    \begin{subfigure}[b]{0.7\textwidth}
        \centering
        \includegraphics[width=\textwidth]{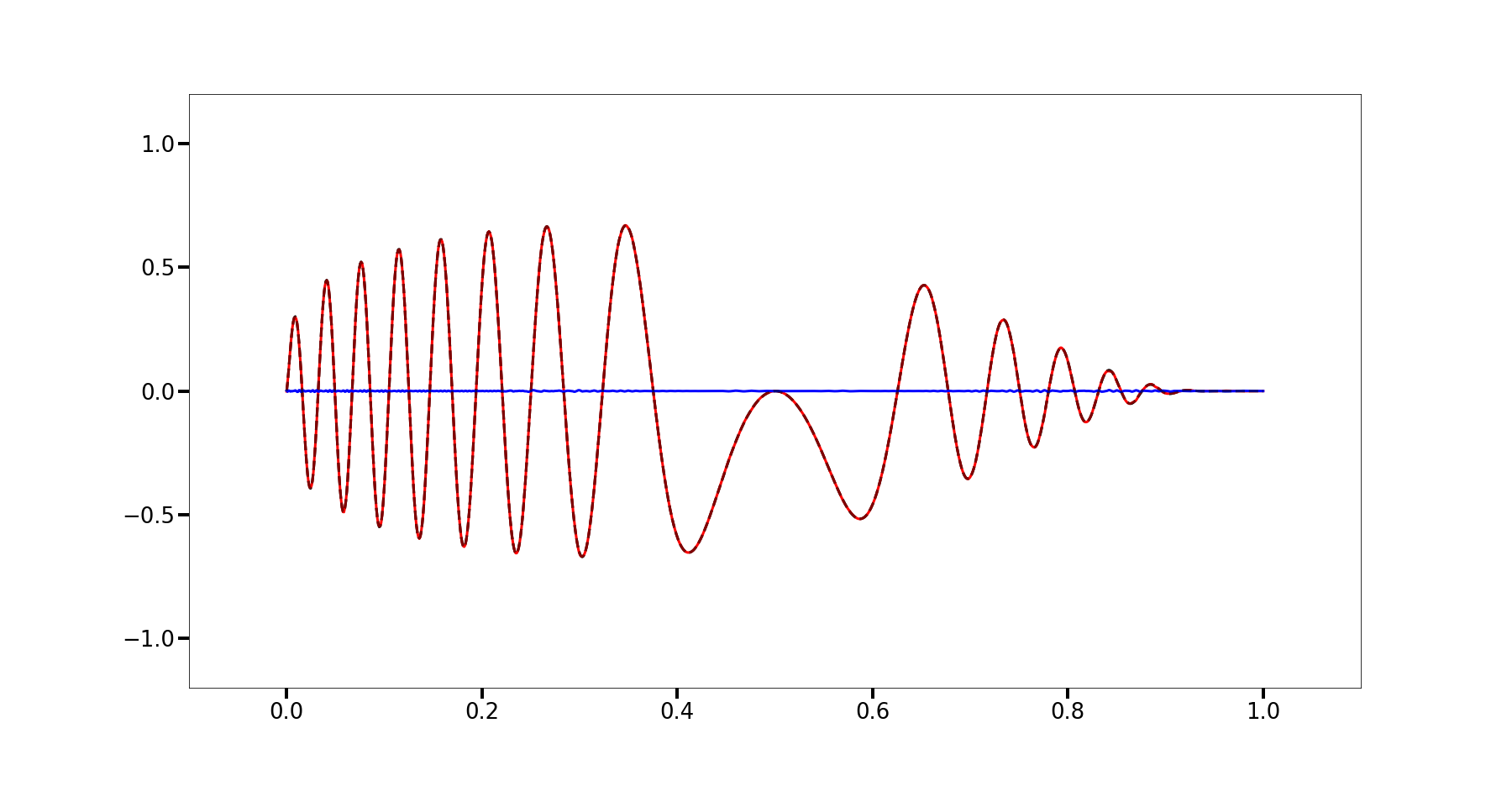}
        \caption{85 \% compression}
        \label{fig:perc_and_error_85_perc_dchirp}
    \end{subfigure}

    \begin{subfigure}[b]{0.7\textwidth}
        \centering
        \includegraphics[width=\textwidth]{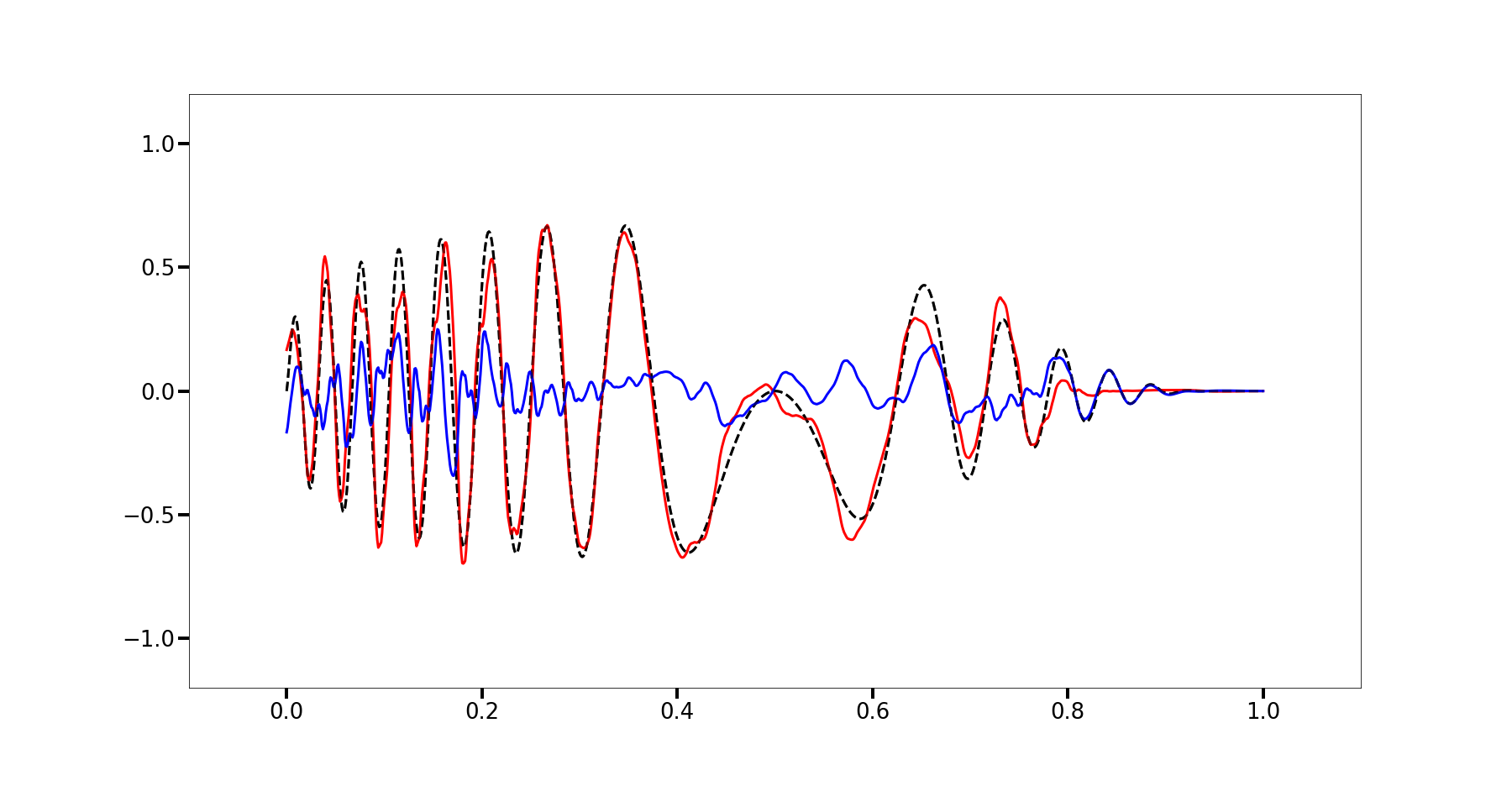}
        \caption{98 \% compression}
        \label{fig:perc_and_error_98_perc_dchirp}
    \end{subfigure}

    \begin{subfigure}[b]{0.7\textwidth}
        \centering
        \includegraphics[width=\textwidth]{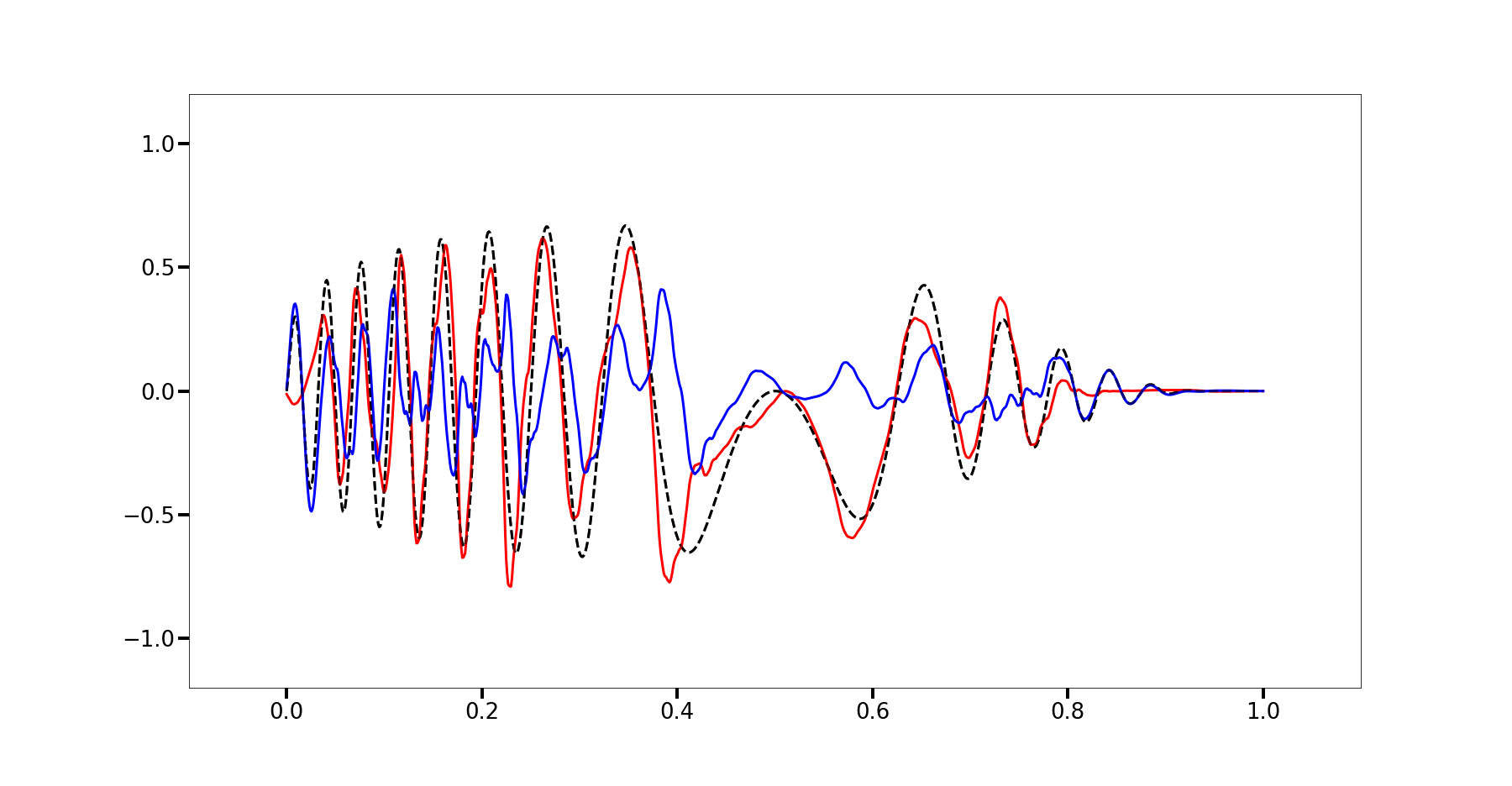}
        \caption{99 \% compression}
        \label{fig:perc_and_error_99_perc_dchirp}
    \end{subfigure}
    \caption{Target function (dashed black), learned function (red) and the error between the two (blue) for "the double chirp" under high compression.}
    \label{fig:dchirp:subsequent}
\end{figure}

\begin{figure}
    \centering
    \begin{subfigure}[b]{0.7\textwidth}
        \centering
        \includegraphics[width=\textwidth]{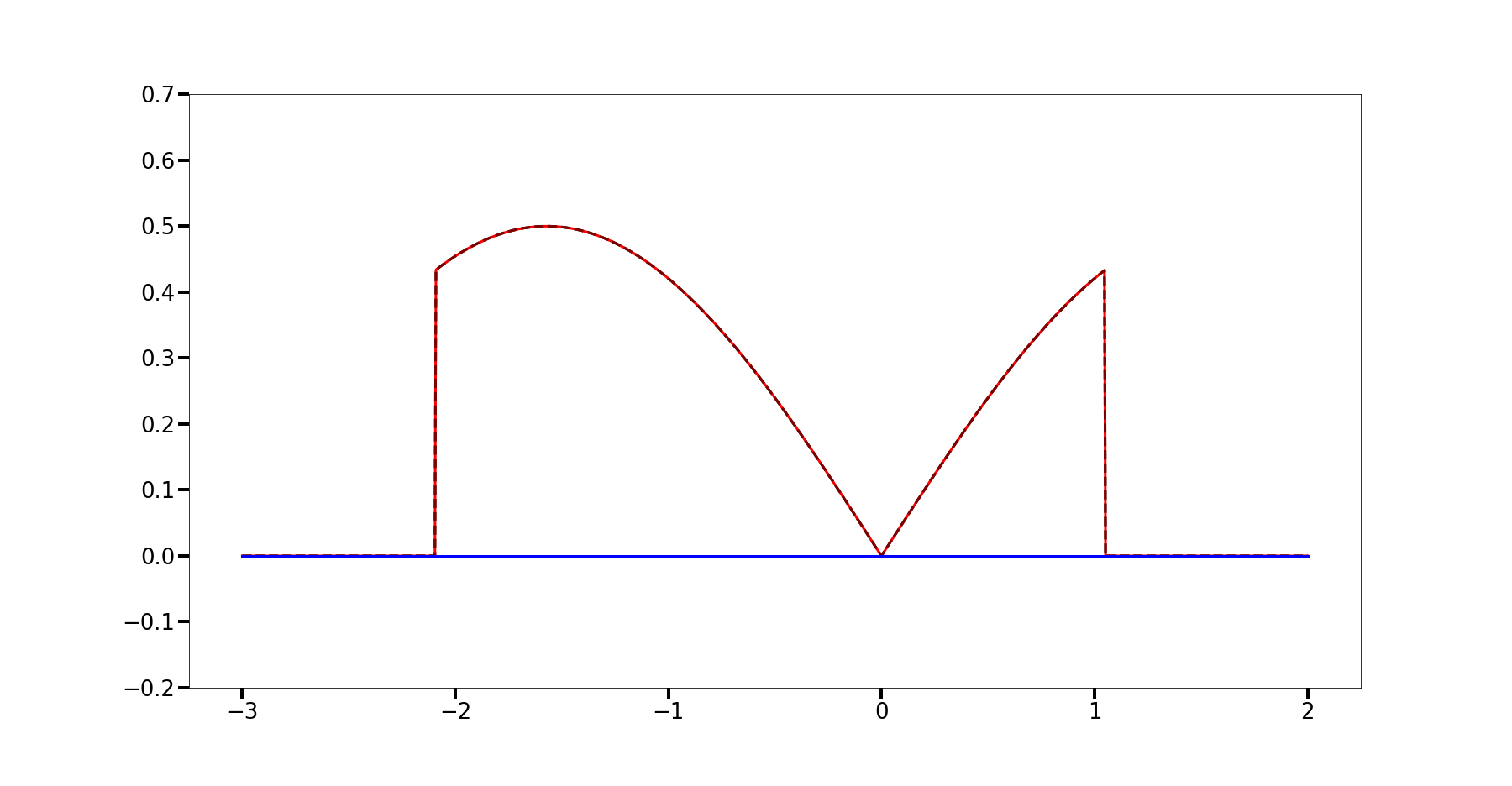}
        \caption{85 \% compression}
        \label{fig:perc_and_error_85_perc_sinusdensity}
    \end{subfigure}

    \begin{subfigure}[b]{0.7\textwidth}
        \centering
        \includegraphics[width=\textwidth]{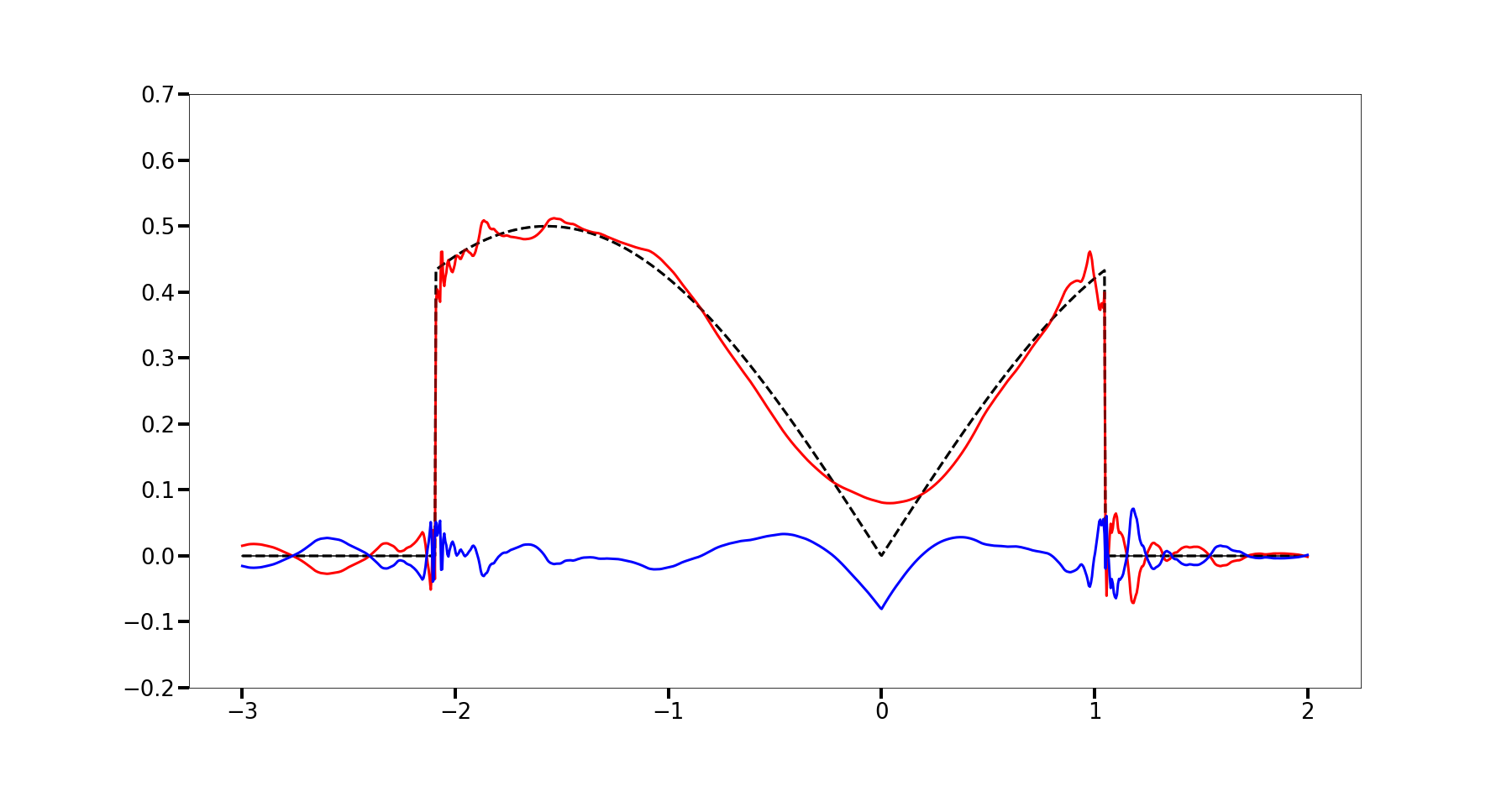}
        \caption{98 \% compression}
        \label{fig:perc_and_error_98_perc_sinusdensity}
    \end{subfigure}

    \begin{subfigure}[b]{0.7\textwidth}
        \centering
        \includegraphics[width=\textwidth]{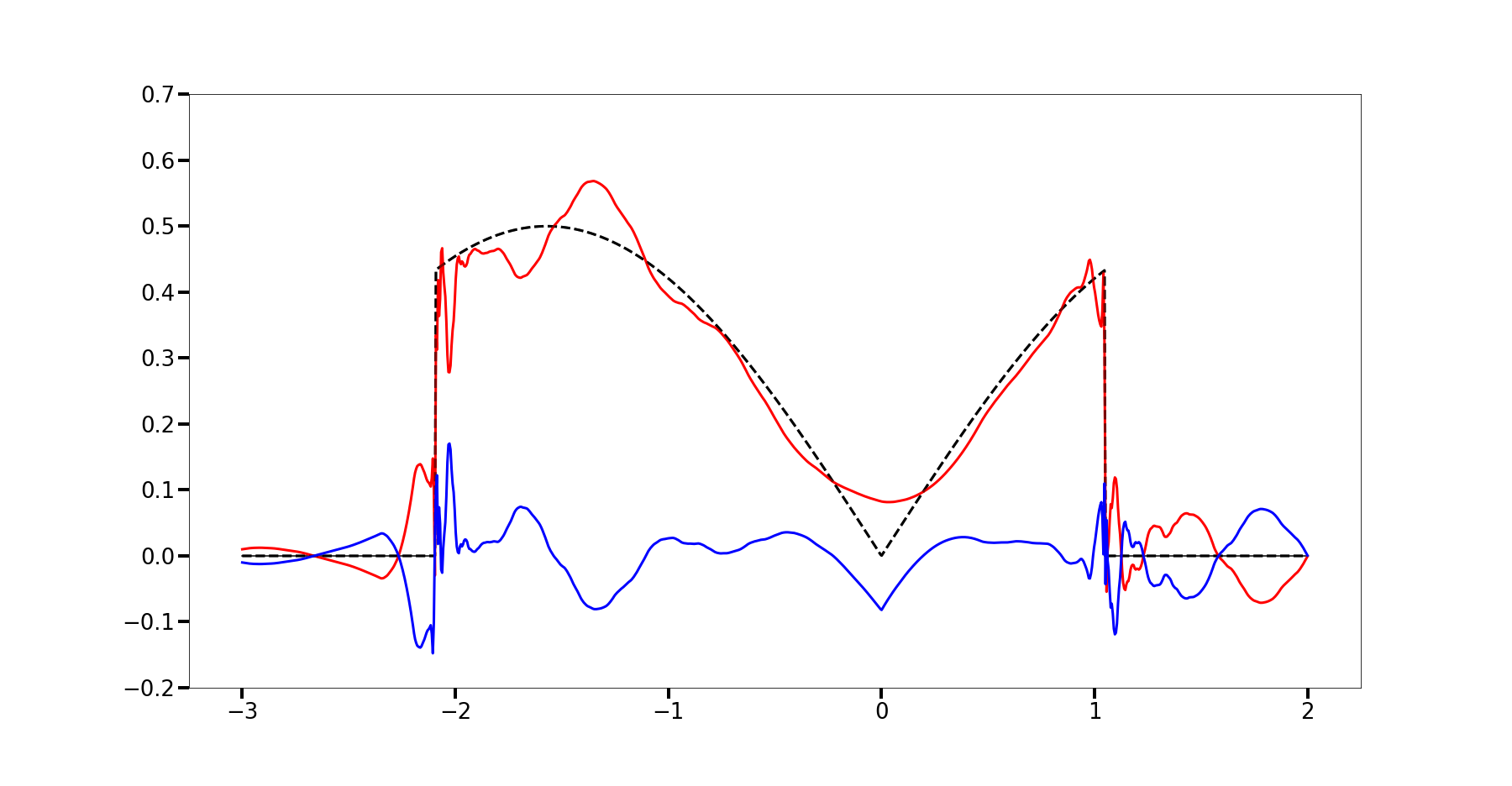}
        \caption{99 \% compression}
        \label{fig:perc_and_error_99_perc_sinusdensity}
    \end{subfigure}
    \caption{Target function (dashed black), learned function (red) and the error between the two (blue) for "the sinusoidal density" under high compression.}
    \label{fig:sinusdensity:subsequent}
\end{figure}


First we focus on faster learning and maximal compression, according to items 1 and 2 in Section \ref{s5}.
The comparative graphical analysis of Figures \ref{fig:lmtear:subsequent}--\ref{fig:sinusdensity:subsequent} leads to the following conclusions.

\begin{enumerate}[label=\arabic*.]
    \item Examples 1, 3, 4 are of piecewise smooth type, while Example 2 is of the fractal type.
    \item The decreasing rearrangement activation allows very fast learning combined with very high compression rate for the piecewise smooth curves: the quality of learning is superb at $85\%$ compression. In comparison, retaining such high quality of learning for the fractal curve in Figure 2 is possible only at compression rate up to $3-4\%$. (See item (a) in each of Figures \ref{fig:lmtear:subsequent}--\ref{fig:sinusdensity:subsequent}.)
    \item Approximation of the target function by the learned one is very good even for superhigh levels of compression ($98-99\%$). This also indicates that if the large-to-medium sample size $N=2^{10}$ be reduced to moderate or even small sizes (cf. Section 1), the rate of learning can be expected to be quite good, while the compression rates will decrease, but remain still quite good.
    \item The effect of subjecting the fractal-type curve to high or superhigh compression rates is that the learned curve get smoothed out to a piecewise smooth (few isolated singularities, similar to Examples 1,3 and 4) or even smooth -- no singularities at all. In \citep{llhm2022} and \citep{llhm2022_1} we shall show that if the fractality of the manifold is due to noise, then, learning the manifold with WBNNs where the decreasing arrangement activation is applied on the noisy (empirical) wavelet $\beta$--coefficients results in denoising and high-quality statistical estimation of the manifold. The level of compression resulting from this process can be useful in determining the chances for the manifold to have certain regularity.
\end{enumerate}


\begin{figure}
    \captionsetup{width=2.5\textwidth}
    \centering
    \begin{subfigure}[b]{0.5\textwidth}
        \captionsetup{width=2.5\textwidth}
        \centering
        \includegraphics[width=\textwidth]{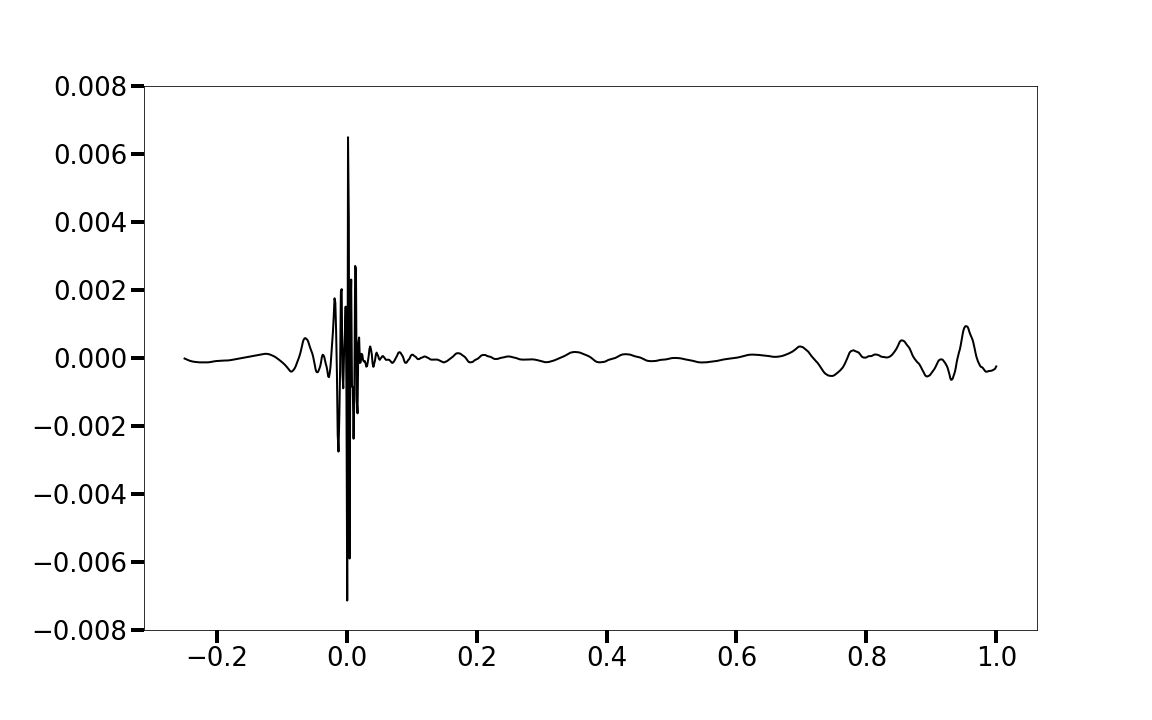}
        \caption{Distribution of benchmark MISE for "$\lambda$-tear", at compression rate $\approx 99.610 \%$. This is the benchmark MISE for all cases (a) -- (d).}
        \label{fig:same_mise_lambdatear}
    \end{subfigure}
    \hfill
    \begin{subfigure}[b]{0.5\textwidth}
        \captionsetup{width=2.5\textwidth}
        \centering
        \includegraphics[width=\textwidth]{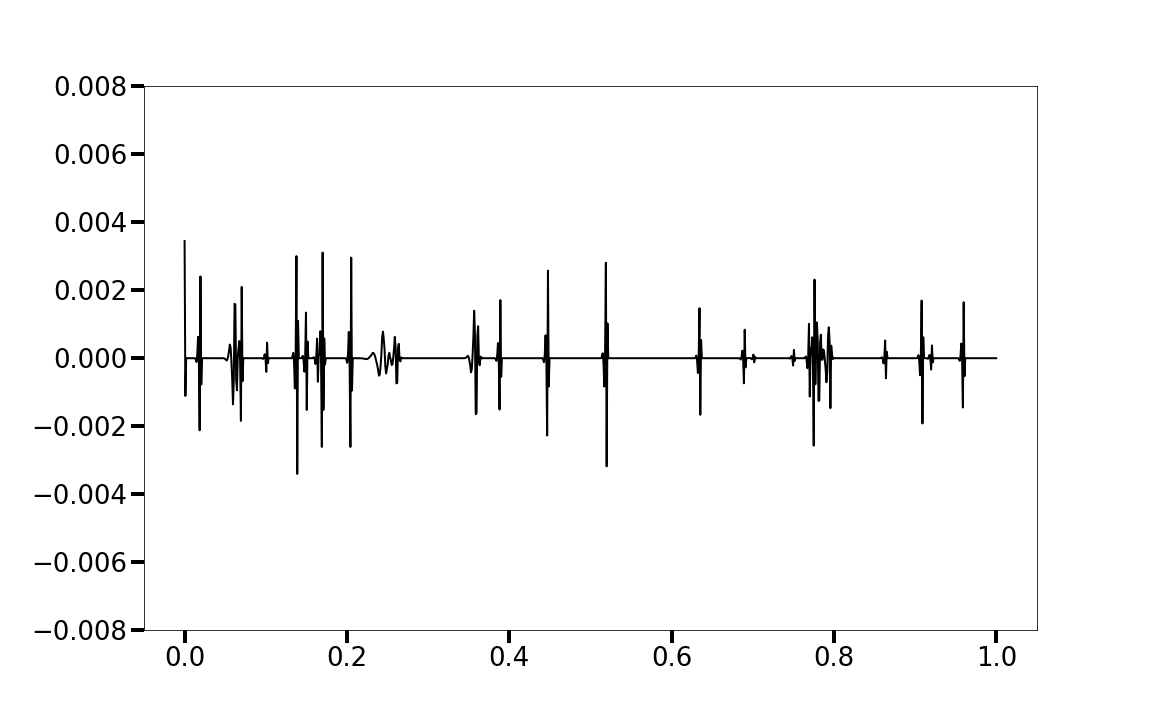}
        \caption{Distribution of benchmark MISE for "Weierstrass function", attained at compression rate $\approx 3.418 \%$.}
        \label{fig:same_mise_weierstrass}
    \end{subfigure}
    \hfill
    \centering
    \begin{subfigure}[b]{0.5\textwidth}
        \captionsetup{width=2.5\textwidth}
        \centering
        \includegraphics[width=\textwidth]{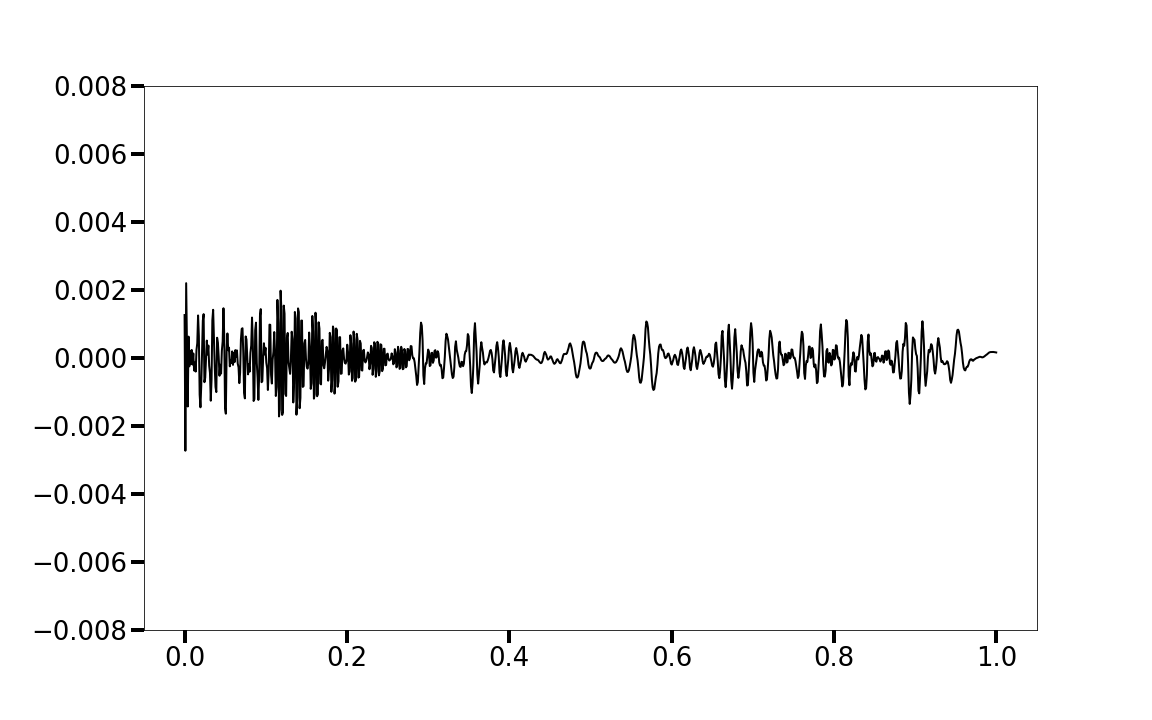}
        \caption{Distribution of benchmark MISE for "double chirp", attained at compression rate $\approx 83.984 \%$.}
        \label{fig:same_mise_dchirp}
    \end{subfigure}
    \hfill
    \begin{subfigure}[b]{0.5\textwidth}
        \captionsetup{width=2.5\textwidth}
        \centering
        \includegraphics[width=\textwidth]{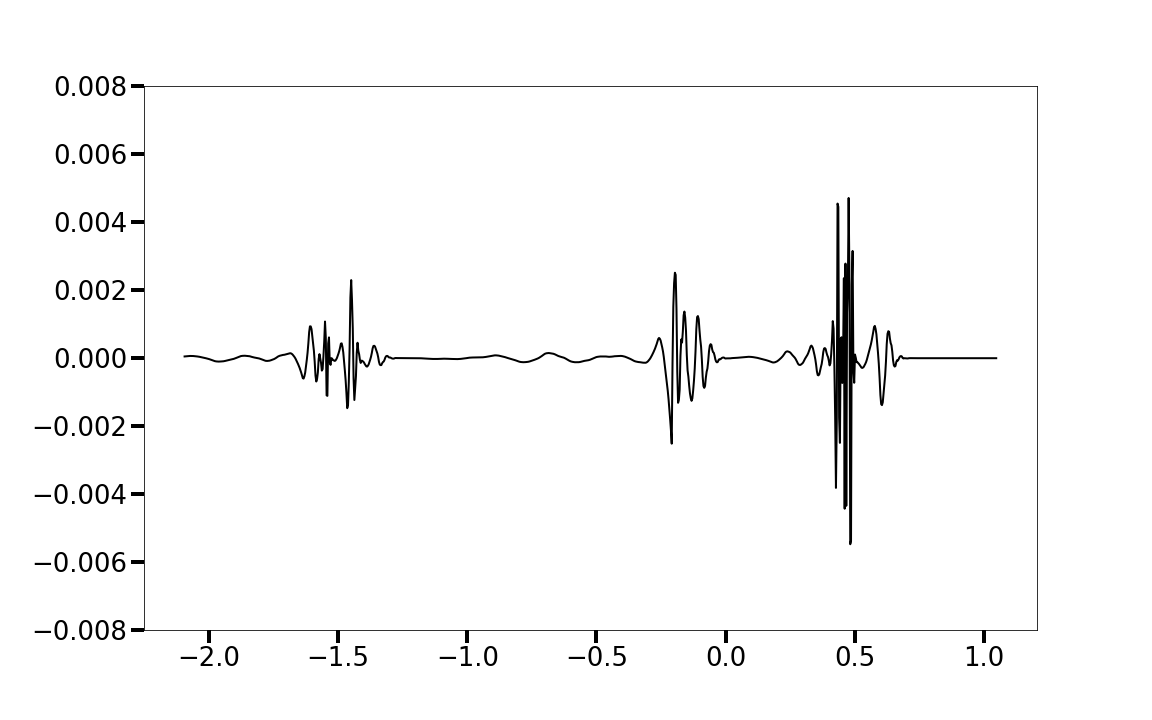}
        \caption{Distribution of benchmark MISE for "sinusoidal density", attained at compression rate $\approx 83.984 \%$.}
        \label{fig:same_mise_sdens}
    \end{subfigure}
    \caption{Benchmark error distribution for Examples 1-4. Benchmark error was MISE for Example 1 at compression $\approx 99.610 \%$.}
    \label{fig:mise:subsequent}
\end{figure}

The comparison of graphical data for the four examples on Figure \ref{fig:mise:subsequent} leads to the following observations.

\begin{enumerate}[label=\arabic*., resume]
    \item For piecewise smooth manifolds having isolated singularities of \emph{the first kind} only (left and right onesided limits exist at the singularity) the distribution of the benchmark MISE error is narrowly concentrated in small neighbourhoods of the respective singularities. This also shows that the vector of $\beta$--coefficients of such manifolds tends to be sparse, i.e., it contains large in absolute value coefficients on all levels $j$ only in small neighbourhoods of the singularities. Thus, manifolds of this type are highly compressible with threshold activation methods and are the fastest to learn with WBNNs. Typical examples are 1 and 4, see Figure \ref{fig:mise:subsequent} (a) and (d), resp.
    \item Piecewise-smooth manifolds with isolated singularities of \emph{the second kind} (at least one of the onesided limits does not exist at the singularity). Typical case of this type of singularity is the presence of a chirp on the side of the missing onesided limit (for example, functions $g(x)=x^a \sin \frac{1}{x^b}$, $x > 0, a > 0, b > 0)$. Chirps can be very spatially inhomogeneous and even exhibit some fractal properties (for example, the function $g$ in the above formula, with $a = 0$, has unbounded variation in a neighbourhood of $x = 0$ and the part of its graph in this neighbourhood is infinitely long. Thus, functions with chirps take a somewhat intermediate place between the functions in item 5 and the function in the next item 7. Typical example is 3 -- see Figure \ref{fig:mise:subsequent} (c).
    \item Fractal-type curves, especially ones with locally constant H{\"o}lder index, gather their Besov regularity from a dense vector of $\beta$--coefficients where all, or, at least, the vast majority of $\beta$--coefficients provide significant contribution which can only be ignored at the price of slowing the rate and decreasing the quality of the learning process. In \citep{llhm2022} and \citep{llhm2022_1} we shall show that manifolds of fractal type can be learned well only using non-threshold shrinkage activation. This type of activation produces $0\%$ compression. Threshold activation methods, including the decreasing rearrangement activation, oversmooth the manifold, thereby altering its fractal type. This explains the low compression rate when achieving the benchmark MISE in the typical Example 2 -- see Figure \ref{fig:mise:subsequent} (b).
    \item (Remark.) One of the main goals of research in \citep{llhm2022} and \citep{llhm2022_1} will be to design \emph{hybrid activation strategies for adaptive learning} by composing a sequence of activation strategies of diverse -- threshold and non--threshold -- nature which will be achieved by the use of \emph{deep WBNNs}, each single-layer WBNN in which will contribute with its own activation strategy.
\end{enumerate}


\begin{figure}[h]
\centering
\includegraphics[width=\textwidth]{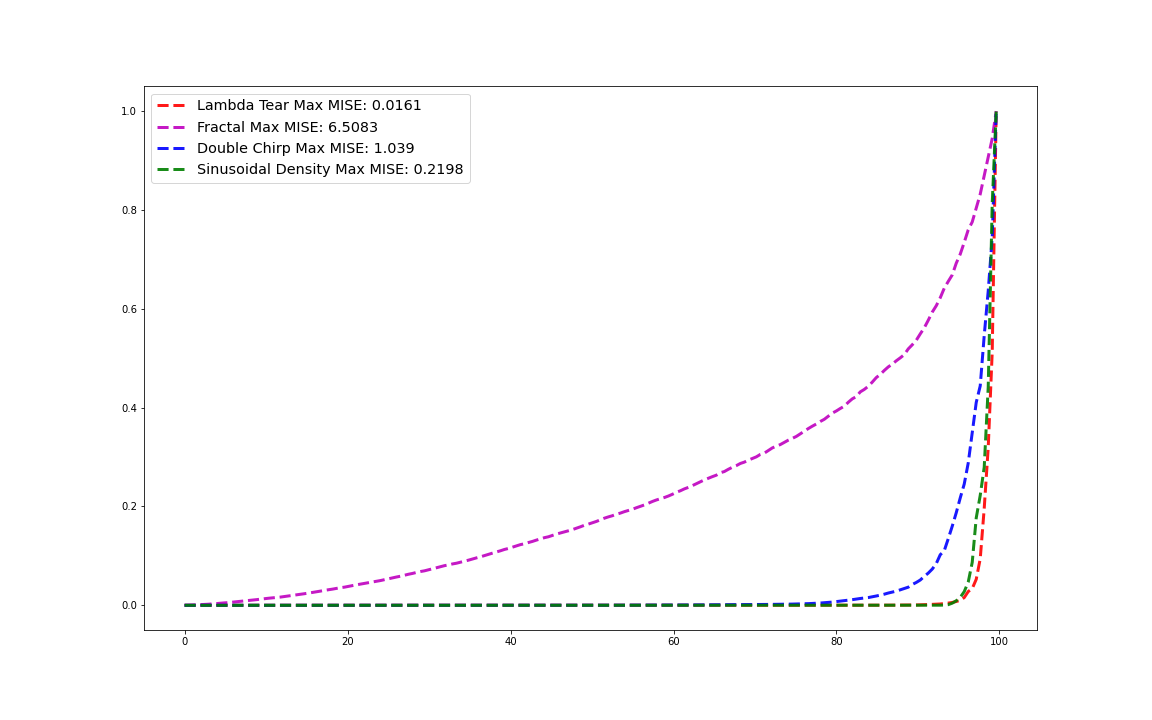}
\caption{Compression percentage (x-axis) vs standardized relative MISE for Examples 1-4}
\label{fig:mise_compare}
\end{figure}


While Figure \ref{fig:mise:subsequent} provided detailed information about the local distribution of a benchmark MISE, Figure \ref{fig:mise_compare} provides an insightful comparison of the ratio between compression percentage and relative MISE for each of the four considered examples.

\begin{enumerate}[label=\arabic*., resume]
    \item The aspect, in which Figure \ref{fig:mise_compare} is most insightful, is the comparative determination of the fractality type of the curves in each of Examples 1--4: "the Weierstrass curve" of Example 2 exhibits markedly fractal behaviour, followed by the "double chirp" of Example 3 exhibiting a somehow 'semi--fractal' behaviour, and with the curves in Example 1 and 4 being of markedly piecewise smooth type.
    \item (Remark.) Tracing the behaviour of the "double chirp" of Example 3 in Figures 5, 7(c) and 8, some notable differences are observed with "the $\lambda$--tear" in Example 1, despite of the fact that they are both of the piecewise smooth type and, especially, despite of the fact that they have \emph{exactly the same Besov regularity}.
    The explanation of this phenomenon is that although the Besov norms of $f_1$ and $f_3$ with the same exact parameters are both finite, the norm of $f_3$ is several orders of decimal magnitude larger than the norm of $f_1$, mainly due to the presence of the factor $64 = 2^6$ in the sine component of the formula (\ref{eq:48}) for $f_3$.
\end{enumerate}


\begin{figure}
    \centering
    \begin{subfigure}[b]{0.495\textwidth}
        \captionsetup{width=0.8\textwidth}
        \centering
        \includegraphics[width=\textwidth]{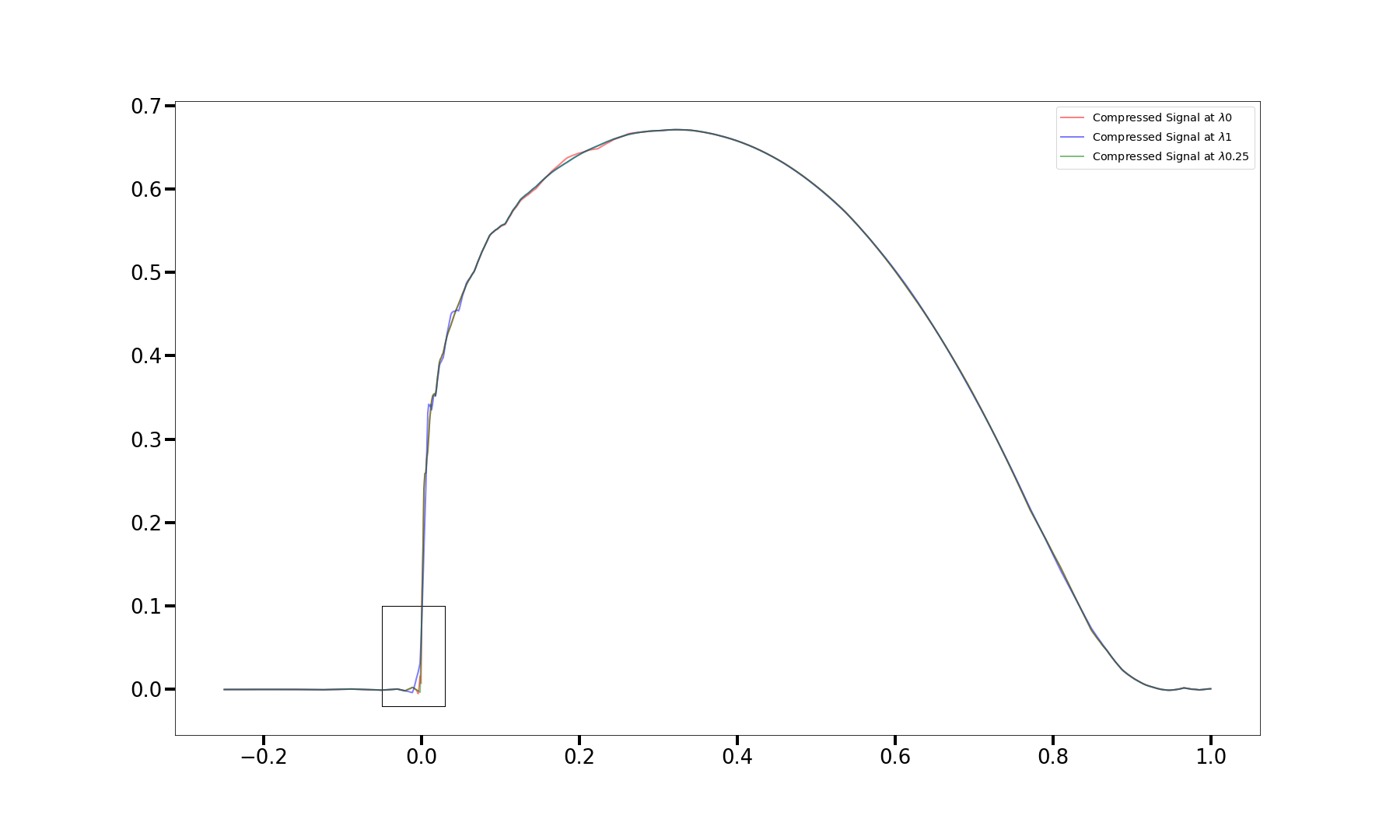}
        \caption{Zoomed region for Example 1 (see Fig. \ref{fig:swarm_zoom_lambdatear}), compression rate $98 \%$}
        \label{fig:swarm_lambdatear}
    \end{subfigure}
    \begin{subfigure}[b]{0.495\textwidth}
        \captionsetup{width=0.8\textwidth}
        \centering
        \includegraphics[width=\textwidth]{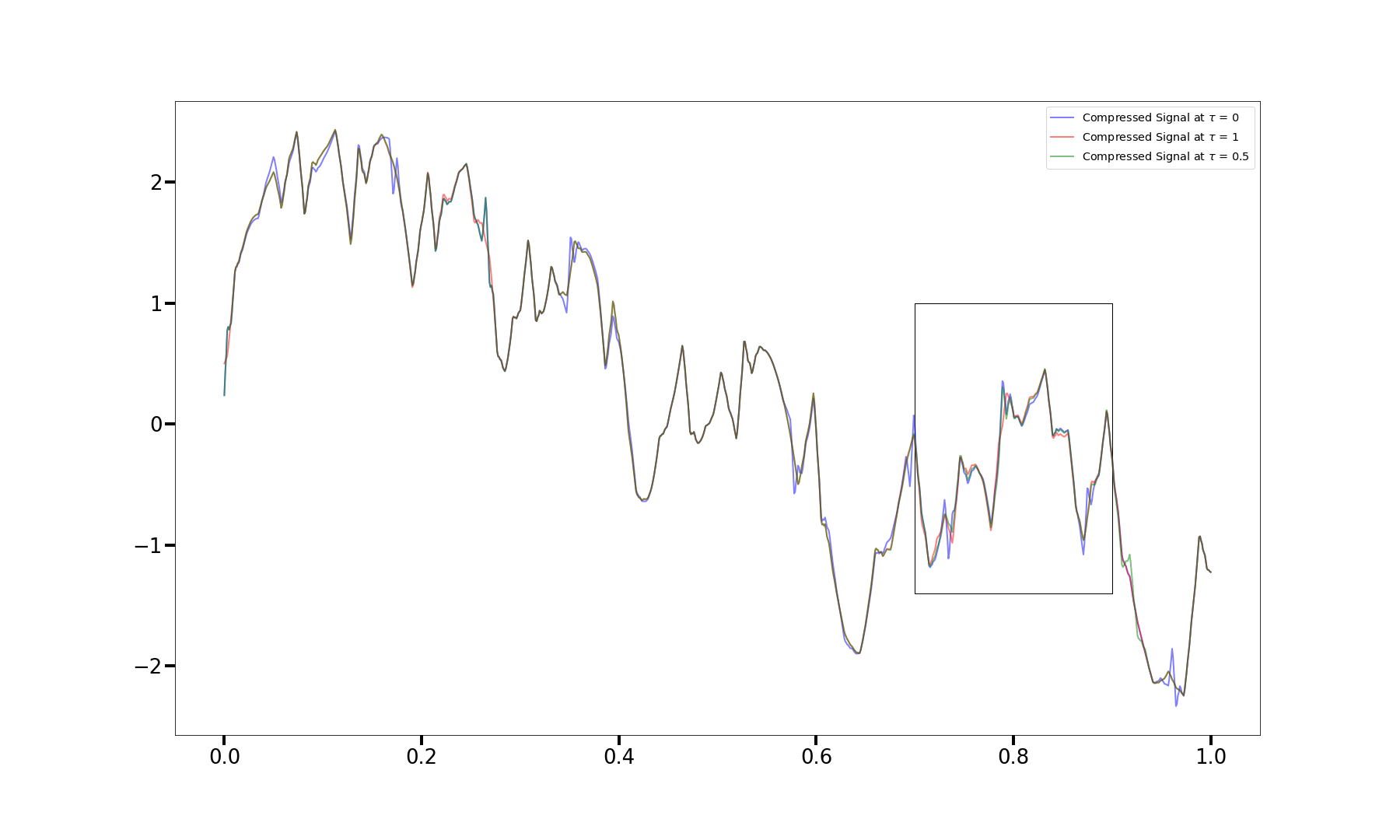}
        \caption{Zoomed region for Example 2 (see Fig. \ref{fig:swarm_zoom_weierstrass}), compression rate $91 \%$}
        \label{fig:swarm_weierstrass}
    \end{subfigure}
    \caption{Local zooming for Examples 1 and 2}
    \label{fig:swarm:1}
\end{figure}

\begin{figure}
    \centering
    \begin{subfigure}[b]{0.495\textwidth}
        \captionsetup{width=0.8\textwidth}
        \centering
        \includegraphics[width=\textwidth]{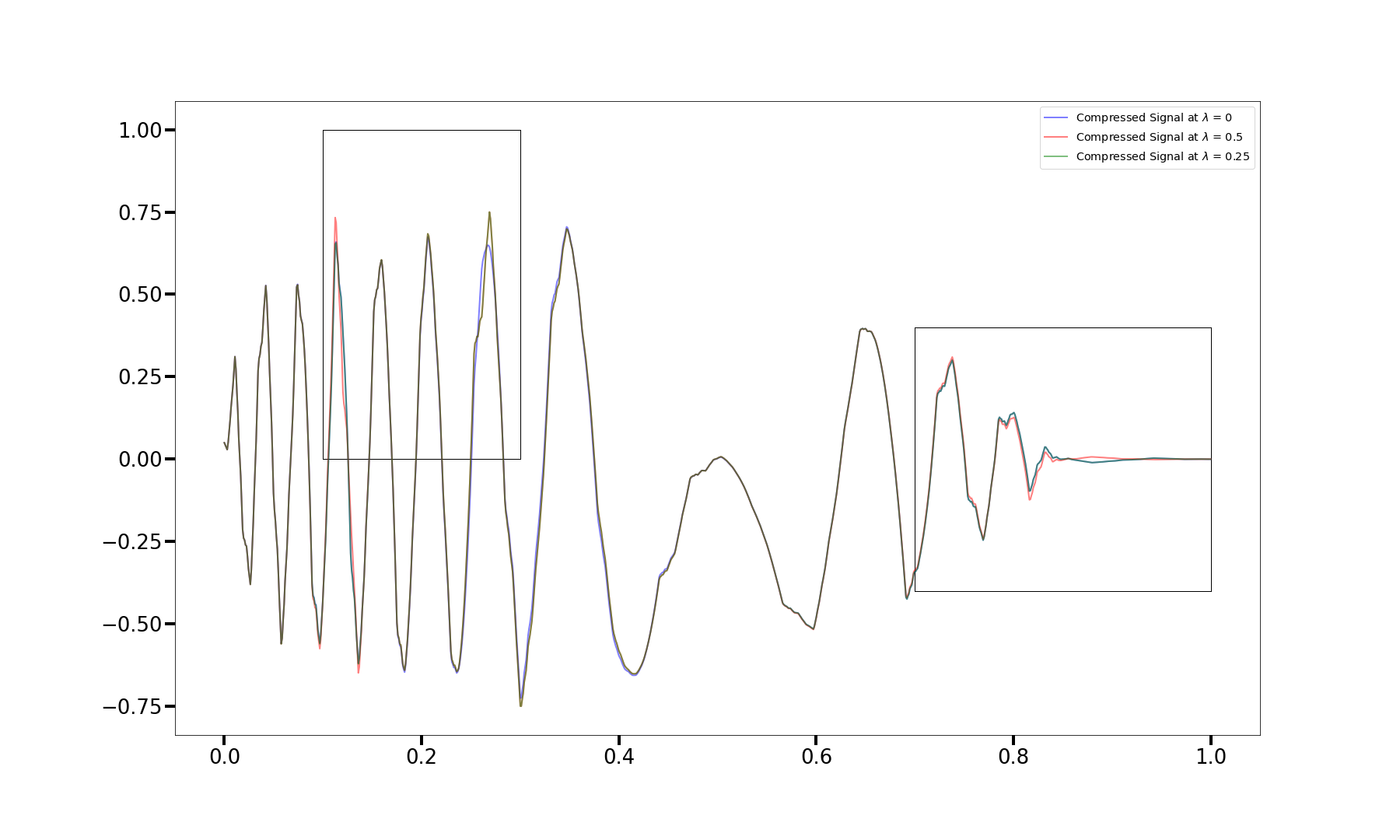}
        \caption{Zoomed regions for Example 3 (from left to right, see Fig. 12 \ref{fig:swarm_zoom_dchirp}) and \ref{fig:swarm_zoom_dchirp2}), resp.), compression rate $96 \%$}
        \label{fig:swarm_dchirp}
    \end{subfigure}
    \begin{subfigure}[b]{0.495\textwidth}
        \captionsetup{width=0.8\textwidth}
        \centering
        \includegraphics[width=\textwidth]{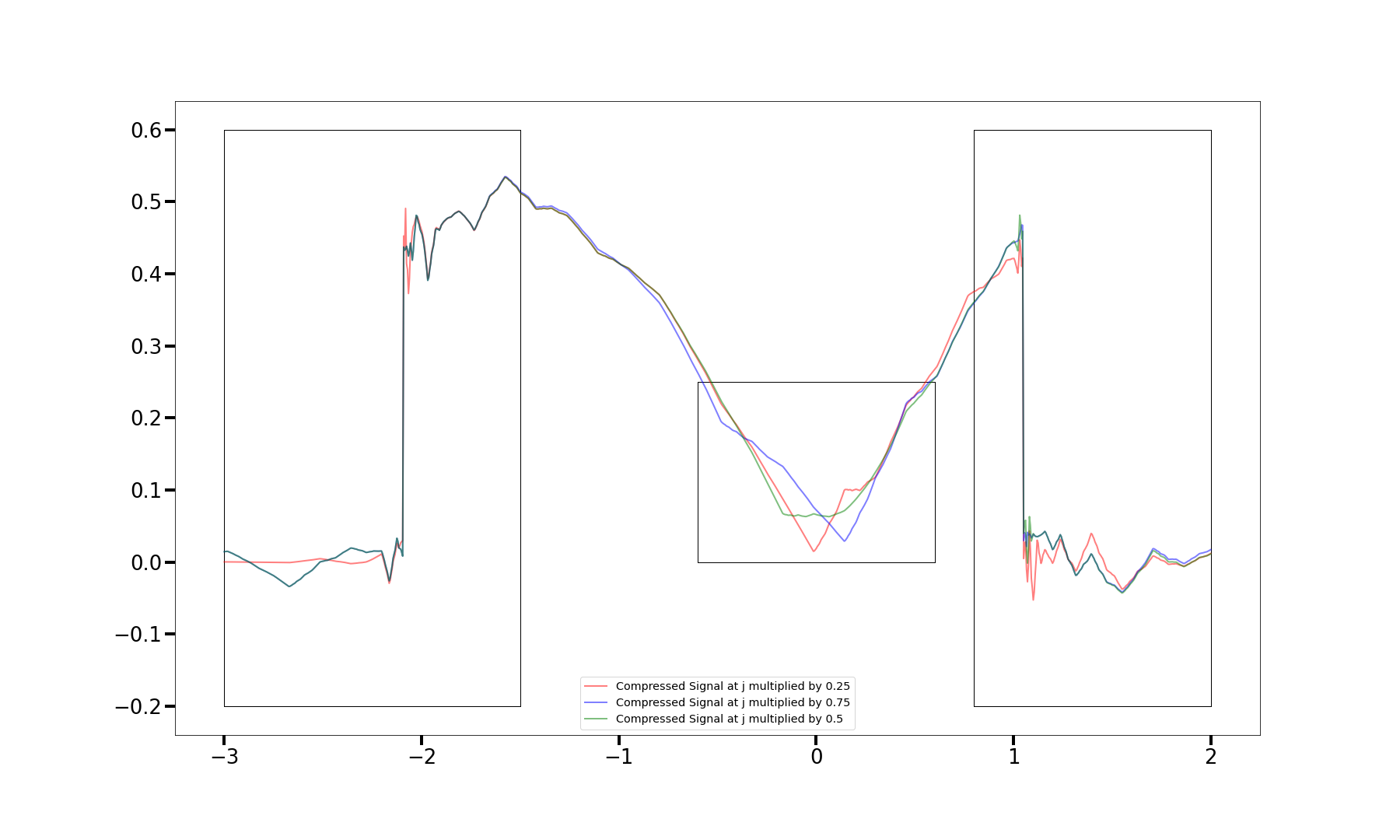}
        \caption{Zoomed regions for Example 4 (from left to right, see (\ref{fig:swarm_zoom_sdens}), \ref{fig:swarm_zoom_sdens2}) and \ref{fig:swarm_zoom_sdens3}), resp.), compression rate $98 \%$}
        \label{fig:swarm_sdens}
    \end{subfigure}
    \caption{Local zooming for Examples 3 and 4}
    \label{fig:swarm:2}
\end{figure}


In Section \ref{s6} the exact Besov regularity of the example was known, and it was used in the construction of the activation operator.
What if only approximate information is available about each of the parameters $(p, q, s)$ of the Besov regularity?
The algorithmically simplest way to overcome this ambiguity is to use a swarm of sufficiently broad single-layer WBNNs.
As discussed in Section \ref{s4}, using a deep WBNN is possible, but requires adjustments, with some loss of efficiency.
This is why here the new research topic about relation between swarm and deep evolutionary AI (\citep{iba2022}) is of great interest.

In Figures 9--13 we produce the graphical results of sequential emulation of the parallel learning process of a 'swarm' of 3 single-layered sufficiently broad (satisfying (\ref{eq:14})) WBNNs, one of which is biased towards underestimating the Besov regularity (blue colour), the second one is using the exact Besov regularity information (green colour) and the last one is biased towards overestimating the Besov regularity.
The graphical results for Examples 1--4 are presented in Figure 9(a) and 9(b), and Figure 10(a) and 10(b), resp.
The rectangular regions on these figures, where the differences are most notable, are marked with window frames.


\begin{figure}
    \centering
    \begin{subfigure}[b]{0.495\textwidth}
        \centering
        \includegraphics[width=\textwidth]{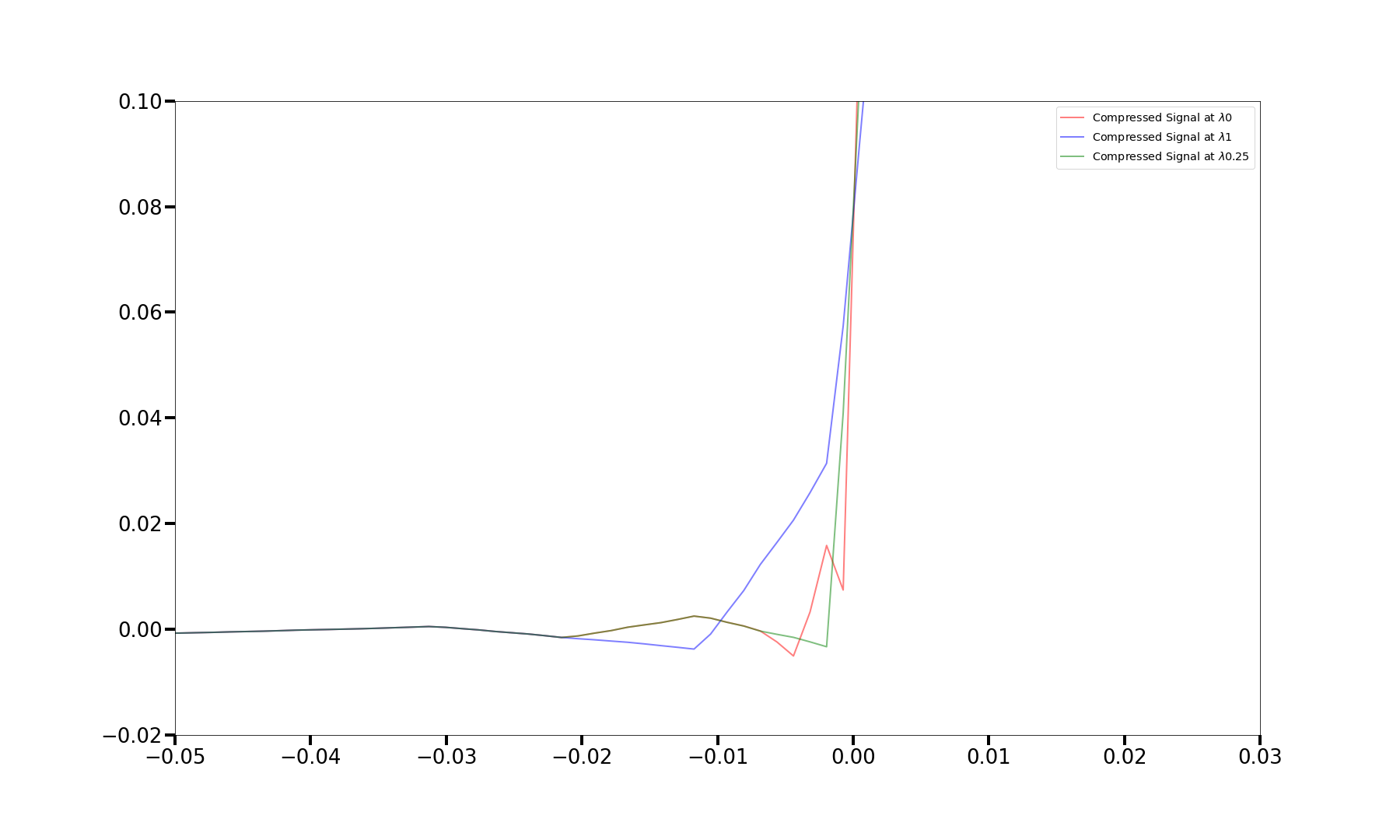}
        \caption{}
        \label{fig:swarm_zoom_lambdatear}
    \end{subfigure}
    \begin{subfigure}[b]{0.495\textwidth}
        \centering
        \includegraphics[width=\textwidth]{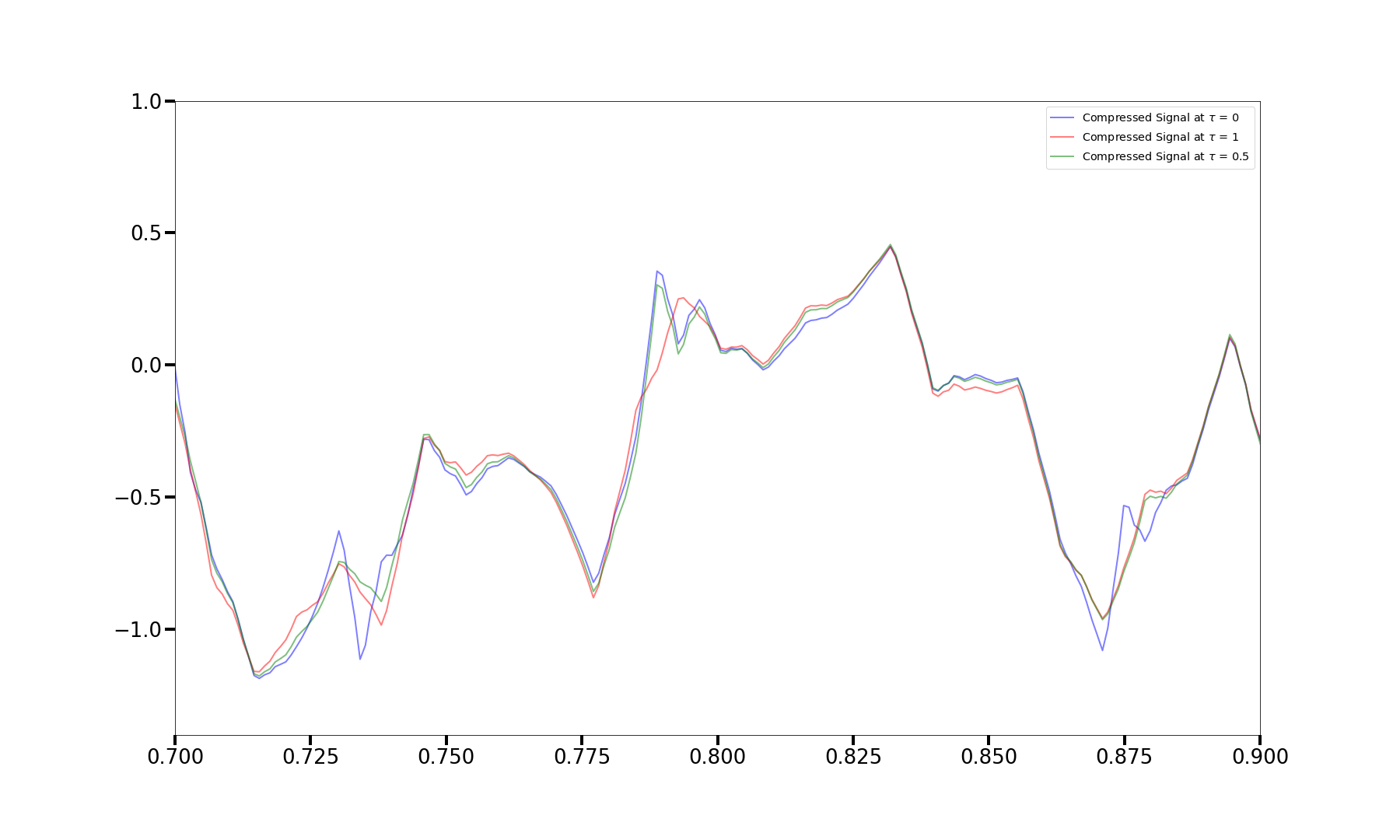}
        \caption{}
        \label{fig:swarm_zoom_weierstrass}
    \end{subfigure}
    \caption{Region zoom for (a) Example 1, (b) Example 2 (see Fig. \ref{fig:swarm_lambdatear} and \ref{fig:swarm_weierstrass}), resp.}
    \label{fig:swarm_zoom:lt_frct}
\end{figure}

\begin{figure}
    \centering
    \begin{subfigure}[b]{0.495\textwidth}
        \centering
        \includegraphics[width=\textwidth]{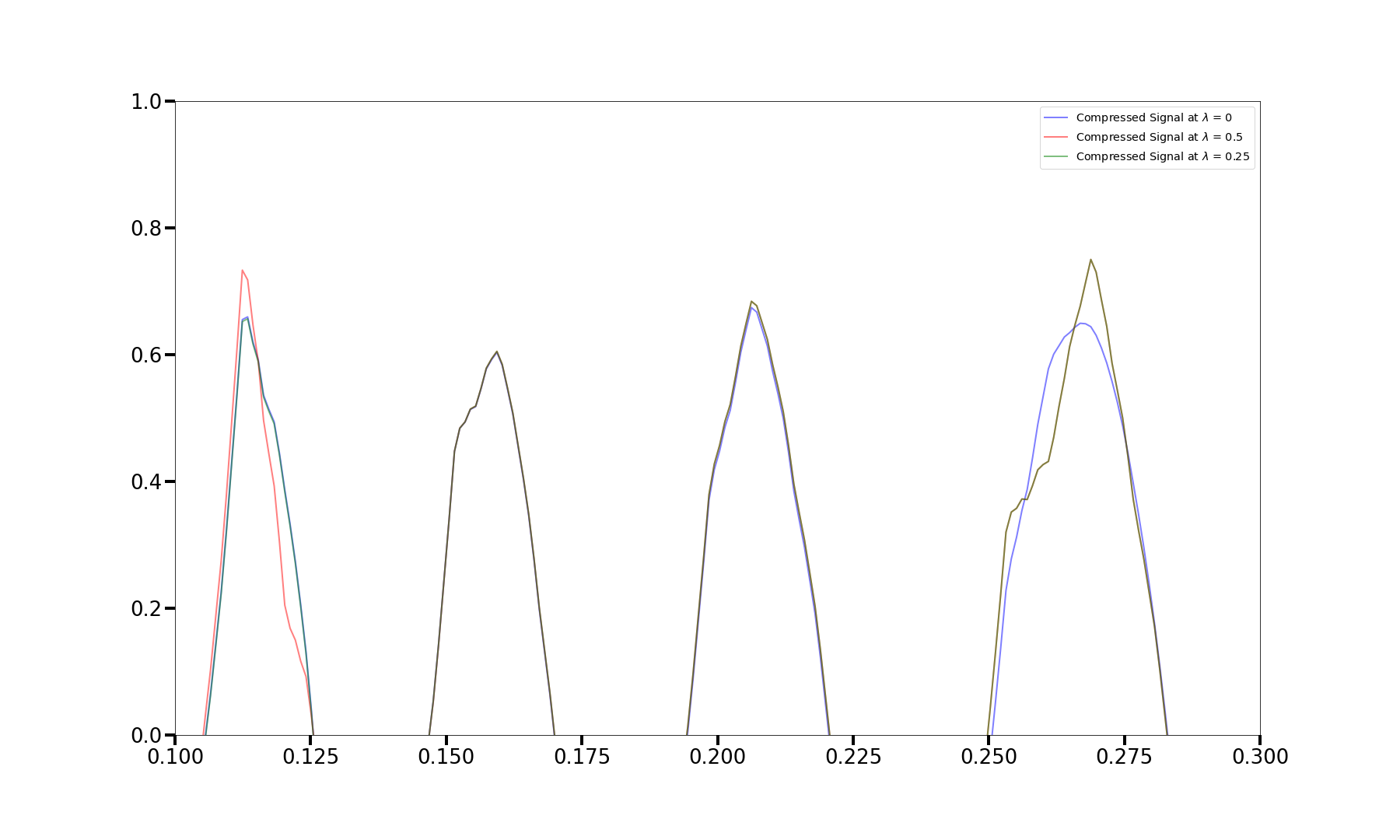}
        \caption{}
        \label{fig:swarm_zoom_dchirp}
    \end{subfigure}
    \begin{subfigure}[b]{0.495\textwidth}
        \centering
        \includegraphics[width=\textwidth]{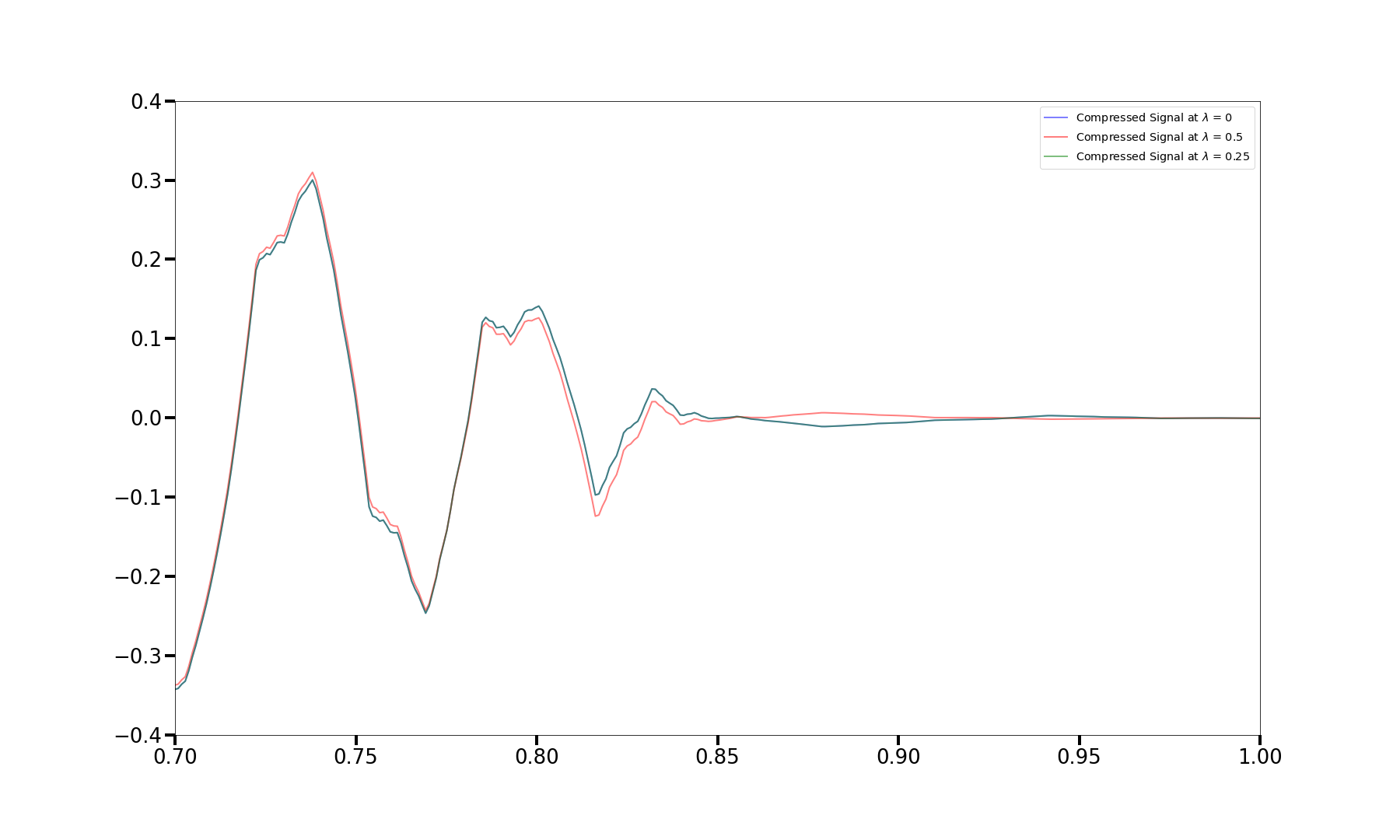}
        \caption{}
        \label{fig:swarm_zoom_dchirp2}
    \end{subfigure}
    \caption{Region zoom for Example 3: see Fig \ref{fig:swarm_dchirp}}
    \label{fig:swarm_zoom:dchirp}
\end{figure}

\begin{figure}
    \centering
    \begin{subfigure}[b]{0.595\textwidth}
        \centering
        \includegraphics[width=\textwidth]{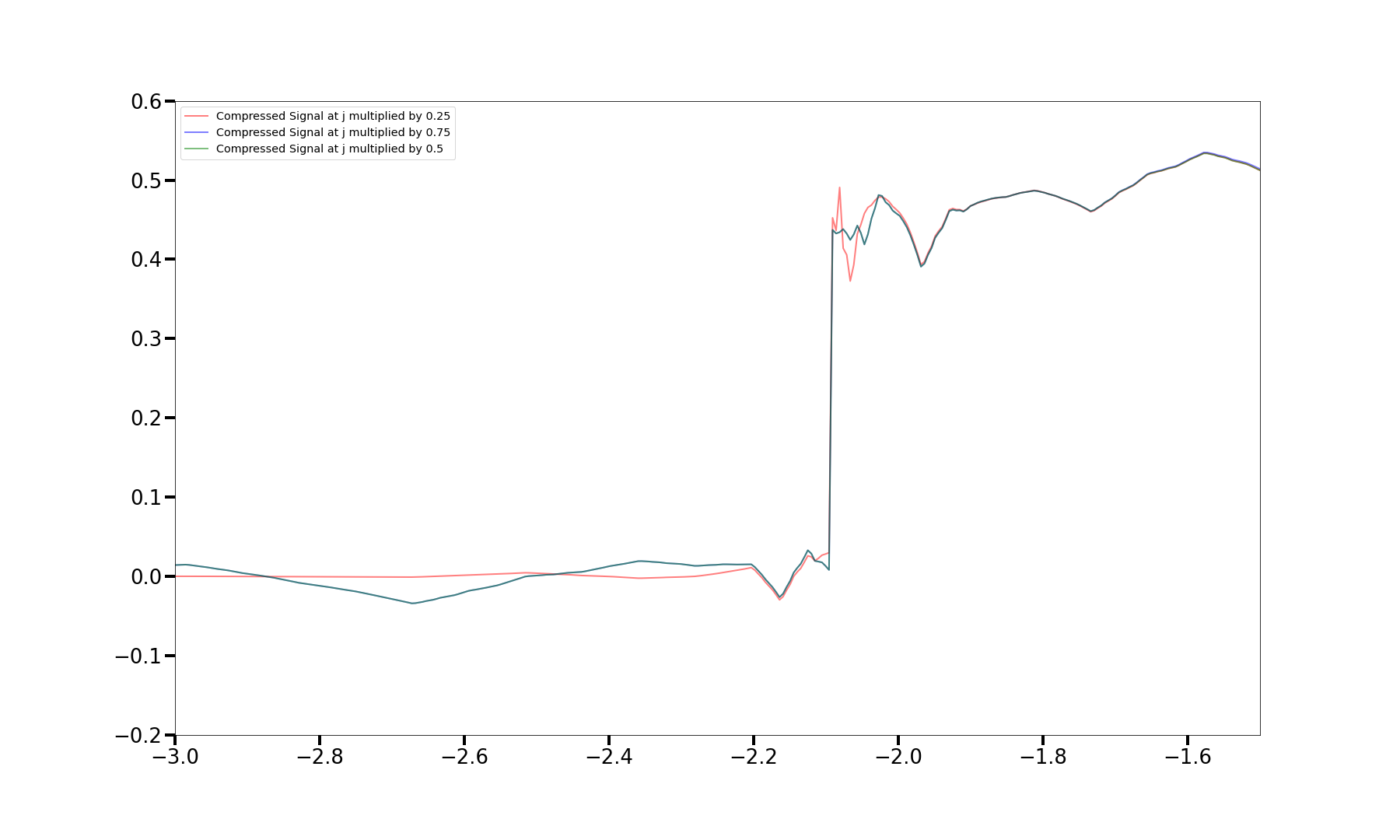}
        \caption{}
        \label{fig:swarm_zoom_sdens}
    \end{subfigure}
    \begin{subfigure}[b]{0.595\textwidth}
        \centering
        \includegraphics[width=\textwidth]{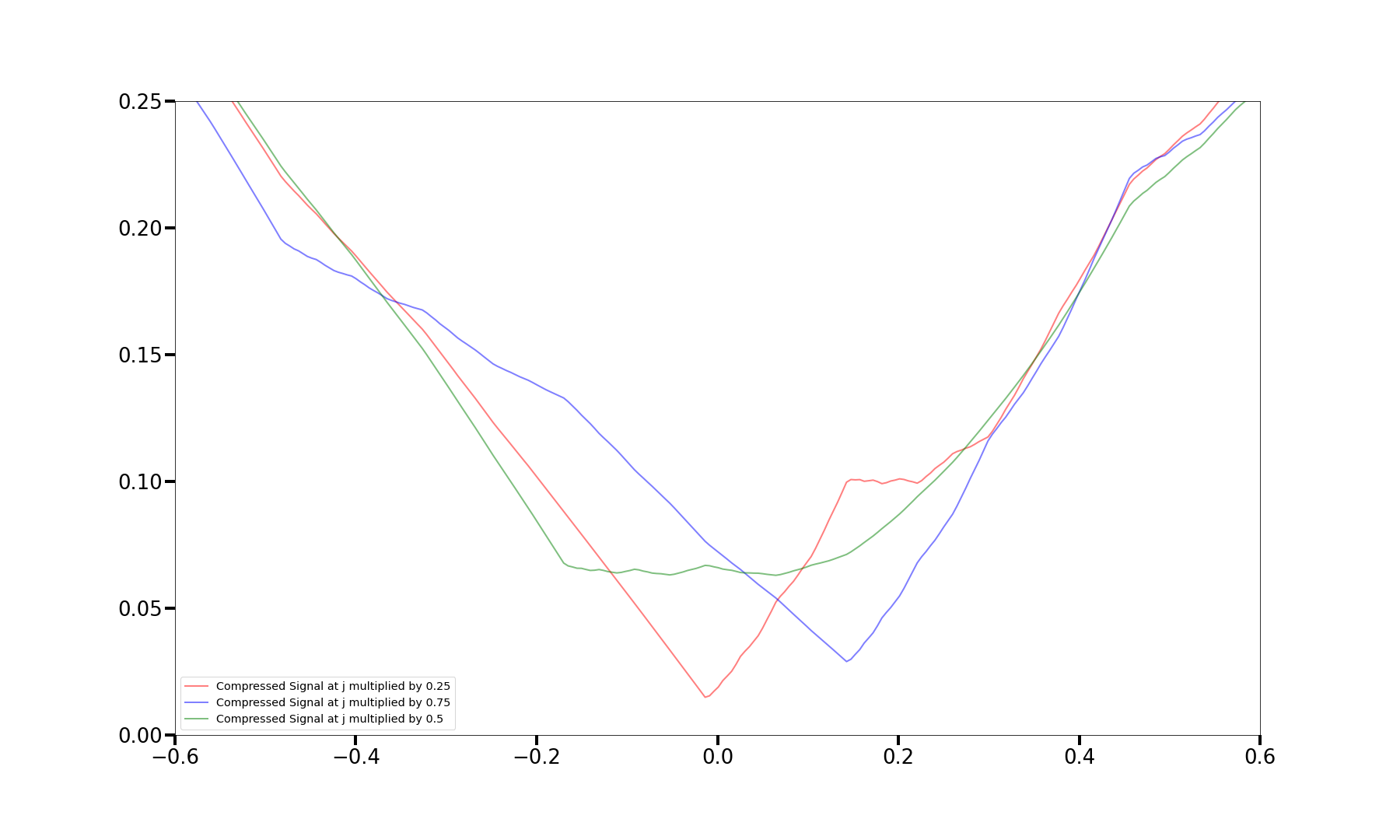}
        \caption{}
        \label{fig:swarm_zoom_sdens2}
    \end{subfigure}
    \begin{subfigure}[b]{0.595\textwidth}
        \centering
        \includegraphics[width=\textwidth]{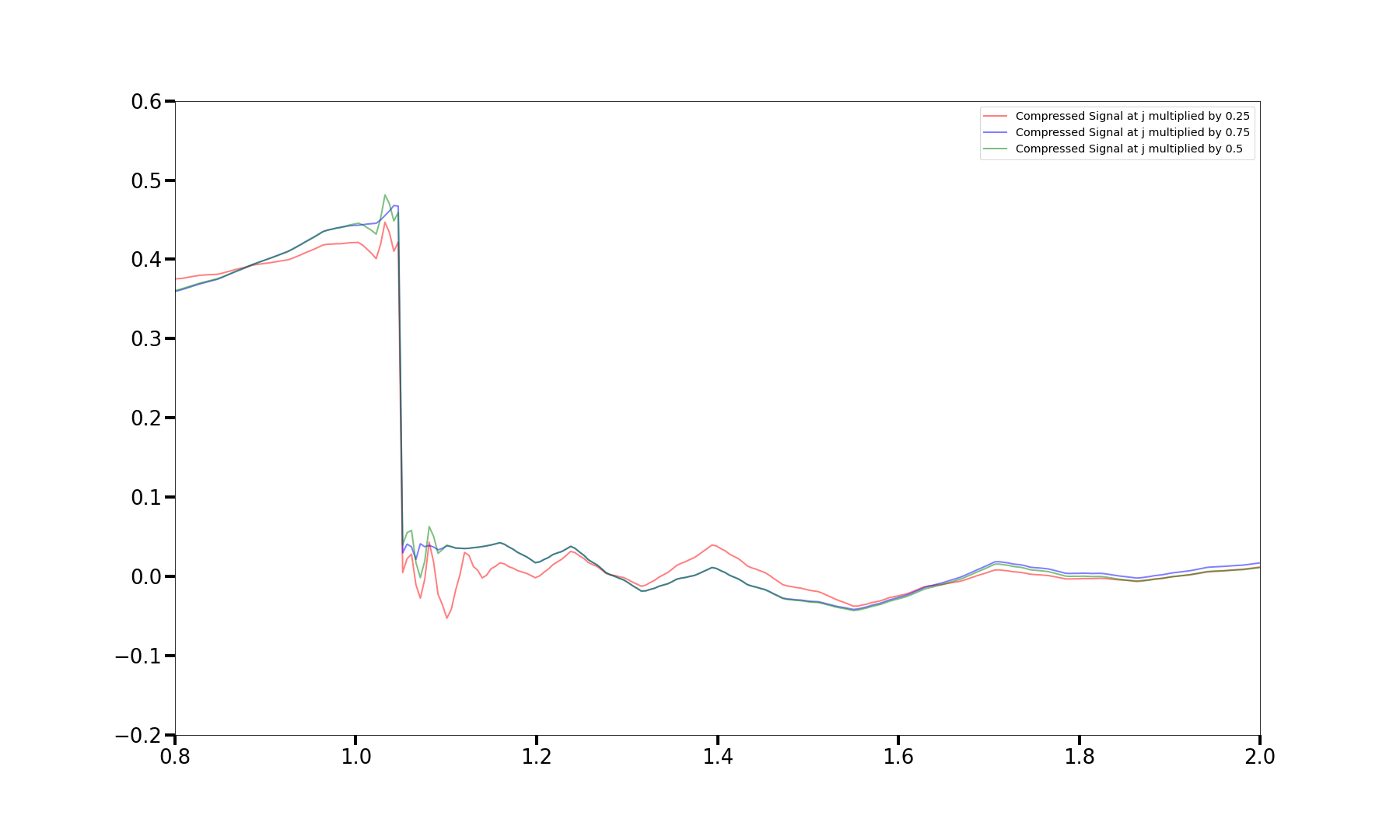}
        \caption{}
        \label{fig:swarm_zoom_sdens3}
    \end{subfigure}
    \caption{Region zooms for Example 4: see Fig \ref{fig:swarm_sdens}}
    \label{fig:swarm_zoom:sdens}
\end{figure}


In Figures 11--13 are given zoomings of all windows in Figures 9 and 10, as follows.
\begin{itemize}
    \item The window in Figure 10(a): Figure 11(a)
    \item The window in Figure 10(b): Figure 11(b)
    \item The two windows in Figure 10(a): Figure 12(a) and (b)
    \item The three windows in Figure 10(b): Figure 13(a), (b) and (c)
\end{itemize}

Some observations from the graphical comparison, as follows:

\begin{enumerate}[label=\arabic*., resume]
    \item The learning process using WBNNs activated by decreasing rearrangement method proves to be \emph{very robust with respect to errors of estimation of the Besov regularity information} when \emph{manifolds without singularities} are being learned.
    \item In the presence of singularities, the robustness in item 11 \emph{decreases}: the more the singularities, the less the robustness, with the maximal deterioration of robustness being in the case of learning fractal-type manifolds. However, due to the relatively uniform distribution over all, or most of the $\beta$--coefficients of a fractal manifold (see item 7 above), for the low compression rate in Figure 7(b), the differences between the blue, green and red lines in Figure 9(b) would hardly be noticeable (see also item 13).
    \item Due to the robustness properties in items 11 and 12 above, it makes sense to use \emph{large-size swarms of single-layered WBNNs} only in the case of learning fractal-type manifolds, for a reason that will be explained in item 15.
    \item Learning of piecewise smooth manifolds with WBNNs activated via the decreasing rearrangement method is robust with respect to small errors of underestimating or overestimating the manifold's Besov regularity, as long as compression rate is none or relatively small. This robustness \emph{rapidly deteriorates} with the increase of the compression rate, but at the same time the quality of learning \emph{deteriorates slowly} with the increase of the compression rate (which is an equivalent way of saying that piecewise smooth manifolds are being learnt fast). This is why in the case of learning a piecewise smooth manifold, small to medium sized swarms of sufficiently broad single-layered WBNNs are expected to be sufficient when the Besov regularity of the manifold is only approximately known, see Figure 11(a) and Figures 12 and 13.
    \item In the case of fractal manifolds, the crucial difference is that, unlike the case of piecewise smooth manifolds, the quality of learning of the fractal \emph{deteriorates rapidly together with the rapid deterioration of robustness} with respect to errors in the Besov regularity estimation, when the compression rates increases. This is why, contrary to the comparable compression rates in Figure 7(a) vs. Figure 9(a), Figure 7(c) vs. Figure 10(a) and Figure 7(d) vs. Figure 10(b), resp., there is a sharp difference in the compression rates in Figure 7(b) vs. Figure 9(b): less than $4 \%$ vs. more than $90 \%$, resp. Because of this important difference, it can be expected that in the current context large-size swarms of sufficiently broad single-layered WBNNs would be needed to attain high quality of learning fractal manifolds. Notice that in the context of item 4, when the fractality is due to noise, the denoised manifold may still be of fractal type, or it may be of piecewise smooth type. In the context of the present item, in this case determination of the size of the swarm requires special care (this topic will be addressed in \citep{llhm2022}).
    \item Finally, let us turn our attention on the three singularities in Figure 13.
    Comparing the three singularities in Figure 13 with the singularity in Figure 11(a), it is observed that the latter singularity is somehow intermediate between the one in Figure 13(b) and the two in Figure 13(a) and (c).
    This is so, because, on the one hand, the singularities in Figure 13(b) and 11(a) are of the same type -- discontinuity of the first derivative $f_1^\prime$, resp. $f_4^\prime$ and, on the other hand, $f_1^\prime (0+) = + \infty$, which forces the graph of $f_1$ left of $0$ to resemble that of the graph of $f_4$ at the left discontinuity point in Figure 13(a) and, modulo vertical axial symmetry, also the graph of $f_4$ in Figure 13(c).
    What is remarkable about the discontinuity singularities of the first kind in Figure 13(a) and (c) (and, to some less expressed extent, also about the singularity in Figure 11(a)) is the presence of \emph{Gibbs phenomenon}. As it can be seen from Figure 3(a) and Figure 6(a), \emph{even at high compression rates} there is \emph{no Gibbs phenomenon at all}, which is due to the selection of \emph{compactly supported} wavelet basis.
    If in this context a trigonometric basis is used there will be very significant Gibbs phenomenon even at compression rate $0\%$.
    Thus, we conclude that in the case of use of compactly supported wavelet basis, notable Gibbs phenomenon may eventually appear at jump points only at superhigh levels of compression, and it is due to uncompressed $\beta$-coefficients with low $j$ ($j = j_0$ or $j$ near to $j_0$).
    It is also at these superhigh levels of compression, and for the same reason, that errors in the estimation of Besov regularity can result in the decreasing rearrangement activation producing notable differences in regions with Gibbs phenomenon.
\end{enumerate}

\section{Proofs}\label{s7}

\begin{proof}[Proof of Theorem \ref{th:1}]
    Let first $s : 0 < s < 1$. Using the the definitions and notations of \citep[Section 2.2]{dechevsky98} for $\rho, k^*, p $ and $q$, the upper bound

    \begin{equation}
        \exists \: c_1 < + \infty : R(f,\hat{f}_N) \leq c_1 N^{-\frac{s}{1+2s}}
    \end{equation}

    where $c_1=c_1(s, \mathrm{diam}(\mathrm{supp}(f)))$, follows after taking power $\frac{1}{\rho}$ from the two sides of \citep[(2.2.3]{dechevsky98} where the choice $q=+ \infty$ has been made. Under the assumptions of \citep[Corollary 2.2.4]{dechevsky98} on $f$, taking in consideration the compactness of $\mathrm{supp} f$ implies the equivalence of the assumptions in \citep[Corollary 2.2.4]{dechevsky98} and the current assumption $f \in B_{p\infty}^s (\mathbb{R})$

    The lower bound 

    \begin{equation}
        \exists \: c_0 > 0 : R (f, \hat{f}_N) \geq c_0 N^{-\frac{s}{1+2s}}
    \end{equation}

    where $c_0 = c_0(s, \mathrm{diam}(\mathrm{supp} f))$, follows from \citep[Theorem 2.3.2]{dechevsky98} for $q=\infty$.
    The selection of $J$ in (\ref{eq:14}) assumes that in all cases considered in \citep[Section 2.2]{dechevsky98}, the optimal level $k^*$ defined there, always satisfies

    \begin{equation}
        j_0 \leq k^* \leq J
    \end{equation}

    which ensures the validity of all upper and lower bounds in \citep[Corollaries 2.2.2-11 and Theorem 2.3.2]{dechevsky98}. Under the assumptions $q=\infty$ and compactness of $\mathrm{supp} \: f$, the norm ${||.||}_{p,s}$ defined in the formulation of \citep[Theorem 2.3.2]{dechevsky98} can be replaced by the simpler ${||.||}_{B_{p\infty}^s (\mathbb{R})}$.
    This proves (i).
    The optimality of the rate $N^{-\frac{s}{1+2s}}$ in the context of the assumptions of Theorem \ref{eq:1} follows by the standard argument in risk estimation: with the increase of $N$, the bias term decreases and tends to $0$ when $N \rightarrow \infty$, while the variance term increases and tends to $+ \infty$ when $N \rightarrow \infty$. 
    So, the optimal rate in N is achieved when the contributions of the bias and variance terms are equal.
    Under the assumptions of Theorem \ref{th:1}, it follows from the proof of \citep[Theorem 2.2.1, under the assumptions of Corollary 2.2.4]{dechevsky98} that the rate for which the bias and variance terms are balanced is $N^{-\frac{s}{1+2s}}$. (ii) is proved.

    Now let $s: 1 \leq s < 2$. In this case the proof is based on the same line of arguments, but with \citep[Corollary 2.2.4]{dechevsky98} being replaced by by \citep[Corollary 2.2.8]{dechevsky98} and noting that the expression $\frac{s - \frac{1}{r} + \frac{1}{q}}{1 + 2s - \frac{2}{r} + \frac{2}{q}}$ in \citep[(2.2.4)]{dechevsky98} becomes $\frac{s}{1 + 2s}$ for $r = q = + \infty$.
\end{proof}

\begin{proof}[Proof of Theorem \ref{th:2}]
    Follows straight-forwardly from the chain of equalities in (\ref{eq:11}).
\end{proof}

\begin{proof}[Proof of Theorem \ref{th:3}]
    To prove the theorem for every quadruple $(p, q, s, \sigma) : 0 < p \leq \infty$, $0 \leq q \leq \infty$, $n {(\frac{1}{p}-1)}_{+} < s+\sigma < r$ we invoke \citep[Lemma 3.10.2]{bergh1976}, as follows.
    Assume that $g=J^{\sigma} f$ and $g \in B_{pq}^{s+\sigma} (\mathbb{R}^n)$, that is, the RHS of (\ref{eq:6}) is finite when $s$ is replaced by $s+\sigma$.
    Then, by \citep[Lemma 3.10.2]{bergh1976}, the series (\ref{eq:5}) for $g$ is convergent in the topology of $B_{pq}^{s+\sigma} (\mathbb{R}^n)$, therefore, $g=J^{\sigma} f$ is learnable by the WBNN generated by the specified wavelet basis.

    In the particular case $1 \leq p \leq \infty$, $1 \leq q \leq \infty$, the above proof can be simplified, by using \citep[Lemma 2.2.1]{bergh1976} instead of \citep[Lemma 3.10.2]{bergh1976}.
\end{proof}

\begin{proof}[Proof of Corollary \ref{cor:1}]
    Follows from Theorem \ref{th:3} by using the lifting property of $J^{\sigma}$.  
\end{proof}

\begin{proof}[Proof of Theorem \ref{th:4}]
    The RHS of (\ref{eq:36}) is just a commuted version of its LHS.
    Since the basis $\{\varphi_{j_0 k}, \psi_{j k} \}$ is a Riesz unconditional basis \citep{daubechies10lect1992} all that has to be shown is that the series in the LHS of (\ref{eq:36}) is absolutely convergent, because then \citep[Lemma 3.10.2]{bergh1976} (or, alternatively in the particular case of Theorem \ref{th:4} when $p \geq 1$ and $q \geq 1$, \citep[Lemma 2.2.1]{bergh1976}) implies the statement of the theorem.
    The absolute convergence of the LHS in (\ref{eq:36}) in $B_{pq}^s$ follows from $f \in B_{pq}^s$, implying the the finiteness of ${||\cdot||}_{B_{pq}^s}$ in (\ref{eq:6}).
    From here, the absolute convergence of the LHS of (\ref{eq:36}) in $B_{\rho \eta}^\sigma$ for every $(\rho, \eta, \sigma)$ specified in the theorem follows from the Sobolev embedding $B_{pq}^s \hookrightarrow B_{\rho \eta}^\sigma$, see (\ref{eq:32}).
\end{proof}

\begin{proof}[Proof of Theorem \ref{th:5}]
    The space $B_{22}^\sigma$ is a Hilbert space and, by Theorem \ref{th:4}, the RHS of (\ref{eq:36}) holds true.
    Because of the Hilbertian geometry of $B_{22}^\sigma$, removing the term involving any one $\psi_{j_{\nu_0} k_{\nu_0}}$ from $\sum\limits_{\nu = 1}^M \beta_{j_\nu k_\nu} \psi_{j_\nu k_\nu}$ in (\ref{eq:36}) has the geometric meaning of orthogonal projection of 
    
    $\sum\limits_{\nu=1}^M \beta_{j_\nu k_\nu} \psi_{j_\nu k_\nu} \in \spn \{\psi_{j_\nu k_\nu}\}_{\nu=1}^M$ 
    onto 
    
    \begin{equation*}
        \Bigg( \sum\limits_{\nu=1}^{\nu_0 - 1} + \sum\limits_{\nu = \nu_0 +1}^M \Bigg) \beta_{j_\nu k_\nu} \psi_{j_\nu k_\nu} \in \spn \bigg\{ 
        \big\{ \psi_{j_\nu k_\nu} \big\}_{\nu=1}^{\nu_0 -1} \bigcup 
        \big\{ \psi_{j_\nu k_\nu} \big\}_{\nu=\nu_0 + 1}^{M}
        \bigg\}.
    \end{equation*}
    
    Since $|\beta_{j_\nu k_\nu}|$ are ordered as decreasing rearrangement with factor $2^{j_\nu (s - \frac{1}{p} + \frac{1}{2})}$ and $\sigma : \sigma - \frac{1}{2} = s - \frac{1}{p}$, (\ref{eq:45}) follows, which proves the theorem.
    Note the following remarkable geometric fact which remained implicit, but is crucial for the proof of the theorem: the basis $\{\varphi_{j_0 k_0}, \psi_{j k}\}$ is orthonormal only in $L_2 (\sigma = 0)$ but remains orthogonal in $B_{22}^\sigma \;\; (0 \leq \sigma < 2)$ which has the norm of a weighted $l_2$--sequence space with weight $2^{j \sigma}$.
\end{proof}

\section{Concluding remarks}\label{s8}

The present work is the first part of a sequence of studies dedicated to the new WBNNs.
The next two parts of this series \citep{llhm2022} and \citep{llhm2022_1} are currently in preparation.
The main focus in \citep{llhm2022} will be on a detailed study of the rich variety of threshold and non-threshold activation methods for learning curves in 2, 3 and higher dimensions.
\citep{llhm2022_1} will be dedicated to the diverse problems which arise when learning multivariate multidimensional geometric manifolds (surfaces, volume deformations and manifolds in dimensions higher than 3).
One topic will be to reduce the dimensionality of high-dimensional WBNNs to 1- and 2-dimensional WBNNs with full preservation of their functional efficiency.
One important application of this approach is to enable the use of GPGPU programming algorithms for learning parametric manifolds with arbitrary number of parameters, immersed in arbitrarily high-dimensional space \citep{dbb2011}, \citep{dbg2012}.
Another topic in these two studies is to make progress in understanding the connection between learning and approximation \citep{jsz2008} in the context of the new WBNNs.
It should be noted that the essence of our new approach in the present paper -- (a) separating the roles of wavelet depth and neural depth, (b) incorporating wavelet depth into the WBNN width to achieve consistency of learning, and (c) using the neural depth for accelerating the rate of consistent learning -- can in principle be used also in the much more general context of arbitrary tree--based adaptive partitions \citep{bcd2007}, \citep{bcd2014}.

We conjecture that Theorem \ref{th:5} can be generalized for a broader range than $(2, 2, \sigma)$ with $\sigma \in [0,r)$, namely, for the general assumptions on $(p, q, s)$ and $\rho, \eta, \sigma$ in (\ref{eq:36}) of Theorem \ref{th:4}.
However, to investigate this conjecture, one needs to resort to a very different and much more technically involved and spacious research approach, beginning with the derivation of \emph{direct inequalities} (Jackson-type, etc.) and \emph{inverse inequalities} (Bernstein/Markov-type, etc.) and then, based on the derived inequalities establish a connection between appropriately selected best-approximation functionals and Peetre \emph{K}-functionals.
We refer to \citep[Chapter 3]{pp1987} for an early, but sufficiently complete general exposition of this line of argument.

In conclusion, we note that by focusing on gradient/subgradient iterative optimization method in learning algorithms for NNs in the introduction, we left an important methodological gap which needs to be filled here \citep{SCHMIDHUBER201585}.
Numerical methods for optimal control based on Bellman's principle are very powerful in learning theory, both by swarm and deep evolutionary AI, including optimal control using feedback for supervised problems \citep{SCHMIDHUBER201585}.
Although, theoretically, Bellman's principle allows finding \emph{global extrema} for a very general class of criterial functionals (including non-convex, non-smooth (including non-Lipschitzian) ones), and under complicated sets of constraints (including ones induced by technological standards in real-life engineering problems): computing the/a global extremum is often unfeasible due to the huge computational complexity.
So, in many cases, a tradeoff is needed between affordable computational complexity and sufficiently high quality of a local extremum attained \citep{dech2006}, \citep{dech2006_1}.
So far, similar to gradient methods, optimal tradeoffs in dynamical programming are also achieved via natural, rather than artificial intelligence.
Nevertheless, if a dynamical programming algorithm is being applied in the context of machine learning using WBNNs, we may now have an acceptable automatic alternative.

\nocite{*}
\printbibliography

@ARTICLE{Zhang1992-wt,
  title    = "Wavelet networks",
  author   = "Zhang, Q. and Benveniste, A.",
  journal  = "IEEE Trans Neural Networks",
  volume   =  3,
  number   =  6,
  pages    = "889--898",
  year     =  1992,
  address  = "United States"
}

@ARTICLE{Alexandridis2013-gb,
  title    = "Wavelet neural networks: a practical guide",
  author   = "Alexandridis, A. K. and Zapranis, A. D.",
  journal  = "Neural Networks",
  volume   =  42,
  pages    = "1--27",
  year     =  2013,
  address  = "United States"
}

@ARTICLE{80341,  
    author={Chen, S. and Cowan, C.F.N. and Grant, P.M.},  
    journal={IEEE Transactions on Neural Networks},   
    title={Orthogonal least squares learning algorithm for radial basis function networks},   
    year={1991}, 
    volume={2}, 
    number={2}, 
    pages={302-309}
}

@ARTICLE{LeCun2015-iu,
  title    = "Deep learning",
  author   = "LeCun, Y. and Bengio, Y. and Hinton, G.",
  journal  = "Nature",
  volume   =  521,
  number   =  7553,
  pages    = "436--444",
  year     =  2015
}

@book{daubechies10lect1992,
    author = {Daubechies, I.},
    title = {Ten Lectures on Wavelets},
    year = {1992},
    publisher = {Society for Industrial and Applied Mathematics},
    address = {USA}
}

@article{Dahmen1997WaveletAM,
  title={Wavelet and multiscale methods for operator equations},
  author={Dahmen, W.},
  journal={Acta Numerica},
  year={1997},
  volume={6},
  pages={55 - 228}
}

@book {bergh1976,
    AUTHOR = {Bergh, J. and L\"{o}fstr\"{o}m, J.},
    TITLE = {Interpolation spaces. {A}n introduction},
    SERIES = {Grundlehren der Mathematischen Wissenschaften, No. 223},
    PUBLISHER = {Springer-Verlag, Berlin-New York},
    YEAR = {1976},
    PAGES = {x+207},
    MRCLASS = {46M35},
    MRNUMBER = {0482275},
}

@article{dech97,
    author = {Dechevsky, L. and Penev, S.},
    title = {On shape-preserving probabilistic wavelet approximators},
    journal = {Stochastic Analysis and Applications},
    volume={15},
    number={2},
    year = {1997},
    pages = {187-215}
}

@article{cohen93,
    author={Cohen, A. and Daubechies, I. and Vial, P.},
    title={Multiresolution analysis, wavelets and fast algorithms on an interval},
    journal={C.R. Acad. Sci. Paris},
    series={Ser. I Math.},
    volume={316},
    number={5},
    pages={417-421},
    year={1993}    
}

@article{cohen93_1,
    author = {Cohen, A. and Daubechies, I. and Vial, P.},
    title={Wavelets on the interval and fast wavelet transforms},
    journal={Appl. Comp. Harmonic Anal.},
    volume={1},
    number={1},
    pages={54-81},
    year={1993}
}

@article{cybenko89,
    author = {Cybenko, G.},
    title = {Approximation by superpositions of a sigmoidal function},
    journal = {Math. Control Signal Systems},
    volume = {2.4},
    year = {1989},
    pages = {303-314}
}

@article{cybenko92,
    author = {Cybenko, G.},
    title = {Correction: "Approximation by superpositions of a sigmoidal function", [Math. Control Signal Systems, 2.4 (1989), pp. 303-314]},
    journal = {Math. Control Signal Systems},
    volume = {5.4},
    year = {1992},
    pages = {455}
}

@book{reed1980,
    author = {Reed, M. and Simon, B.},
    title = {Methods of Modern Mathematical Physics, Vol. 1: Functional Analysis. 2nd edition.},
    edition = {2},
    publisher = {Academic Press [Harcourt Brace Jovanovich, Publishers], New York},
    year = {1980}
}

@inproceedings{lu2017,
    author = {Lu, Z. and Pu, H. and Wang, F. and Hu, Z. and Wang, L.},
    booktitle = {Advances in Neural Information Processing Systems},
    title = {The expressive power of neural networks: A view from the width},
    volume = {30},
    year = {2017}
}

@article{dechevsky98,
    author = {Dechevsky, L. and Penev, S.},
    title = {On shape-preserving wavelet estimators of cumulative distribution functions and densities},
    journal = {Stochastic Analysis and Applications},
    volume = {16},
    number = {3},
    year = {1998},
    pages = {423-462}
}

@book{triebel83,
    author = {Triebel, H.},
    title = {Theory of Function Spaces},
    series = {Monographs in Mathematics},
    volume = {78},
    publisher = {Birkh\"{a}user Verlag, Basel},
    year = {1983}
}

@book{samko1993,
    author = {Samko, S. G. and Kilbas, A. A. and Marichev, O. I.},
    title = {Fractional Integrals and Derivatives: Theory and Applications},
    publisher = {Gordon and Breach Science Publishers, Yverdon},
    year = {1993}
}

@book{frazier1991,
    author = {Frazier, M. and Jawerth, B. and Weiss, G.},
    title = {Littlewood-Paley Theory and the Study of Function Spaces},
    series = {CBMS Regional Conference Series in Mathematics},
    volume = {79},
    publisher = {American Mathematical Society, Providence, RI},
    year = {1991}
}

@book{iba2018,
    author = {Iba, H.},
    title = {Evolutionary Approach to Machine Learning and Deep Neural Networks. Neuro-evolution and Gene Regulatory Networks},
    publisher = {Springer, Singapore},
    year = {2018}
}

@book{iba2022,
    author = {Iba, H.},
    title = {Swarm Intelligence and Deep Evolution. Evolutionary Approach to Artificial Intelligence},
    publisher = {Taylor {\&} Francis CRC Press},
    year = {2022}
}

@book{hoyle1957,
    author = {Hoyle, F.},
    title = {The Black Cloud},
    publisher = {William Heinemann Ltd.},
    year = {1957}
}

@book{lem1964,
    author = {Lem, S.},
    title = {Niezwyciezony},
    publisher = {Wydawnictwo MON},
    year = {1964}
}

@book{lem2020,
    author = {Lem, S.},
    title = {The Invincible},
    publisher = {The MIT Press},
    year = {2020}
}

@book{leighton1992,
    author = {Leighton, F. T.},
    title = {Introduction to Parallel Algorithms and Architectures. Arrays, Trees, Hypercubes},
    publisher = {Morgan Kaufman Publishers, Inc., San Mateo, CA},
    year = {1992}
}

@article{dechevsky1999,
    author = {Dechevsky, L.T. and Ramsay, J.O. and Penev, S.I.},
    title = {Penalized wavelet estimation with {B}esov regularity constraints},
    journal = {Math. Balkanica (N.S.)},
    volume = {13},
    number = {3-4},
    year = {1999},
    pages = {257-376}
}

@article{dechevsky1999alone,
    author = {Dechevsky, L.T.},
    title = {Atomic decomposition of function spaces and fractional integral and differential operators. TMSF, AUBG '99, Part A (Blagoevgrad)},
    journal = {Fract. Calc. Appl. Anal.},
    volume = {2},
    number = {4},
    year = {1999},
    pages = {367-381}
}

@inproceedings{cupy_learningsys2017,
  author       = "Okuta, R. and Unno, Y. and Nishino, D. and Hido, S. and Loomis, C.",
  title        = "CuPy: A NumPy-Compatible Library for NVIDIA GPU Calculations",
  booktitle    = "Proceedings of Workshop on Machine Learning Systems (LearningSys) in The Thirty-first Annual Conference on Neural Information Processing Systems (NIPS)",
  year         = "2017",
  url          = "http://learningsys.org/nips17/assets/papers/paper_16.pdf"
}

@book{brenner1975,
    author = {Brenner, P. and Thom{\'e}e, V. and Wahlbin, L. B.},
    title = {Besov Spaces and Applications to Difference Methods for Initial Value Problems},
    series = {Lecture Notes in Mathematics, No. 434},
    publisher = {Springer-Verlag, Berlin-New York},
    year = {1975}
}

@article{SCHMIDHUBER201585,
    title = {Deep learning in neural networks: An overview},
    journal = {Neural Networks},
    volume = {61},
    pages = {85-117},
    year = {2015},
    author = {Jürgen Schmidhuber}
}

@article{dech2006,
    author = {Dechevsky, L.T. and Gulliksen, L.M},
    title = {A multirigid dynamical programming algorithm for discrete dynamical systems and its applications to numerical computation of global geodesics},
    journal = {Int. J. Pure Appl. Math.},
    volume = {33},
    number = {2},
    pages = {257--286}, 
    year = {2006}
}

@article{dech2006_1,
    author = {Dechevsky, L.T. and Gulliksen, L.M},
    title = {Application of a multirigid dynamical programming algorithm to optimal parametrization , and a model solution of an industrial problem},
    journal = {Int. J. Pure Appl. Math.},
    volume = {33},
    number = {3},
    pages = {381--406}, 
    year = {2006}
}

@techreport{llhm2022,
    author = {Dechevsky, L.T and Person, L.-E. and Singh, H. and Tangrand, K.M.},
    title = {Learning non-parametric regression-functions and densities by univariate wavelet-based neural networks},
    year = {2022}
}

@techreport{llhm2022_1,
    author = {Dechevsky, L.T and Person, L.-E. and Singh, H. and Tangrand, K.M.},
    title = {Learning of multidimensional geometric manifolds with wavelet-based neural networks},
    year = {2022}
}

@inproceedings{jsz2008,
    author = {Jetten, K. and Smale, S. and Zhou, D.-X},
    title = {Learning Theory and Approximation},
    booktitle = {Mathematisches Forschungsintitut Oberwolfach Workshop, June 29th - July 5th. Oberfolfach Report 30/2008},
    pages = {1655--1705},
    year = {2008}
}

@article{bcd2014,
    author = {Binev, P. and Cohen, A. and Dahmen, W. and DeVore, R.},
    title = {Classification algorithms using adaptive partitioning},
    journal = {Ann. Statist.},
    volume = {42},
    number = {6},
    pages = {2141--2163}, 
    year = {2014}
}

@book{pp1987,
    author = {Petrushev, P.P. and Popov, V.A.},
    title = {Rational Approximation of Real Functions},
    series = {Encyclopedia of Mathematics and its Applications, No. 28},
    publisher = {Cambridge University Press, Cambridge},
    year = {1987}
}

@inproceedings{dbb2011,
    author = {Dechevsky, L.T and Bratlie, J. and Bang, B. and Laks{\aa}, A. and Gundersen, J.},
    title = {Wavelet-based lossless one- and two-dimensional representation of multidimensional geometric data},
    booktitle = {AIP Conf. Proc.},
    volume = {1410},
    publisher = {Amer. Inst. Phys., Melville, NY},
    pages = {83--97},
    year = {2011}
}

@article{dbg2012,
    author = {Dechevsky, L.T and Bratlie, J. and Gundersen, J.},
    title = {Index mapping between tensor-product wavelet bases of different number of variables, and computing multivariate orthogonal discrete wavelet transforms on graphics processing units},
    journal = {Lecture Notes in Comput. Sci.},
    volume = {7116},
    publisher = {Springer, Heidelberg},
    year = {2012},
    pages = {402--410}
}

@article{bcd2007,
    author = {Binev, P. and Cohen, A. and Dahmen, W. and DeVore, R.},
    title = {Universal algorithms for learning theory. II. Piecewise polynomial functions.},
    journal = {Constr. Approx.},
    volume = {26},
    number = {2},
    year = {2007},
    pages = {127--152}
}

@article{hairer2017,
    title = {The reconstruction theorem in Besov spaces},
    journal = {Journal of Functional Analysis},
    volume = {273},
    number = {8},
    pages = {2578-2618},
    year = {2017},
    author = {Hairer, M. and Labbé, C.},
}

\end{document}